\documentclass[11pt]{article}

\usepackage{bm}
\usepackage[top=1in,bottom=1in,left=1in,right=1in]{geometry}

\usepackage{textgreek,mathtools}
\usepackage{amsmath}
\usepackage{graphicx}
\usepackage{subfig}
\usepackage[colorinlistoftodos]{todonotes}
\usepackage[colorlinks=true, allcolors=blue]{hyperref}
\usepackage{amsmath}
\usepackage{upgreek}
\usepackage{amssymb}
\usepackage{amsfonts}
\usepackage{amsthm}
\usepackage{mathrsfs}
\usepackage{fontenc}
\usepackage{empheq}
\usepackage{enumerate}
\usepackage{comment}
\usepackage{appendix}
\usepackage{bbm}
\usepackage{algorithm,algorithmic}
\mathtoolsset{showonlyrefs}
\usepackage{enumerate}
\usepackage{makecell}
\usepackage{booktabs}
\usepackage{multirow}
\usepackage{siunitx}
\usepackage{threeparttable}

\usepackage{chngcntr}

\DeclareMathOperator*{\arginf}{arg\,inf}

\theoremstyle{plain}
\newtheorem{thm}{Theorem}[section]

\newtheorem{prop}[thm]{Proposition}

\newtheorem{assu}[thm]{Assumption}

\theoremstyle{definition}
\newtheorem{defn}[thm]{Definition}

\newtheorem{nota}[thm]{Notation}

\theoremstyle{remark}
\newtheorem{rem}[thm]{Remark}

\numberwithin{equation}{section}
\newcommand{\EE}{\mathbb{E}}

\newcommand{\PP}{\mathbb{P}}

\newcommand{\ud}{\,\mathrm{d}}
\newcommand{\mc}[1]{\mathcal{#1}}

\newcommand{\EPS}{\varepsilon}
\newcommand{\R}{\mathbb{R}}
\newcommand{\F}{\mathscr{F}}
\newcommand{\SET}[1]{\left\{#1\right\}}
\newcommand{\E}{\mathbb{E}}
\newcommand{\NORM}[1]{\Vert#1\Vert}
\newcommand{\A}{\mathscr{A}}
\newcommand{\BAR}[1]{\overline{#1}}

\newcommand{\transpose}{^{\operatorname{T}}}

\newcommand{\Tr}{\mathrm{Tr}}
\newcommand{\N}{\mathbb{N}}
\newcommand{\X}{\mathscr{X}}

\title{Finite-Agent Stochastic Differential Games on Large Graphs: II. Graph-Based Architectures}

\author{Ruimeng Hu\thanks{Department of Mathematics, and Department of Statistics and Applied Probability, University of California, Santa Barbara, CA 93106-3080, {\em rhu@ucsb.edu}.} \and  Jihao Long\thanks{Institute for Advanced Algorithms Research, Shanghai, China, \em{longjh1998@gmail.com}.} \and Haosheng Zhou\thanks{Department of Statistics and Applied Probability, University of California, Santa Barbara, CA 93106-3110, \em{hzhou593@ucsb.edu}.}}
\date{\today}

\begin{document}

\maketitle
% Abstract
\begin{abstract}
We propose a novel neural network architecture, called Non-Trainable Modification (NTM), for computing Nash equilibria in stochastic differential games (SDGs) on graphs. These games model a broad class of graph-structured multi-agent systems arising in finance, robotics, energy, and social dynamics, where agents interact locally under uncertainty. The NTM  architecture imposes a graph-guided sparsification on feedforward neural networks, embedding fixed, non-trainable components aligned with the underlying graph topology. This design enhances interpretability and stability, while significantly reducing the number of trainable parameters in large-scale, sparse settings.

We theoretically establish a universal approximation property for NTM in static games on graphs and numerically validate its expressivity and robustness through supervised learning tasks. Building on this foundation, we incorporate NTM into two state-of-the-art game solvers, Direct Parameterization and Deep BSDE (backward stochastic differential equation), yielding their sparse variants (NTM-DP and NTM-DBSDE). Numerical experiments on three SDGs across various graph structures demonstrate that NTM-based methods achieve performance comparable to their fully trainable counterparts, while offering improved computational efficiency.
\end{abstract}


%\tableofcontents

\section{Introduction}

Graph-structured multi-agent systems arise in diverse applications such as financial networks \cite{elliott2014financial,jackson2021systemic}, distributed robotics \cite{bullo2009distributed,julian2012distributed}, social interactions \cite{freeman2004development,tabassum2018social}, and large-scale energy markets \cite{ghappani2025consensus,rokhforoz2023multi}.
Modeling these systems as stochastic differential games (SDGs) on graphs provides a rigorous framework for analyzing strategic behavior under uncertainty and local interactions. However, computing Nash equilibria (NE) in such games presents significant analytical and computational challenges, especially in high-dimensional and/or sparse settings where analytical tractability is limited and traditional numerical methods do not scale \cite{daskalakis2009complexity}.

In our earlier work \cite{hu2024finite}, we analyzed finite-agent linear-quadratic (LQ) games on graphs, one of the few classes of games that offer certain tractability. That work introduced a general class of graph-structured SDGs with heterogeneous player interactions and established convergence guarantees for fictitious play \cite{brown1949some,brown1951iterative}, a widely used learning scheme for computing Nash equilibrium (NE).
For vertex-transitive graphs, we further derived semi-explicit equilibrium characterizations as efficient constructions of numerical baselines.
In addition, we illustrated the interplay between game dynamics and graph topology through numerical experiments. These results motivate us to design numerically efficient and interpretable neural network (NN) architectures for approximating NE strategies in general graph-based games.

This paper, the second part of the series, focuses on algorithmic design. We introduce a novel NN architecture called \textbf{Non-Trainable Modification (NTM)}, which integrates graph topology directly into the network design by introducing fixed, non-trainable weights based on the adjacency structure. This graph-guided sparsification enhances interpretability, maintains training stability, and reduces the number of trainable parameters, an especially desirable feature for large and sparse graphs. Compared to conventional graph convolutional networks (GCNs) \cite{defferrard2016convolutional,kipf2016semi}, NTM offers better alignment with game-theoretic structures and is well-suited for strategy parameterization in stochastic games on graphs.

\smallskip
\noindent\textbf{Related Literature.}
Recent developments in deep learning have inspired numerous new methods for solving SDGs, most of which consist of two critical components: a technique for handling multi-agent interactions and an optimization scheme for strategy updates.
Commonly used techniques include fictitious play (FP) \cite{brown1949some,brown1951iterative}, where each player optimizes its own strategy while assuming other players use fixed past strategies, and policy iteration \cite{zhang2023global}, where all players follow their previous strategies to evaluate the value function, followed by a policy improvement.
In terms of the optimization scheme, algorithms are motivated by different characterizations of the optimality/equilibrium criteria.
Direct parameterization \cite{bachouch2022deep,han2016deep} aligns with directly minimizing the expected cost, deep Galerkin method \cite{al2019applications,sirignano2018dgm} focuses on solving the Hamilton-Jacobi-Bellman (HJB) system, while BSDE-based solvers \cite{han2017deep,han2020deep,hure2020deep} are motivated by the Feynman-Kac characterization of the HJB solution.
We refer readers to the survey \cite{hu2023recent} for a comprehensive overview of the development of deep learning-based game solvers.

Motivated by the non-Euclidean structure of graphs, the machine learning community has developed graph neural networks (GNNs), which outperform traditional graph embedding approaches in various graph-related tasks \cite{zhou2020graph}.
Most GNNs leverage message passing mechanisms and fall into two broad categories: spectral-based, which propagate information via the graph Laplacian \cite{defferrard2016convolutional,kipf2016semi}, and spatial-based, which operate directly on local neighborhoods \cite{atwood2016diffusion}. GNNs have achieved notable success in diverse domains such as chemistry and biology \cite{duvenaud2015convolutional,fout2017protein}, knowledge graphs \cite{bordes2013translating}, and generative modeling \cite{bojchevski2018netgan}. We refer readers to \cite{wu2020comprehensive,zhang2019graph} for comprehensive surveys.

Despite the success of GNNs in many domains, little work has explored their application to stochastic differential games (SDGs) on graphs. Directly applying standard GNN architectures, such as GCNs, in conjunction with fictitious play or policy iteration often yields unstable or suboptimal performance. In our experiments, we observe that GCNs, originally designed for feature extraction in transductive learning and node classification tasks, can be sensitive to parameter initialization and fail to produce robust or consistent results across test cases. These observations highlight the need for developing alternative graph-based architectures tailored to the structure and demands of graph-based SDGs. The NTM architecture proposed in this work provides one such approach.

\smallskip
\noindent\textbf{Main Contributions.}
The main contributions of this paper are summarized as follows.
To the best of our knowledge, this work provides one of the first systematic studies on the design of theoretically grounded graph-based architectures for numerically solving graphical games.

\begin{enumerate}[(i)]
    \item \textbf{A novel graph-based architecture with theoretical guarantees.}
    We propose the NTM architecture, a graph-guided modification of feedforward neural networks (FNNs) with fixed non-trainable components. 
    We establish a universal approximation result for Nash equilibrium (NE) strategies in static graphical games, and characterize how such non-trainable components influence the expressivity of the architecture.
    This provides a theoretical foundation for incorporating graph structures directly into neural architectures for equilibrium approximation.
    \item \textbf{Empirical expressivity and stability analysis.}
    We benchmark NTM against standard FNNs and state-of-the-art GCNs through extensive supervised learning experiments. 
    The results demonstrate that NTM achieves competitive expressivity while exhibiting improved training stability compared to GCNs.
    To the best of our knowledge, the training instability of GCNs has received limited attention in the context of function approximation.
    \item \textbf{Integration into deep learning-based game solvers.} We integrate NTM into two deep-learning-based solvers, Direct Parameterization (DP) \cite{han2016deep} and Deep BSDE \cite{han2017deep}, resulting in new algorithms (NTM-DP and NTM-DBSDE) for solving SDGs on graphs. Numerical results show that these methods achieve comparable performance to their original counterparts while significantly reducing the number of trainable parameters, highlighting the potential of structured architectures for complexity reduction, an important aspect that remains underexplored for deep learning-based game solvers.
\end{enumerate}

\smallskip
\noindent \textbf{Organization of the paper.} The rest of the paper is organized as follows: Section~\ref{sec:games_graphs} introduces the setup of SDGs on graphs and provides a brief review of deep fictitious play (DFP), which serves as the basis of later discussion.
Section~\ref{sec:arch} defines and investigates the NTM architecture, proving a universal approximation result, illustrating the interpretability, and presenting supervised learning results.
The combination of NTM and state-of-the-art game solvers results in non-trainable versions of numerical algorithms, which are further tested in Section~\ref{sec:numerics}.
Through numerical experiments conducted on three different models of SDGs on graphs, we demonstrate the validity of NTM for strategy parameterization, introducing sparsity while maintaining comparable performance.
Finally, Section~\ref{sec:conclusion_future} concludes the paper and provides possible directions for future research.

\smallskip
\noindent\textbf{Common notations.} We summarize some commonly used notations throughout the paper:
let \([N] := \{1,2,\ldots,N\}\), and \(\mathbb{S}^{N\times N}\) be the set of real symmetric \(N\times N\) matrices.
The vector concatenation function is denoted by \(\mathrm{Concat}(\cdot)\).
For dimensions \(d_1,d_2\in\N^+\): \(I_{d_1} \in \R^{d_1\times d_1}\) denotes the identity matrix, \(0_{d_1} \in \R^{d_1}\) (resp. \(0_{d_1\times d_2} \in \R^{d_1\times d_2}\)) denotes the zero vector (resp. matrix), and \(\mathbf 1_{d_1} \in \R^{d_1}\) (resp. \(\mathbf 1_{d_1\times d_2} \in \R^{d_1\times d_2}\)) denotes the all-one vector (resp. matrix).  
The Hadamard product is denoted by \(\odot\), and the Kronecker (tensor) product is denoted by \(\otimes\).
By default, \(\|\cdot\|\) stands for the matrix \(2\)-norm, and the operator \(\mathrm{diag}\) converts a vector into a diagonal matrix.

\section{Stochastic Differential Games on Graphs}\label{sec:games_graphs}

In this section, we describe a class of $N$-player stochastic differential games associated with a connected, simple, undirected graph  \(G=(V,E)\). Each vertex in the graph represents a player, and two players interact directly if and only if there is an edge between their corresponding nodes $v_i$ and $v_j$. 

\begin{defn}
    \label{defn:graph}
    Let \(G = (V,E)\) be a finite, connected graph with simple undirected edges, where \(V = \{v_1,\ldots,v_N\}\) denotes the set of vertices, and each edge \(e\in E\) is an unordered pair \(e = (u,v)\) for some \(u,v\in V\).
    Two vertices \(u,v\in V\) are said to be \emph{adjacent}, denoted \(u\sim v\), if and only if \((u,v)\in E\).
    The \emph{degree} of a vertex \(v\in V\), denoted by \(d_v\), is the number of edges connected to \(v\).
    The \emph{\(\ell\)-neighborhood} of a vertex \(v\in V\) is defined as the set of vertices that are exactly \(\ell\) steps away from \(v\), i.e., 
     $  \mathcal{N}^\ell_G(v) := \{u\in V:\text{there exists a path of length \(\ell\) in \(G\) between \(u\) and \(v\)}\}.$
    In particular, the \emph{\(1\)-neighborhood} of vertex \(v\), denoted by \(\mathcal{N}_G(v)\), consists of all vertices adjacent to \(v\) in \(G\).
    When the context is clear, we omit the dependence on $G$ and use the shorthand notation \(\mathcal{N}^{\ell(i)}\) for the \(\ell\)-neighborhood of vertex \(v_i\).

    For such a graph \(G\), the (normalized) graph Laplacian \(L\in\mathbb{S}^{N\times N}\) is defined as
    \begin{equation}
        \label{eqn:L}
        L_{ij} := \begin{cases}
            1 & \text{if}\ i=j\\
            -\frac{1}{\sqrt{d_{v_i}d_{v_j}}} & \text{if}\ i\neq j\ \text{and}\ v_i\sim v_j\\
            0 & \text{otherwise,}
        \end{cases},\ \forall i,j\in[N].
    \end{equation}
\end{defn}

As established in spectral graph theory \cite{chung1997spectral}, the graph Laplacian $L$ provides a powerful matrix representation that captures key structural properties of a graph, such as connectivity and isoperimetric features.

\subsection{The Game Setup}\label{sec:setup}

Consider a filtered probability space \((\Omega,\F,\{\F_t\}_{t\geq 0},\PP)\), supporting independent Brownian motions \(\{W^0_t\},\{W^1_t\},\ldots,\{W^N_t\}\) and \(\F_t = \sigma(W^0_s,\ldots,W^N_s,\ \forall s\in[0,t])\).
The state process of player \(i\), denoted by \(\{X^i_t\}\), is controlled through its strategy process \(\{\alpha^i_t\}\) and influenced by its $\ell$-neighborhood processes, subject to both idiosyncratic noises \(\{W^i_t\}\) and common noises modeled by \(\{W^0_t\}\):
\begin{equation}
    \label{eqn:state_dynamics}
    \ud X^i_t = b^i(t,X^i_t,X^{\mathcal{N}^{\ell(i)}}_t,\alpha^i_t)\ud t + \sigma^i(t,X^i_t,X^{\mathcal{N}^{\ell(i)}}_t,\alpha^i_t) \ud W^i_t + \sigma_0^i(t,X^i_t,X^{\mathcal{N}^{\ell(i)}}_t,\alpha^i_t) \ud W^0_t,\ \forall i\in[N].
\end{equation}
Here, \(X^{\mathcal{N}^{\ell(i)}}_t := [X^{j_1}_t,\ldots,X^{j_{k(i;\ell)}}_t]\transpose\) is the vector of states of all the players in the \(\ell\)-neighborhood of vertex \(v_i\), explicitly written as \(\mathcal{N}^{\ell(i)} = \{v_{j_1},...,v_{j_{k(i;\ell)}}\}\subset V\), where \(k(i;\ell) := |\mathcal{N}^{\ell(i)}|\).
Without loss of generality, we assume that both the state and control processes take values in \(\R\).

The game is defined on a finite time horizon \([0,T]\). Player \(i\) specifies its strategy \(\{\alpha^i_t\}_{t\in[0,T]}\) from the admissible set
\begin{equation}
    \label{eqn:admissible}
    \A := \left\{\alpha:\alpha\ \text{is progressively measurable w.r.t.}\ \{\F_t\},\ \E \int_0^T |\alpha_t|^2\ud t<\infty\right\},
\end{equation}
to minimize its expected cost of the form:
\begin{equation}
    \label{eqn:J}
    J^i(\alpha) := \E \left[\int_0^T f^i(t,X^i_t,X^{\mathcal{N}^{\ell(i)}}_t,\alpha^i_t)\ud t + g^i(X^i_T,X^{\mathcal{N}^{\ell(i)}}_T)\right],
\end{equation}
where the running cost \(f^i\) and the terminal cost \(g^i\) are given.

For well-posedness, the coefficients and cost functionals are assumed to satisfy the conditions in \cite[Assumption(Games), Section~2.1]{carmona2018probabilistic}, ensuring that the state dynamics in \eqref{eqn:state_dynamics} admit a unique strong solution for any admissible strategy profile \(\alpha := (\alpha^1,\ldots,\alpha^N)\in \A^N\).

\begin{rem}[Model interpretation]
    \label{rem:model_interpret}
    The state dynamics~\eqref{eqn:state_dynamics} and cost functionals~\eqref{eqn:J} represent a class of stochastic differential games whose information propagation is governed by the underlying graph structure.
    Specifically, the dynamics of \(\{X^i_t\}\) and the corresponding incentives received by player \(i\) depend explicitly only on the states of players within the
     \(\ell\)-neighborhood of vertex \(v_i\).
     In this paper, we mainly focus on the case \(\ell = 1\), while leaving general values of \(\ell\) for future research.

    Notably, when the underlying graph $G$ is complete (\(G = K_N\)), the proposed framework reduces to the standard formulation for SDGs studied in the literature (e.g., \cite{carmona2013mean}).
    Although Definition~\ref{defn:graph} assumes that the graph $G$ is connected, the model can naturally extend to disconnected graphs by applying the same formulation independently to each connected component of \(G\).
\end{rem}

\medskip

In competitive games, the notion of solution is given by the Nash equilibrium defined as follows.
\begin{defn}[Nash Equilibrium]
    \label{defn:NE}
    A collection of strategies of all players \(\hat{\alpha} := (\hat\alpha^1, \ldots, \hat\alpha^N) \) is called a Nash equilibrium (NE) if 
    \begin{align}
        J^i((\alpha,\hat{\alpha}^{-i}))\geq J^i(\hat{\alpha}),\ \forall i\in[N],\ \forall \alpha \in \A,
    \end{align}
    where \((\alpha,\hat{\alpha}^{-i}) := (\hat{\alpha}^1, \ldots, \hat{\alpha}^{i-1}, \alpha,\hat{\alpha}^{i+1}, \ldots ,\hat{\alpha}^N)\)
    denotes the strategy profile obtained by replacing player \(i\)'s strategy with \(\alpha\) while maintaining all other players' strategies fixed.
\end{defn}

In other words, a Nash equilibrium is a collection of strategies from which no player has an incentive to unilaterally deviate, provided that all other players adhere to their equilibrium strategies.

Additionally, one must explicitly specify the information set available to each player for decision-making. Different choices of information sets give rise to different strategy concepts, such as open-loop, closed-loop, and Markovian strategies. In this paper, our discussion is restricted to Markovian strategies, meaning player $i$'s strategy takes the form \(\alpha^i_t = \phi_i(t,X_t)\) where \(\phi_i:[0,T]\times \R^N\to\R\) is a deterministic feedback function. To align with the admissible set $\A$ (cf.~\eqref{eqn:admissible}), each feedback function \(\phi_i\) is required to be Borel measurable and satisfy the linear growth condition \(\sup_{(t,x)\in[0,T]\times \R^N}\frac{|\phi_i(t,x)|}{1+\NORM{x}}<\infty\).

\subsection{Deep Fictitious Play}\label{sec:DFP}

One widely used approach for computing NE in multi-agent games is fictitious play (FP), originally introduced by Brown \cite{brown1949some,brown1951iterative}. When applied to SDGs, the core idea is to recast the NE computation as a sequence of stochastic control problems. In each round, player $i$ optimizes its objective assuming that all other players are fixed and follow their historical strategies. Solving this problem yields an updated strategy for player $i$. A full round of FP involves one such update for each player. The procedure is repeated across multiple rounds, with the aim that all players’ strategies converge to the NE. At stage $k$, player $i$ solves the stochastic control problem over $\alpha^i$, which minimizes
\begin{equation}\label{def:JFP}
    J^{i,k}(\alpha^i) := \E \left[\int_0^T f^i(t,X^{i,k}_t,X^{\mathcal{N}^{\ell(i)}, k-1}_t,\alpha^i_t)\ud t + g^i(X^{i,k}_T,X^{\mathcal{N}^{\ell(i)}, k-1}_T)\right],
\end{equation}
subject to the state dynamics
\begin{multline}\label{def:XtFP}
    \ud X^{i,k}_t = b^i(t,X^{i,k}_t,X^{\mathcal{N}^{\ell(i)}, k-1}_t,\alpha^i_t)\ud t + \sigma^i(t,X^{i,k}_t,X^{\mathcal{N}^{\ell(i)},k-1}_t,\alpha^i_t) \ud W^i_t \\+ \sigma_0^i(t,X^{i,k}_t,X^{\mathcal{N}^{\ell(i)},k-1}_t,\alpha^i_t) \ud W^0_t,
\end{multline}
where $X^{\mathcal{N}^{\ell(i)},k-1}$ denotes the states in player $i$'s $\ell$-neighborhood, driven by the previously updated strategies $\{\hat \alpha^{j, k-1}\}_{j \in [N]}$. The optimized strategy then defines $\hat \alpha^{i, k}$ in the subsequent stage. 

In the literature, various versions of FP have been proposed. The one described above is known as simultaneous FP. In contrast, alternating FP sequentially updates $\hat \alpha^{i, k}$ from player 1 to $N$, using $\hat \alpha^{j,k}$ if already updated within the round, and $\hat \alpha^{j,k-1}$ otherwise. For further discussion, see \cite[Section~5]{hu2019deep}.

Given the complexity and high dimensionality of general SDGs, analytical solutions are rarely available. Motivated by recent advances in deep learning, \cite{hu2019deep} introduced deep fictitious play (DFP), a deep learning-based extension of the classical fictitious play (FP). The core idea of DFP is to parameterize the solution of the control problem \eqref{def:JFP}–\eqref{def:XtFP} using NNs, which can be trained via stochastic optimization.
The key benefit of DFP lies in its decoupled nature, i.e., each player maintains an independent NN, and different NNs are updated using separate loss functions.
As a result, when using DFP in the sequel, we only need to focus on the strategy parameterization of a single player.
Taking into account the game’s structure, particularly its dependence on an underlying graph, a natural question arises: 
\begin{center}
\textbf{Can this graph structure be leveraged in the design of a graph-based deep learning architecture that specifically adapts to DFP?}
\end{center}
In particular, when parameterizing player $i$'s control $\alpha^{i,k}$ for solving \eqref{def:JFP}--\eqref{def:XtFP}, can one adopt an NN architecture specifically tailored to reflect this graph-based dependence?  The answer to this question will be the main focus of Section~\ref{sec:arch}.

\subsection{Existing Work: Linear-Quadratic Results from \cite{hu2024finite}}\label{sec:LQ}

The stochastic differential game introduced in Section~\ref{sec:setup} provides a general framework to study strategic decision-making, in which players' direct interaction can be modeled through graph structures. While the general theory for this framework remains underdeveloped, specific linear-quadratic (LQ) game settings have been analyzed in recent literature, including works by \cite{hu2024finite,lacker2022case}.

Below, we briefly summarize the key results established in \cite{hu2024finite}, which serve as benchmark solutions for the numerical experiments presented in later sections.

In \cite{hu2024finite}, the game~\eqref{eqn:state_dynamics}--\eqref{eqn:J} features direct interactions within the $\ell$-neighborhoods of each player for general values of \(\ell\), with coefficients and cost functionals explicitly specified as follows
\begin{align}
b^i(t,x^i,x^{\mathcal{N}^{\ell(i)}},\alpha) & = am_G^i(x;\ell) + \alpha, \ \sigma^i \equiv \sigma, \ \sigma_0^i \equiv 0, \ m_G^i(x;\ell):=\sum_{j:v_j\in \mathcal{N}^{\ell(i)}}\tfrac{n(v_j,v_i;\ell)}{|\mathcal{N}^{\ell(i)}|}x^j - x^i,\label{eqn:LQ_dynamics}\\
f^i(t,x^i,x^{\mathcal{N}^{\ell(i)}},\alpha) &= \tfrac{1}{2}\alpha^2 - q\alpha\; m_G^i(x;\ell) + \tfrac{\EPS}{2}[m_G^i(x;\ell)]^2,\ 
\label{eqn:LQ_terminal_cost}
    g^i(x^i,x^{\mathcal{N}^{\ell(i)}}) = \tfrac{c}{2}[m_G^i(x;\ell)]^2,
\end{align}
where \(n(v_j,v_i;\ell)\) stands for the number of paths of length \(\ell\) connecting vertices \(v_i\) and \(v_j\). 
Here, the model parameters \(a\geq 0\) and \(q,c,\EPS>0\), with the condition \(q^2\leq\EPS\) imposed to ensure the well-posedness of the problem.

Under this model, the convergence of fictitious play methods \cite{brown1949some,brown1951iterative} is established on general graphs under a smallness condition on the model parameters that is independent of $N$.
Furthermore, when the graph $G$ exhibits additional symmetry, specifically when it is \emph{vertex-transitive} (i.e., for any two vertices, there exists a graph automorphism mapping one to the other), the Nash equilibrium admits a semi-explicit form, as summarized in the following theorem.

Define \(M(\ell,L) := I - (I-L)^\ell\in\mathbb{S}^{N\times N}\) and the subset \(\X\subset \mathbb{S}^{N\times N}\) as
\begin{equation}
    \X := \SET{X:X \geq 0\text{ is a linear combination of } I,L^1,\ldots,L^{(N-1)\ell}}.
    \label{eqn:solution_space}
\end{equation}

\begin{thm}[{\cite[Theorem~4.2 \& Section~4.1.1]{hu2024finite}}]
    \label{thm:semi-explicit}
Let \(G\) be a simple, connected, vertex-transitive graph with undirected edges.
Assume \(q^2 = \EPS\) in model~\eqref{eqn:LQ_dynamics}--\eqref{eqn:LQ_terminal_cost}. Then for any \( T>0\), there exists a unique solution \(R:[0,T]\to \X\) to the following ODE:
\begin{equation}
    R'(t) = \frac{1}{c}\Tr \left[Q'(R(t)) e^{-t(a+q)M(\ell,L)}\right]e^{-t(a+q)M(\ell,L)},\quad R(0) = 0,
    \label{eqn:second_charac}
\end{equation}
where the function \(Q:\X\to \R\) is defined as
$
    Q(X) := \left[\det\left(I+cXM(\ell,L)\right)\right]^{\frac{1}{N}}.$ 

The Markovian NE for player \(i\) is then given by
\begin{equation}
    \hat{\alpha}^i(t,x) = -qe_i\transpose M(\ell,L) x - e_i\transpose F^i_t x,
    \label{eqn:transitive_NE}
\end{equation}
where $F^i$ is given by
\begin{equation}
    F^i_t = \frac{1}{\frac{\Tr(P_t) - (a+q)\Tr(M(\ell,L))}{N}}[P_t - (a+q)M(\ell,L)]e_ie_i\transpose [P_t - (a+q)M(\ell,L)],
    \label{eqn:closed_form_F}
\end{equation}
with \(P\) defined in terms of \(R\) by
\begin{equation}
    P_t = (a+q)M(\ell,L) + R'(T-t)cM(\ell,L)[I + R(T-t)cM(\ell,L)]^{-1}.
    \label{eqn:R_to_P}
\end{equation}
\end{thm}

In particular, when \(\ell = 1\), \(M(\ell,L) = L\) and the model~\eqref{eqn:LQ_dynamics}--\eqref{eqn:LQ_terminal_cost} reduces to an LQ game with mean reverting dynamics within each \(1\)-neighborhood of the graph.

\section{A Graph-Based Network Architecture}\label{sec:arch}
In this section, we propose a graph-based neural network architecture called NTM (non-trainable modification\footnote{Non-trainable parameters are those that remain fixed during training and are not updated by optimizers such as SGD or Adam. This contrasts with standard trainable parameters, which are iteratively adjusted based on the gradient of a loss function.}), specifically designed for solving SDGs on graphs. NTM incorporates non-trainable components fixed throughout training and is well suited for integration with DFP to handle multi-agent interactions. The architecture is theoretically motivated, interpretable, strongly expressive, and demonstrates superior training stability compared to existing graph convolutional networks (GCNs) \cite{defferrard2016convolutional,kipf2016semi}.  By integrating NTM into state-of-the-art deep learning algorithms for game solving, we obtain non-trainable variants of these methods, highlighting the general applicability of the NTM framework. 

\subsection{NTM: An Interpretable Sparse Network Architecture}\label{sec:NTM}

Given dimensions \(d_{\mathrm{in}},d_{\mathrm{out}}\in\N^+\), consider the approximation of a mapping \(f:\R^{Nd_{\mathrm{in}}}\to\R^{d_{\mathrm{out}}}\).
Such \(f\) can be parameterized by the multi-dimensional NTM architecture proposed in Definition~\ref{defn:NTM_general}. We first introduce certain notations for vectors and matrices in Notation~\ref{nota:NTM_general}, which will be used throughout this work.

\begin{nota}
    \label{nota:NTM_general}
    Given \(d \in \N^+\), \(i\in[N]\) and layer index \(k\in \N^+\), let \(z^{(k)}\in\R^{Nd}\) denote the neuron values at layer $k$. Define \(z^{(k)}_i\in\R^{d}\) as the vector consisting of the \(((i-1)d + 1)\)-th to the \((id)\)-th components of \(z^{(k)}\). That is, \( z^{(k)} = \mathrm{Concat}(z^{(k)}_1,\ldots,z^{(k)}_N)\).
    Similarly, we use \( z^{(0)} = \mathrm{Concat}(z^{(0)}_1,\ldots,z^{(0)}_N)  \in \R^{N d_{\mathrm{in}}}\) to represent the input of the NTM architecture.
    Throughout, the nonlinear activation function \(\sigma:\R \to \R\) applies component-wise.
\end{nota}

\begin{defn}[Multi-Dimensional NTM]
    \label{defn:NTM_general}
    Given graph \(G\), vertex index \(i\in[N]\), network depth \( K\geq 2\) and message channel width \(M\in\N^+\), the \((K+1)\)-layer-\(M\)-channel NTM architecture associated with \(v_i\) and graph \(G\) is denoted by \(\phi^{\mathrm{NTM}}_{i,K,M,G}:\R^{Nd_{\mathrm{in}}}\to\R^{d_{\mathrm{out}}}\).
    Its forward propagation rule is defined recursively for \(k\in[K-1],\ p\in[N]\):
    \begin{align}
     \label{eqn:NTM_general_1}
    &z_p^{(k+1)} =
      \sum_{r=1}^{M} g_{pr}^{(k)}\odot
      \sigma\Big(
        W_{pr,p}^{(k)}\,z_p^{(k)}
        +
        \sum_{q:v_q\in \mathcal{N}_G(v_p)} W_{pr,q}^{(k)}\,z_q^{(k)}+ h^{(k)}_{pr}
      \Big) + b^{(k)}_p,\\
      \label{eqn:NTM_general_2}
      &z^{(K+1)} = W_{\mathrm{out},i}\,z_i^{(K)}
        +
        \sum_{q:v_q\in \mathcal{N}_G(v_i)} W_{\mathrm{out},q}\,z_q^{(K)} + b_{\mathrm{out}} \in \R^{d_{\mathrm{out}}},\ \forall k\in[K-1],\ p\in[N],
    \end{align}
    followed by a pre-processing step of the network input: \(z_p^{(1)} = W_{\mathrm{in},p} z_p^{(0)} + b_{\mathrm{in},p}\).
    The trainable parameters have the following shapes:
         $g^{(k)}_{pr},h^{(k)}_{pr},b^{(k)}_{p},b_{\mathrm{in},p}\in \R^{d}$, $W^{(k)}_{pr,q}\in \R^{d\times d}$, $W_{\mathrm{in},p}\in\R^{d\times d_{\mathrm{in}}}$, $W_{\mathrm{out},q}\in\R^{d_{\mathrm{out}}\times d}$,  $b_{\mathrm{out}}\in\R^{d_\mathrm{out}},$
    where \(d\in\N^+\) is interpreted as the dimension of hidden neurons.
\end{defn}

The forward propagation rule produces \(z_p^{(k+1)}\) by aggregating information from the previous layer \(z^{(k)}\) over the 1-neighborhood of \(v_p\) in \(G\), following the sparsity pattern of the graph Laplacian \(L\) in \eqref{eqn:L}. 
This connection between NTM and \(L\) is made explicit by the following identity:
\begin{equation}
    W_{pr,p}^{(k)}\,z_p^{(k)}
        +
        \sum_{q:v_q\in \mathcal{N}_G(v_p)} W_{pr,q}^{(k)}\,z_q^{(k)} =
       E_p (W^{(k)}_{r} \odot L_{\mathrm{mask}})
        z^{(k)},
\end{equation}
where \(W^{(k)}_r \in \mathbb{R}^{Nd \times Nd}\) is a block matrix with \((W^{(k)}_r)_{pq} := W_{pr,q}^{(k)}\), and \(L_{\mathrm{mask}} \in \mathbb{S}^{Nd \times Nd}\) is a binary sparsity mask defined by \((L_{\mathrm{mask}})_{ip,jq} := \mathbb{I}_{\{L_{pq} \neq 0\}}\) for any \(i,j\in[d]\) and \(p,q\in[N]\). Here, \(\mathbb{I}\) denotes the indicator function, and \(E_p:= e_p\transpose\otimes I_d \in \mathbb{R}^{d\times Nd}\) extracts the \(p\)-th \(d\)-dimensional block, where \(e_p\in\R^N\) is the \(p\)-th standard basis vector.

This design mirrors the structure of spatial GCNs in the machine learning literature (e.g., \cite{wu2020comprehensive, zhang2019graph}), such as diffusion-convolutional neural networks (DCNN) \cite{atwood2016diffusion}. However, a key distinction is that the number of parameters in the NTM architecture scales with the graph size $N$, unlike GCNs. This difference will be further clarified in Section~\ref{sec:SVL}.

To better illustrate the advantages of the NTM architecture, particularly its interpretability and sparsity, we focus on the one-dimensional case while leaving the multi-dimensional case for future studies.
Specifically, we consider games where each player’s state and strategy lie in \(\R\), and the feedback function at any fixed time $t_0 \in [0,T]$ takes the form \(x\mapsto \phi_i(t_0,x)\), mapping \(\R^N\to\R\).
With the application of NTM for strategy parameterization, the dimensions must be set as \(d_{\mathrm{in}} = d_{\mathrm{out}}  = 1\). We can further simplify the architecture by setting the dimension of hidden neuron $d = 1$, resulting in the one-dimensional (1D) NTM architecture. In this scenario, we take the trivial pre-processing step with $W_{\mathrm{in},p} = 1$ and $b_{\mathrm{in},p} = 0$, \(\forall p\in[N]\), leading to $z^{(1)} = z^{(0)}$. In the following context, except for expressivity discussions (Section~\ref{sec:expressivity}) that focus on the multi-dimensional NTM, whenever we refer to NTM \(\phi^{\mathrm{NTM}}_{i,K,M,G}\) without specification, we mean the 1D version with \(d_{\mathrm{in}} = d_{\mathrm{out}} = d = 1\) and network input \(z^{(1)}= z^{(0)}\). 

\smallskip
\noindent\textbf{The choice of depth $K$.} If one requires that \(\phi^{\mathrm{NTM}}_{i,K,M,G}(z^{(1)})\in\R\) depends on all components of \(z^{(1)}\in\R^N\), \(\forall i\in[N]\), one must choose \(K \geq \mathrm{diam}(G)\), where \(\mathrm{diam}(G)\) denotes the diameter of the graph \cite{chung1997spectral}, i.e., the maximum length of the shortest path between any two vertices.
This condition for \(K\) aligns with the error bound~\eqref{eqn:UAT} in the universal approximation result (Theorem~\ref{thm:express}) later established in Section~\ref{sec:expressivity}.
However, it can be overly restrictive for large sparse graphs, e.g., in a cycle graph \(G = C_N\), \(\mathrm{diam}(C_N) = \lfloor\frac{N}{2}\rfloor\) which scales linearly with \(N\). 
Fortunately, the following observations alleviate the need for a large depth \(K\): (i) many real-world graphs---especially in domains like finance and social science---exhibit the small-world property: they are highly clustered with small diameters \cite{watts1999networks} (ii) the universal approximation error bound~\eqref{eqn:UAT} shrinks exponentially in \(K\). 
Shown in Section~\ref{sec:numerics},  small values of $K$ are often sufficient to capture long-range dependencies in graph-based stochastic games.

\smallskip
\noindent\textbf{Interpretability of NTM architecture.} 
The NTM architecture is interpretable and aligns with the decision-making structure in games on graphs.
Consider \(\phi^{\text{NTM}}_{i,K,M,G}\) such that \(K\geq \mathrm{diam}(G)\), with the network input \(z^{(1)}\) being the global state vector across all players, and its output \(z^{(K+1)}\) approximating player \(i\)'s strategy \(\phi_i(t_0,z^{(1)})\) at a fixed time \(t_0\in[0,T]\).

The forward propagation of NTM can be interpreted as follows: player \(i\)'s decision \(z^{(K+1)}\) is based on states \(z^{(K)}\) of players in \(\mathcal{N}_G(v_i)\), who directly influence its dynamics and incentives. 
Each \(z^{(K)}_p\), in turn,  aggregates information from \(z^{(K-1)}\) within \(\mathcal{N}_G(v_p)\).
As this process continues backward through the layers, the dependence of \(z^{(K+1)}\) on \(z^{(K-1)}\) extends to \(\bigcup_{j=1}^{2}\mathcal{N}^j_G(v_i)\). 
Iterating this logic over $K$ layers, player $i$'s strategy reflects a hierarchy of influence---prioritizing directly connected players within $\mathcal{N}_G(v_i)$, while incorporating indirect influences passed through neighbors.

\begin{rem}
Using the sparsity structure of the graph Laplacian $L_G$, instead of that of the adjacency matrix $A_G$, enables the architecture to retain information from smaller neighborhoods \(\bigcup_{j=1}^k\mathcal{N}^{j}_G(v_i)\), even as data propagates to larger neighborhoods \(\mathcal{N}^{k+1}_G(v_i)\).
For illustration, consider a modified graph \(G'\) constructed by adding self-loops to each vertex of a simple graph $G$. Then \(A_{G'}\) shares the same sparsity pattern  as \(L_G\), and \(\mathcal{N}^k_{G'}(v_i) =  \bigcup_{j=1}^k\mathcal{N}^{j}_G(v_i)\cup \{v_i\}\). Adding self-loops effectively encodes memory of intermediate neighborhood layers, which is central to the NTM interpretability.
\end{rem}

\smallskip
\noindent\textbf{Sparsity of NTM architecture.}
Another key advantage of NTM lies in its significantly reduced number of trainable parameters compared to a fully connected feedforward neural network (FNN) of the same width, offering comparable interpretability while improving training efficiency.

We first define the baseline of comparison: a \((K + 1)\)-layer fully connected FNN with $H$ hidden neurons per layer,  denoted by \(\phi^{\mathrm{FNN}}_{K,H}:\R^N\to\R\). Its forward propagation follows:
\begin{equation}
     z^{(k+1)} = \sigma(W^{(k)}z^{(k)}+b^{(k)})\in\R^H, \quad z^{(K + 1)} = w_{\mathrm{out}}\transpose z^{(K)}+b_{\mathrm{out}},\ \forall k\in[K-1],
\end{equation}
with trainable parameters \(b^{(k)},w_{\mathrm{out}}\in \R^H,\ W^{(1)}\in \R^{H\times N},\ W^{(2)},\ldots,W^{(K-1)}\in \R^{H\times H},\ b_{\mathrm{out}}\in\R\).

In standard FNN training, all parameters are updated iteratively using gradient-based optimizers (e.g., SGD or Adam), with gradients computed via auto-differentiation.
Unlike fully connected FNNs, NTM incorporates fixed, non-trainable weights determined by the graph structure.
In particular, based on equations~\eqref{eqn:NTM_general_1}–\eqref{eqn:NTM_general_2}, the trainable parameter \(W^{(k)}_{pr,q}\) (resp. \(W_{\mathrm{out},q}\)) exists if and only if \(L_{pq} \neq 0\) (resp. \(L_{iq} \neq 0\)). These constraints reflect the graph topology and remain unchanged throughout training. This graph-aware reduction in trainable parameters\footnote{This approach is closely related to sparse training techniques in machine learning \cite{ma2018survey}, which aim to reduce model complexity while maintaining a comparable performance.} motivates the name non-trainable modification (NTM).

Quantitatively, to enable a fair comparison between NTM and FNN, the FNN must use width \(H = N\) in order to be interpretable in the sense discussed above, resulting in \((K - 1)(N^2 + N) + N + 1\) trainable parameters in \(\phi^{\mathrm{FNN}}_{K,N}\). In contrast, the 1D NTM \(\phi^{\mathrm{NTM}}_{i,K,M,G}\) has at most \(M(K - 1)(4N + 2|E|) + N + 1\) trainable parameters, where \(|E|\) denotes the number of edges in \(G\). The approximate ratio of trainable parameters is \(\boldsymbol{M s_G}\), where the edge density \(s_G := |E| / \binom{N}{2}\). This reduction becomes especially significant for large, sparse graphs.

Nevertheless, we emphasize that enforcing sparsity often comes at a cost, typically observed as diminished performance.
In theory, it may reduce network expressivity and destroy the structure of the parameter space.
Numerically, it may change the optimization landscape, leading to instability and poor convergence during training.
In contrast, when introduced carefully and in alignment with the problem structure, the drawbacks of sparsification can be mitigated, allowing its benefits to be fully leveraged. The NTM architecture, being graph-guided and problem-aligned, maintains adequate expressivity and exhibits greater robustness, as supported by Theorem~\ref{thm:express} and shown by supervised learning experiments in Section~\ref{sec:SVL}. 
For an illustration of how arbitrarily fixing parameters can harm performance, see Appendix~\ref{app:SVL_approximation}.

\subsection{Expressivity Analysis Motivated by Static Games}\label{sec:expressivity}

This section establishes an expressivity result for the general multi-dimensional NTM architecture (cf. \eqref{eqn:NTM_general_1}--\eqref{eqn:NTM_general_2}).
By explicitly constructing the NTM parameters, we establish Theorem~\ref{thm:express}, which shows that NTM possesses a universal approximation property for equilibrium feedback functions arising in static $N$-player Markovian games on graphs.
This result provides theoretical support for using NTM architectures to parameterize strategies when numerically solving games on graphs.

Consider a static $N$-player game, where each player \(i\) has a state \(x^i\in \mathcal{X} \subset \R^{d_{\mathrm{in}}}\) and chooses a strategy \(a^i \in \mathcal{A} \subset \R^{d_{\mathrm{out}}}\), with \(\mathcal{X}\) and \(\mathcal{A}\) being compact under the Euclidean topology (\(d_{\mathrm{in}},d_{\mathrm{out}}\) match those in Definition~\ref{defn:NTM_general}). Recall that \(\mathcal{N}_G(v_i) = \{v_{j_1}, \ldots, v_{j_{k(i)}}\}\) denotes the 1-hop neighborhood of player \(i\). Define \(O^{\mathcal{N}(i)} := \mathrm{Concat}(O^{j_1}, \ldots, O^{j_{k(i)}})\) for any object \(O\) (e.g., state or action) associated with these neighbors.
For example, \(x^{\mathcal{N}(i)} \in \mathcal{X}^{k(i)}\) and \(a^{\mathcal{N}(i)} \in \mathcal{A}^{k(i)}\) represent the local state and action vectors of player \(i\)'s neighbors.

Each player \(i\in[N]\) seeks to minimize a cost functional
\begin{equation}
    G_i:\mathcal{X}\times \mathcal{A}\times \prod_{j:v_j\in \mathcal{N}_G(v_i)}(\mathcal{X}\times \mathcal{A})\ni (x^i,a^i,x^{\mathcal{N}(i)},a^{\mathcal{N}(i)})\mapsto G_i(x^i,a^i,x^{\mathcal{N}(i)},a^{\mathcal{N}(i)}) \in\R,
\end{equation}
which depends on its own state-action pair and those of its neighbors.

Given the full state \(x := \mathrm{Concat}(x^1, \ldots, x^N) \in \mathcal{X}^{N}\) and joint action \(a := \mathrm{Concat}(a^1, \ldots, a^N) \in \mathcal{A}^{N}\), player \(i\)'s best response is given by \(A_i(x^i, x^{\mathcal{N}(i)}, a^{\mathcal{N}(i)}) := \arginf_{b \in \mathcal{A}} G_i(x^i, b, x^{\mathcal{N}(i)}, a^{\mathcal{N}(i)})\).
Aggregating all best responses yields the joint best response mapping:
\begin{equation}
    A(x, a) := \big(A_1(x^1, x^{\mathcal{N}(1)}, a^{\mathcal{N}(1)}), \ldots, A_N(x^N, x^{\mathcal{N}(N)}, a^{\mathcal{N}(N)})\big),
\end{equation}
which defines a function \(A : \mathcal{X}^N \times \mathcal{A}^N \to \mathcal{A}^N\).

Assuming Markovian strategies of the form \(a^i = \phi_i(x)\) for player \(i\)'s feedback function
\(\phi_i: \mathcal{X}^N \to \mathcal{A}\), a Nash equilibrium corresponds to a fixed point \(\widehat{\phi} = (\widehat{\phi}_1, \ldots, \widehat{\phi}_N)\) satisfying
\begin{equation}
    \label{eqn:NE_FP}
    \widehat{\phi}(x) = A(x, \widehat{\phi}(x)), \quad \forall x \in \mathcal{X}^N.
\end{equation}
Note that this fixed-point characterization is consistent with the definition of Nash equilibrium in the dynamic setting (see Definition~\ref{defn:NE}).

\begin{assu}
    \label{assu:express}
    Assume the following:
    \begin{enumerate}[(i)]
        \item There exists \(\rho \in (0,1)\) such that
        \begin{equation}
            \|A(x, a) - A(x, b)\|_\infty \leq \rho \|a - b\|_\infty, \quad \forall a, b \in \mathcal{A}^N,\ x \in \mathcal{X}^N.
        \end{equation}
        \item For any \(\delta > 0\) and \(i \in [N]\), there exists a neural network of the form
        \begin{equation}
            \widetilde A_i(x^i,x^{\mathcal{N}(i)},a^{\mathcal{N}(i)})=
            \sum_{r=1}^{M}\beta_{ir}\odot
            \sigma\Big(\gamma_{ir,i}^{(x)}\,x^i + \sum_{j:v_j \in \mathcal{N}_G(v_i)}\big(\gamma_{ir,j}^{(x)}x^j + \gamma_{ir,j}^{(a)}a^j\big) + \eta_{ir}\Big),
        \end{equation}
        with parameters of shapes \(\beta_{ir}, \eta_{ir} \in \R^{d_{\mathrm{out}}}, \ \gamma_{ir,j}^{(x)} \in \R^{d_{\mathrm{out}} \times d_{\mathrm{in}}}, \ \gamma_{ir,j}^{(a)} \in \R^{d_{\mathrm{out}} \times d_{\mathrm{out}}}\),
        such that
        \begin{equation}
            \big|\widetilde A_i(x^i, x^{\mathcal{N}(i)}, a^{\mathcal{N}(i)}) - A_i(x^i, x^{\mathcal{N}(i)}, a^{\mathcal{N}(i)})\big| \leq \delta,\ \forall(x^i, x^{\mathcal{N}(i)}, a^{\mathcal{N}(i)}) \in \mathcal{X} \times \mathcal{X}^{k(i)} \times \mathcal{A}^{k(i)}.
        \end{equation}
        Here, \(d_{\mathrm{in}}, d_{\mathrm{out}}, \sigma, M\) match those in Definition~\ref{defn:NTM_general}.
    \end{enumerate}
\end{assu}

Part (i) of Assumption~\ref{assu:express} ensures the existence and uniqueness of the Nash equilibrium. 
Indeed, by Banach's fixed-point theorem, the contraction property of the mapping \(a \mapsto A(x,a)\) guarantees a unique fixed point \(\widehat{\phi}\) satisfying equation~\eqref{eqn:NE_FP}. Part (ii) is motivated by the universal approximation theorem for feedforward neural networks, when the activation function \(\sigma\) is sufficiently expressive (e.g., ReLU, Tanh, etc.), and \(A_i\) is continuous over its compact domain \cite{cybenko1989approximation,hornik1990universal,leshno1993multilayer}.

We next state a universal approximation theorem for the NTM architecture with respect to Nash equilibrium feedback functions in static \(N\)-player games on graphs, the proof of which is presented in Appendix~\ref{app:proof_exp}. 
From a practical perspective, this result implies that, with sufficient depth and appropriate parameters, the NTM architecture can approximate Markovian NE strategies in static graphical games arbitrarily well.

\begin{thm}
\label{thm:express}
Suppose Assumption~\ref{assu:express} holds with \(d_{\mathrm{in}},d_{\mathrm{out}}\in \N^+\), \(M \geq 2\), and \(\sigma = \mathrm{ReLU}\). 
If $d = d_{\mathrm{in}} + d_{\mathrm{out}}$, then for any \(\delta > 0\), \(i \in [N]\), \(K \in \mathbb{N}\), and any simple connected undirected graph \(G\), there exists a set of trainable parameters such that the multi-dimensional NTM architecture \(\phi^{\mathrm{NTM}}_{i,K,M,G}\) satisfies:
\begin{equation}
    \label{eqn:UAT}
    \|\phi^{\mathrm{NTM}}_{i,K,M,G}(x) - \widehat{\phi}_i(x)\|_\infty \leq \frac{\delta}{1 - \rho} + \rho^{K - 1} \|\widehat{\phi}(x)\|_\infty,\quad \forall x \in \mathcal{X}^N.
\end{equation}
\end{thm}

As Theorem~\ref{thm:express} shows, the approximation error in \eqref{eqn:UAT} can be made arbitrarily small by choosing a sufficiently large depth $K$.
This differs from many prior works on the expressivity of GCNs and GNNs, which primarily focus on tasks such as feature extraction or graph classification~\cite{wu2020comprehensive,zhang2019graph}, where expressivity refers to distinguishing graph structures through embeddings. 
To our knowledge, Theorem~\ref{thm:express} provides one of the first theoretical guarantees of expressivity for approximating structured functions over graphs, specifically equilibrium feedback maps. 
In particular, this result establishes one of the first provable connections between the structure of non-trainable components and the expressive power of the architecture, thereby providing a theoretical foundation for incorporating graph structures directly into neural architectures for equilibrium approximation.

Though Theorem~\ref{thm:express} requires $d = d_{\mathrm{in}} + d_{\mathrm{out}}$, potentially violated by the 1D NTM (where $d=1 < d_{\mathrm{in}} + d_{\mathrm{out}}$), Sections~\ref{sec:SVL} and~\ref{sec:numerics} provide numerical evidence that the result remains valid in practice.

In dynamic games, where feedback strategies depend on time, we assign a separate NTM to each discretized time point $0 = t_0 < t_1 < \dots < t_{N_T} = T$ to approximate \(x\mapsto \phi_i(t,x)\). 
While the game is formulated in continuous time, practical numerical implementations inevitably rely on time discretization.
Under this discretized perspective, the dynamic problem reduces to a sequence of static approximation tasks at different time steps, for which the universal approximation result in Theorem \ref{thm:express} applies.

\subsection{Supervised Learning Benchmark for Expressivity and Stability}\label{sec:SVL}
This section presents numerical experiments comparing different network architectures in terms of expressivity and training stability. 
We design four supervised learning tasks, motivated by Proposition~\ref{prop:NE_approx}, which approximates NE strategies in LQ games, to evaluate how well each architecture captures the structure of equilibrium maps.
Results show that 1D NTM matches the expressivity proved in Theorem~\ref{thm:express} and consistently outperforms GCNs in training stability, supporting its use for strategy parameterization in graph-based games.

\smallskip
\noindent\textbf{An approximation result for the NE.}
Building on the theoretical analysis in Theorem~\ref{thm:semi-explicit}, we establish Proposition~\ref{prop:NE_approx}, an explicit approximation result for NE strategies. The proof is deferred to Appendix~\ref{app:proof_approx}. We define $\mathscr{P}_L := \mathrm{span}\{ I,L^1,\ldots,L^{N-1}\}\subset \mathbb{S}^{N\times N}$.

\begin{prop}
    \label{prop:NE_approx}
    Consider the LQ SDG on graphs~\eqref{eqn:LQ_dynamics}--\eqref{eqn:LQ_terminal_cost} with \(\ell = 1\).
    Let \(G\) be a simple, connected, vertex-transitive graph with undirected edges and assume \(q^2 = \EPS\). 
    The Markovian NE \(\hat{\alpha}^i(t,x)\) exists for any player \(i\in[N]\). For any \(\delta>0\), \(t\in[0,T]\), \( i\in[N]\), and any compact \(\mathcal{K}\subset\R^N\), there exists \(A_t(L)\in \mathscr{P}_L\) such that
\begin{equation}
    \sup_{x\in \mathcal{K}}|\hat{\alpha}^i(t,x) - e_i\transpose A_t(L) x|<\delta.
\end{equation}
\end{prop}

Proposition~\ref{prop:NE_approx} shows that under certain conditions on the graph and model parameters, the NE strategy of player \(i\) in the LQ game approximately lies in a restricted function space \(\mathcal{A}^i\):
\begin{equation}
    \label{eqn:NTM_func_class}
    \mathcal{A}^i := \{f:\R^N\to\R|f(x)= e_i\transpose A(L) x,\ \forall A(L)\in\mathscr{P}_L\}.
\end{equation}
Therefore, when strategies of the LQ game are parameterized by NTM, expressivity should be evaluated relative to this class, i.e., whether \(\phi^{\mathrm{NTM}}_{i,K,M,G}\) can approximate any function in \(\mathcal{A}^i\) up to an arbitrary accuracy.
We remark that, by Theorem~\ref{thm:semi-explicit}, a similar characterization of \(\mathcal{A}^i\) in \eqref{eqn:NTM_func_class} remains valid for general values of \(\ell\).

While $\mathcal{A}^i$ offers a meaningful benchmark, it does not generalize to broader classes of non-LQ games. Although Theorem~\ref{thm:express} provides a theoretical guarantee for expressivity, it does not directly apply for 1D NTM and the assumptions may not be easy to verify in practice. 
To address this gap, we turn to numerical experiments to directly assess both expressivity and training stability, with the test functions motivated by the structure of the restricted function class \(\mathcal{A}^i\). 

\smallskip
\noindent\textbf{Architectures for comparison.}
We compare three architectures: standard FNN, our proposed NTM, and the state-of-the-art Chebyshev GCN \cite{defferrard2016convolutional} defined in Definition~\ref{defn:Chebyshev_GCN}, using notations from Notation~\ref{nota:NTM_general}.
Unlike spatial-based architectures like NTM, the Chebyshev GCN is spectral-based and operates in the frequency domain using the graph Fourier transform.
Graph signals are transformed into the spectral domain, where filtering operations, parameterized by trainable weights, are applied, effectively implementing convolution over the graph Laplacian spectrum.

\begin{defn}[Chebyshev GCN \cite{defferrard2016convolutional}]
    \label{defn:Chebyshev_GCN}
    Let \(G\) be a graph with Laplacian \(L\), and \(K\geq 2\). The \((K+1)\)-layer Chebyshev GCN associated with graph \(G\) is denoted by \(\phi^{\mathrm{CHEB}}_{K,G}:\R^N\to\R\).
    It consists of \(K\) Chebyshev graph convolution layers followed by one fully connected output layer:
    \begin{align}
        \label{eqn:cheb_GCN_conv}
        z^{(k+1)} &= \sigma(z^{(k)}W^{(k,1)} + \tilde{L}z^{(k)}W^{(k,2)} + \mathbf{1}_{N\times F_{k+1}}\;\mathrm{diag}(b^{(k)}))\in\R^{N\times F_{k+1}},\\
        z^{(K+1)} &= w_{\mathrm{out}}\transpose z^{(K)}+ b_{\mathrm{out}}\in\R,\ \forall k\in[K-1],
    \end{align}
    where \(\{F_k\}_{k=1}^K\) is a sequence of customizable feature dimensions, with \(F_1 = 1\) (since the input \(z^{(1)}\in\R^N\)) and \(F_K = 1\).
    The trainable parameters have shapes: \(W^{(k,1)},W^{(k,2)}\in \R^{F_k\times F_{k+1}},\ b^{(k)}\in\R^{F_{k+1}},\ w_{\mathrm{out}}\in \R^N,\ b_{\mathrm{out}}\in\R\).
    Here, \(\tilde{L} := \frac{2}{\lambda_{\mathrm{max}}}L - I\) is the scaled graph Laplacian, where \(\lambda_{\mathrm{max}}\) is the largest eigenvalue of \(L\). 
    The operator \(\mathrm{diag}\) converts a vector into a diagonal matrix.
\end{defn}

Although both GCNs and NTM are graph-based, their architectures differ significantly in structure and interpretation, particularly in how they propagate information and parameterize interactions. These differences are further elaborated in Remark~\ref{rem:GCN_shape}.

\begin{rem}
    \label{rem:GCN_shape}
    For the Chebyshev GCN \(\phi^{\mathrm{CHEB}}_{K,G}\), the input \(z^{(1)}\in\R^N\) is convoluted to a graph signal \(z^{(k)}\in\R^{N\times F_k}\) at each layer using trainable spectral filters.  The bias term \(b^{(k)}\in\R^{F_{k+1}}\) is broadcast to all nodes using the matrix \(\mathbf{1}_{N\times F_{k+1}}\,\mathrm{diag}(b^{(k)})\in\R^{N\times F_{k+1}}\).
    Practically used for feature extraction and typically applied for large datasets, the number of trainable parameters in graph convolution layers of GCNs is independent of the graph size $N$, in contrast to FNN and NTM.

    We point out that, DFP aligns better with spatial-based architectures than spectral-based ones.
    With both essentially spreading information on graphs, spectral information divergence mixes all the players' information, making it hard to retrieve those belonging to each of the player.
    Differently, the spatial approach maintains the graph neighborhood structure, enabling the identification of each player's information state.
    
\end{rem}

\smallskip

\noindent\textbf{The supervised learning task.}
In the supervised learning task, numerical experiments are designed to assess two key aspects of the architectures: (i) expressivity and (ii) training stability.

Expressivity tests evaluate how well each architecture approximates a given target function, reflecting their representational power. Meanwhile, training stability tests assess how reliably the optimization process identifies near-optimal parameters. Architectures with stable training consistently converge to good solutions, while unstable ones may yield suboptimal results even after extensive training. Stability is essential for robustness to noise and replicability in practice, and is assessed by whether the architecture consistently produces reasonable approximations of the learning target across multiple training runs.

The experiments consist of four supervised learning tasks, each involving the approximation of a specific target function \(f:\R^N\to\R\). Motivated by the theoretical insights, these target functions are designed to resemble the structure of those in the restricted function class \(\mathcal{A}^i\) (defined in~\eqref{eqn:NTM_func_class}).
Without loss of generality, we set \(i=1\) so that all test cases are analogues to player \(1\)'s equilibrium feedback strategy under certain settings.
Denote by \(f_j:\R^N\to\R\) the target function in the \(j\)-th test case, where \(j\in\{1,2,3,4\}\), as specified below:

\smallskip
\textbf{LQ} (LQ Game): \(f_1(x) := -qe_1\transpose Lx - e_1\transpose F^1_0x\) maps the state vector \(x\in\R^N\) to player \(1\)'s NE strategy at time \(0\) in the LQ game on graph~\eqref{eqn:LQ_dynamics}--\eqref{eqn:LQ_terminal_cost} with \(\ell = 1\).
    Here, \(F^i_t\) solves the Riccati system \cite[equation~(2.11)]{hu2024finite}.
    When \(G\) is vertex-transitive and \(q^2 = \EPS\), this construction can be approximated up to an arbitrary accuracy by functions in \(\mathcal{A}^1\). This serves as the canonical test of whether an architecture can approximate NE strategies under LQ settings.

\smallskip
    \textbf{LQMH} (LQ with Multi-Hop Interaction): \(
    f_2(x) := -qe_1\transpose M(\ell,L)x - e_1\transpose F^1_0x\) models player 1’s NE strategy in the LQ game on graph~\eqref{eqn:LQ_dynamics}--\eqref{eqn:LQ_terminal_cost} under general values of \(\ell\),
where \(M(\ell,L) := I - (I-L)^\ell\) and \(F^i_t\) satisfies \cite[equation~(2.11)]{hu2024finite} with \(L\) replaced by \(M(\ell,L)\). This tests whether the architecture can capture multi-hop dependencies over graphs.

\smallskip
   \textbf{NL} (Nonlinear Dependence): 
\(
    f_3(x) := 3 e_1\transpose \frac{\det(0.5I + L)\, L + \|(I +L)^{-1}\|_F\, L^2}{\|\det(0.5I + L)\, L + \|(I +L)^{-1}\|_F\, L^2\|_F} \log(\mathbf{1}_N + x)\), where the logarithm is applied component-wise.
    This tests the ability to represent nonlinear state dependence in the form \(e_1\transpose A(L)g(x)\), where \(A(L)\in \mathscr{P}_L\) and \(g\) induces component-wise nonlinear actions.

\smallskip
    \textbf{NLM} (Nonlinear Mixture): \(f_4(x) := 10\sum_{i=1}^N e_i\transpose \frac{(\det(0.5I + L)\, L^i + \Tr((I +iL)^{-1})\, L^{N+1-i})}{\|\det(0.5I + L)\, L^i + \Tr((I +iL)^{-1})\, L^{N+1-i}\|_F} (\sin(2\pi x)\odot (\mathbf 1_{N}+x)^{\frac{i}{N}})\) is a complex mixture involving multi-hop interactions and nonlinear state dependence, where sine and power functions act element-wise for vectors. This goes far beyond \(\mathcal{A}^1\), testing whether the architecture can approximate mixtures of multiple players’ NE strategies with nonlinear, graph-structured dependencies. It is designed to be the most challenging benchmark.

Training is performed on the domain of the unit hypercube \([0,1]^N\) with the standard mean squared error as the loss function. 
The key hyperparameter choices are summarized in Table~\ref{tab:summ_SVL}, with full implementation details provided in Appendix~\ref{app:hyper_SVL}.

\begin{table}[t]
\centering
\begin{threeparttable}
\caption{Summary of hyperparameters used in Section~\ref{sec:SVL}.}
\label{tab:summ_SVL}
\begin{tabular}{lll}
\toprule
\textbf{Category} & \textbf{Hyperparameter} & \textbf{Value} \\
\midrule
\multirow{3}{*}{Architecture}
& FNN & $4$ layers, width $32$, \texttt{Tanh} \\
& NTM & $4$ layers, channel width $M = 3$, $i = 1$, \texttt{ReLU} \\
& GCN & $4$ layers, feature dimensions $(1,64,1)$, \texttt{ReLU} \\
\midrule
\multirow{6}{*}{Training}
& Batch size & $N_{\mathrm{batch}} = 256$ \\
& Epochs\(^\dagger\) & $N_{\mathrm{epoch}} = 2000 \, j$ \\
& \multirow{2}{*}{Learning rate\(^\dagger\) (LR)} 
    & FNN/NTM: $\eta_1 = \eta_2 = 0.001,\ \eta_3 = 0.005,\ \eta_4 = 0.01$ \\
& 
    & GCN: $\eta_1 = \eta_2 = 0.002,\ \eta_3 = 0.01,\ \eta_4 = 0.02$ \\
& LR decay factor & $\gamma = 0.5$\\
& LR decay step (epochs) & $\tau = 500$ \\
\bottomrule
\end{tabular}
\begin{tablenotes}[flushleft]
\footnotesize
\item[$\dagger$] Depends on the test case index $j\in \{1,2,3,4\}$.
\end{tablenotes}
\end{threeparttable}
\end{table}

\smallskip
\noindent\textbf{Evaluation metrics.}
For testing, we generate \(N_{\mathrm{test}} = 25000\) samples \(y^1,\ldots,y^{N_{\mathrm{test}}}\) using Latin hypercube sampling (LHS) \cite{mckay2000comparison} for better space-filling over the unit hypercube \([0,1]^N\). Performance is evaluated using the relative mean-squared error (RMSE):
\begin{equation}
    \mathrm{RMSE} := \frac{ \sum_{m=1}^{N_{\mathrm{test}}}[\phi(y^m) - f(y^m)]^2}{ \sum_{m=1}^{N_{\mathrm{test}}} f^2(y^m) + \delta},
\end{equation}
where \(\delta = 10^{-8}\) ensures numerical stability. A lower RMSE indicates better approximation of the target function.

\smallskip
\noindent\textbf{Graph choices.} We test on the following four representative graphs:
\begin{align}
    &C_N:\ \text{the \(N\)-vertex cycle graph},\quad S_N:\ \text{the \(N\)-vertex star graph},\\
    &K_{m,m}:\ \text{the \(2m\)-vertex complete bipartite graph with partitions of size}\ m\ \text{and}\ m,\\
    &\mathrm{RST}_N:\ \text{the \(N\)-vertex uniformly random spanning tree}.
\end{align}
Among these graphs, only \(\mathrm{RST}_N\) is random;  it is sampled and fixed across all architectures and test cases.

\smallskip
\noindent \textbf{Numerical results\footnote{All experiments were conducted on an Nvidia GeForce RTX 2080 Ti GPU, on a server with Linux operating system.}.}
For each test case, graph, and architecture, we repeat the training and evaluation 1000 times, producing 1000 RMSE values visualized via box plots in Figure~\ref{fig:SVL}.
To facilitate a quantitative comparison, we also report the corresponding summary statistics in Appendix~\ref{app:summary_stat}.
For test cases \textbf{LQ} and \textbf{LQMH}, numerical results are presented under the following set of model parameters:
\begin{equation}
    N=10,\quad T = 1.0,\quad a=0.1,\quad \sigma=0.5,\quad q = 0,\quad \EPS = 1.0,\quad c = 1.0,\quad \ell = 3.
\end{equation}

\begin{figure}[!htbp]
    \centering
    \includegraphics[width=1.0\textwidth]{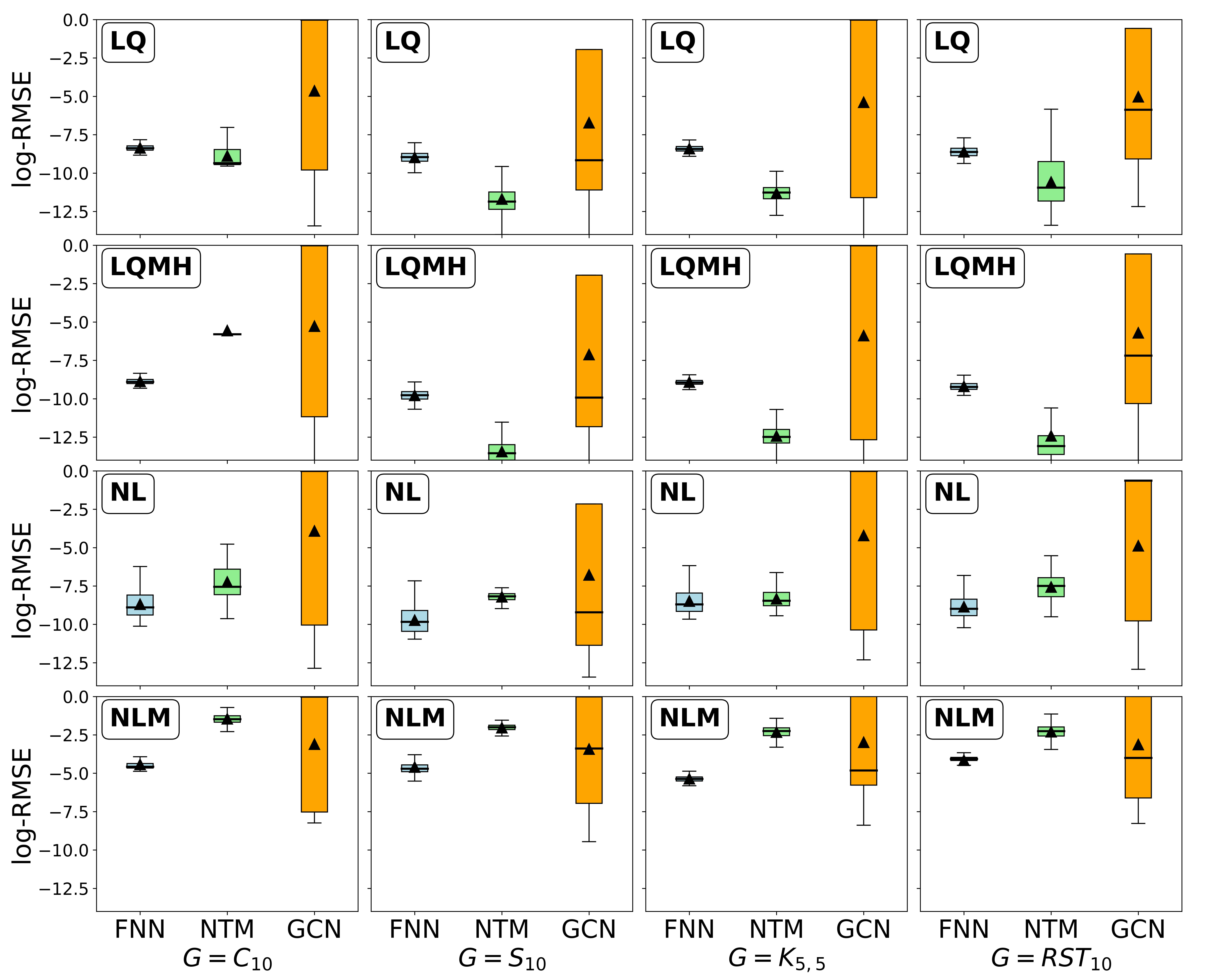}
    \caption{Box plots of log-RMSE over 1000 runs for all test cases, comparing different NN architectures across various graphs. Bold black lines indicate the medians; black triangles indicate the means.}
    \label{fig:SVL}
\end{figure}

In each box plot, the minimum RMSE reflects the expressivity of the architecture. 
Empirically, a low minimum RMSE value is indicative of strong expressivity of the corresponding architecture.
From the results, the Chebyshev GCN and FNN demonstrate strong expressivity across all test cases.
NTM demonstrates strong expressivity for test cases \textbf{LQ}, \textbf{LQMH}, \textbf{NL}, but fails to capture the complexity of \textbf{NLM}, which involves a mixed effect of different players' feedback functions.
Nevertheless, this limitation is not critical, as the adoption of DFP only requires the ability to approximate a single player's strategy.
Therefore, we conclude that, NTM retains enough expressivity to approximate a single player's NE strategies in general (non-LQ) games, despite its non-trainable entries. 
Note that NTM does not perform as well on the cycle graph as on the other graphs in the test case \textbf{LQMH}, since the NTM tested has \(K = 3\), but the cycle graph has a larger diameter \(5\).
Moreover, long-range dependencies play a more significant role in \textbf{LQMH}, compared to other test cases.
This observation aligns with the interpretation of the depth provided in Section~\ref{sec:NTM}.

A large fluctuation and a large maximum of log-RMSE in Figure~\ref{fig:SVL} indicate training instability. 
In our experiments, both FNN and NTM exhibit relatively stable training behavior, while the Chebyshev GCN shows noticeably less stable performance across the considered test cases.
This empirical observation motivates the exploration of alternative architectures, such as the proposed NTM.
To the best of our knowledge, such instability in GCNs has not been reported in prior literature. 
We emphasize that our findings are based on the specific experimental setting considered here, rather than indicating an intrinsic limitation of GCN architectures.
We hypothesize that this behavior stems from random initialization when the input feature dimension is not sufficiently large. Verifying this hypothesis requires additional experiments on multi-dimensional games, which we leave for future work.

Notably, the conclusions drawn from Figure~\ref{fig:SVL} are robust across the GCN variants and hyperparameter settings tested in our experiments.
Numerical experiments conducted based on another widely used GCN \cite{kipf2016semi} yield qualitatively similar results.
Extensive trials with varying hyperparameters yield consistent performance trends, suggesting that these results are not sensitive to specific tuning choices. We also observe that increasing the feature dimension, changing the activation function, and adding fully connected layers do not mitigate the instability issue in GCNs.
Additional numerical results are provided in Appendix~\ref{app:SVL}.

\begin{rem}[Two Notions of GCN Expressivity]
\label{rem:GCN_express}
In machine learning, GCN expressivity typically refers to the model’s ability to generate node embeddings that distinguish different graph structures, often benchmarked via the Weisfeiler-Lehman (WL) graph isomorphism test \cite{huang2021short}. This notion is central to works like \cite{kipf2016semi,xu2018powerful}.
In contrast, our study focuses on function approximation capacity---how well a GCN approximates graph-structured functions, such as NE strategies of games on graphs. This notion is crucial in control and game-theoretic contexts, where parameterizing policies or strategies is the primary goal. 
\end{rem}

\medskip

In summary, our experiments demonstrate that NTM is a generally applicable and interpretable architecture for strategy parameterization. It offers sufficient expressivity and stable training. Compared with FNN, it benefits from a structured design and reduced parameter count; compared with GCNs, it avoids instability issues.

\subsection{NTM Algorithms for Games on Graphs}\label{sec:alg}

We now integrate NTM into the DFP framework (Section~\ref{sec:DFP}), which reduces a multi-agent game into a sequence of single-agent stochastic control problems. 
Within this setup, we incorporate NTM into two widely used game solvers, Direct Parameterization (DP) \cite{han2016deep} and Deep BSDE (DBSDE) \cite{han2017deep}, resulting in non-trainable, graph-structured variants. These versions retain compatibility with existing methods while enhancing interpretability and sparsity.

\medskip
\noindent\textbf{Direct parameterization (DP).}
The DP algorithm, introduced in \cite{han2016deep}, solves stochastic control problems \eqref{def:JFP}--\eqref{def:XtFP} by directly parameterizing the feedback strategy \(\alpha^i\) with a neural network (NN) \(\phi_i^{\mathrm{NN}}\), mapping states \(x\in\R^N\) to strategies \(\phi_i^{\mathrm{NN}}(t,x)\) at time \(t\in[0,T]\). Given fixed strategies \(\{\phi_j^{\mathrm{NN}}(t,x)\}_{j \neq i}\) of other players, the dynamics~\eqref{eqn:state_dynamics} are simulated to approximate the expected cost~\eqref{eqn:J} via discretization and Monte Carlo. This cost serves as the loss function for training \(\phi_i^{\mathrm{NN}}\). Repeating this for all \(i\in[N]\) completes one DFP iteration.

In the non-trainable variant, we replace the NN \(\phi_i^{\mathrm{NN}}(t,\cdot)\), for each player $i$, using a single NTM architecture \(\phi^{\mathrm{NTM}}_{i,K,M,G}(t,\cdot)\). This yields the NTM-DP algorithm, summarized in Algorithm~\ref{alg:NTM-DP}. 

\begin{algorithm}
\renewcommand{\algorithmicrequire}{\textbf{Input:}}
\renewcommand{\algorithmicensure}{\textbf{Output:}}
\caption{NTM-DP: Non-Trainable Direct Parameterization for Games on Graphs}
\label{alg:NTM-DP}
\begin{algorithmic}[1]
    \REQUIRE A family of NTM architectures \(\{\phi^{\mathrm{NTM}}_{j,K,M,G}(t,\cdot)\}_{j \in [N],\, t \in [0,T]}\)
    \STATE Initialize trainable parameters of \(\phi^{\mathrm{NTM}}_{j,K,M,G}(t,\cdot),\ \forall j \in [N],\, t \in [0,T]\)
    \REPEAT
        \FOR{each player \(i \in [N]\)}
            \STATE Simulate sample paths of \(\{X_t\}\) using current strategies \(\alpha^j_t = \phi^{\mathrm{NTM}}_{j,K,M,G}(t, X_t),\forall j \in [N]\)
            \STATE Compute player \(i\)’s expected cost \(J^i(\alpha)\) from simulated paths
            \STATE Update trainable parameters of \(\phi^{\mathrm{NTM}}_{i,K,M,G}(t,\cdot)\), for all \(t \in [0,T]\), using loss \(J^i(\alpha)\)
        \ENDFOR
    \UNTIL{convergence or maximum number of DFP rounds reached}
    \ENSURE A family of trained NTMs approximating the Nash equilibrium strategies
\end{algorithmic}
\end{algorithm}

\medskip

\noindent\textbf{Deep BSDE.}
Originally developed for solving semi-linear parabolic PDEs \cite{han2017deep}, the Deep BSDE method has since been extended to compute Markovian Nash equilibria in stochastic differential games \cite{han2020deep}. Unlike direct parameterization (DP), Deep BSDE applies to Markovian games with uncontrolled non-degenerate diffusion coefficients.

We provide a high-level description of Deep BSDE here, leaving mathematical derivations and implementation details to Appendix~\ref{app:DBSDE}. In brief, the value functions $v^i$ that characterize the NE satisfy a system of Hamilton-Jacobi-Bellman (HJB) equations. When the minimizer in each HJB equation is unique and explicit, the system reduces to a semi-linear parabolic PDE system. Deep BSDE leverages the nonlinear Feynman-Kac representation \cite{pardoux2005backward,pardoux1999forward} to reformulate this PDE system as forward-backward SDEs (FBSDEs). The core idea is to approximate the adjoint process using an NN \(\mc{G}_i^{\mathrm{NN}}:\R^N\to\R^N\), trained by minimizing the mismatch of terminal conditions.

In the non-trainable version of Deep BSDE, the 1D NTM architecture is modified to approximate a mapping \(\R^N\to\R^N\): the output layer~\eqref{eqn:NTM_general_2} is adjusted to ensure the correct dimension, while the forward propagation rule~\eqref{eqn:NTM_general_1} remains the same:
\begin{equation}
    z^{(K+1)}=(W_{\mathrm{out}}\odot L_{\mathrm{mask}})\transpose z^{(K)} + b_{\mathrm{out}}\in\R^N,
\end{equation}
where the trainable parameters have the following shapes: \(W_{\mathrm{out}}\in \R^{N\times N},\ b_{\mathrm{out}}\in\R^N\) and the mask \((L_{\mathrm{mask}})_{ij} = 1\) if and only if \(L_{ij}\neq 0\) encodes the sparsity pattern of the graph Laplacian.
This modification, denoted by $\phi^{\mathrm{NTM},i}_{K,M,G}$, preserves the validity of the NTM interpretation (cf. Section~\ref{sec:NTM}). The architecture itself is identical across players; the superscript only indicates which player’s strategy is being approximated. Combining the NTM architecture with Deep BSDE yields the NTM–DBSDE algorithm, presented as Algorithm~\ref{alg:NTM-DBSDE} in Appendix~\ref{app:NTM-DBSDE}.

\section{Numerical Experiments}\label{sec:numerics}

In this section, we evaluate the performance of NTM-based algorithms (NTM-DP and NTM-DBSDE proposed in Section~\ref{sec:alg}) for numerically solving stochastic differential games on graphs.
We begin by introducing the trajectory-based evaluation metrics, along with the collection of graph structures considered in our experiments. 
We then compare their parameter counts and errors with those of FNN-based counterparts across three different game models: a linear–quadratic (LQ) game (Section~\ref{sec:numerics-lq}),  a multi-agent portfolio game (Section~\ref{sec:numerics-portfolio}), and a nonlinear variant of the LQ game (Section~\ref{sec:numerics_nonlq}). The first two admit semi-explicit solutions; the third illustrates performance in a more complex setting without analytical baselines.  
Results show that despite having far fewer parameters, NTM-based solvers match the accuracy of their trainable counterparts across game models and graph structures.

All experiments are implemented in \texttt{PyTorch} and run on an Nvidia GeForce RTX 2080 Ti GPU. 
The code is publicly available at the GitHub repository.\footnote{\url{https://github.com/Haosheng-Zhou/Finite-Agent-Graph-SDG-Architecture}}
The key hyperparameter choices are summarized in Table~\ref{tab:summ_game}, with full implementation details provided in Appendix~\ref{app:hyper_game}.

\begin{table}[t]
\centering
\begin{threeparttable}
\caption{Summary of hyperparameters used in Section~\ref{sec:numerics}.}
\label{tab:summ_game}
\begin{tabular}{lll}
\toprule
\textbf{Category} & \textbf{Hyperparameter} & \textbf{Value} \\
\midrule
\multirow{3}{*}{Architecture}
& FNN & $3$ layers, width $32$, \texttt{ReLU} \\
& NTM & $3$ layers, channel width $M =3$, \texttt{ReLU} \\
& Value network$^{\dagger}$ & 4-layer FNN, width $32$, \texttt{Tanh} \\
\midrule
\multirow{7}{*}{Training}
& Time steps & $N_T = 50$ \\
& DFP rounds & $N_{\mathrm{round}} = 40$ \\
& Epochs per DFP round & $N_{\mathrm{epoch}} = 150$ \\
& Batch size & $N_{\mathrm{batch}} = 256$ \\
& Learning rate (LR) & $\eta = 0.001$ \\
& LR decay factor & $\gamma = 0.5$ \\
& LR decay step (epochs) & $\tau = 30$ \\
\bottomrule
\end{tabular}
\begin{tablenotes}
\footnotesize
\item[$\dagger$] Used only for (NTM-)DBSDE.
\end{tablenotes}
\end{threeparttable}
\end{table}

\smallskip

\noindent \textbf{Evaluation metrics.} We compute relative mean square error (RMSE) over \(N_{\mathrm{test}} = 5000\) paths of the equilibrium states and strategies. Let \(\{\hat{X}^{i,m}_t\}, \{\hat{\alpha}^{i,m}_t\}\) denote the \(m\)-th sample path of the baseline equilibrium (computed via semi-explicit solutions or vanilla DFP) state and strategy processes of player \(i\), and \(\{\tilde{X}^{i,m}_t\},\{\tilde{\alpha}^{i,m}_t\}\) their neural-network approximations. Baseline paths are simulated on a finer grid \(\Delta'\) with \(1000\) time intervals, while approximated paths are simulated on the coarser grid \(\Delta\) and then extended by holding values constant within each subinterval. The RMSEs are defined as
\begin{equation}
    \mathrm{RMSE}_X := \frac{\sum_{i\in[N],t\in\Delta',m\in[N_{\mathrm{test}}]}(\hat{X}^{i,m}_t - \tilde{X}^{i,m}_t)^2}{\sum_{i\in[N],t\in\Delta',m\in[N_{\mathrm{test}}]}(\hat{X}^{i,m}_t)^2},\quad \mathrm{RMSE}_\alpha := \frac{\sum_{i\in[N],t\in\Delta',m\in[N_{\mathrm{test}}]}(\hat{\alpha}^{i,m}_t - \tilde{\alpha}^{i,m}_t)^2}{\sum_{i\in[N],t\in\Delta',m\in[N_{\mathrm{test}}]}(\hat{\alpha}^{i,m}_t)^2}.
\end{equation}
Smaller values indicate better pathwise approximation of equilibrium states and strategies.

To quantify the Monte Carlo estimation error and assess robustness, under a fixed trained model, we report the sample mean and standard deviation of the RMSEs computed over \(20\) independent evaluation runs, each based on \(N_{\mathrm{test}} = 5000\) paths generated with different random seeds.

\smallskip

\noindent\textbf{Graph choices.} We test on the following four representative graphs:
\begin{align} &K_N:\ \text{the \(N\)-vertex complete graph},\\ &C_N:\ \text{the \(N\)-vertex cycle graph},\quad S_N:\ \text{the \(N\)-vertex star graph},\\ &\mathrm{RST}_N:\ \text{the \(N\)-vertex uniformly random spanning tree},
\end{align}
which serve distinct purposes: the complete graph provides a sanity check with known solutions \cite{carmona2013mean}; the cycle graph while sparse and vertex-transitive,  tests long-range dependency when $K$ is chosen small; the star graph combines sparsity with strong asymmetry; and random spanning trees represent the sparsest connected graphs with randomized asymmetric structure. Only \(\mathrm{RST}_N\) is random; once sampled, it is held fixed throughout all experiments.

We remark that numerical results on complete graphs \(K_N\) are only reported in Section~\ref{sec:numerics-lq} to verify the correctness of our numerical implementation. On complete graphs, the NTM architecture does not contain non-trainable entries and effectively reduces to a fully parameterized model, yielding performance comparable to that of FNN with minimal parameter saving.
Therefore, these cases do not provide additional insight into the advantages of NTM and are omitted for brevity.

For completeness, we also examine the performance of GCN-based architectures for solving games on graphs; this analysis is presented in Appendix~\ref{app:GCN_fail}.

\subsection{Linear-Quadratic Stochastic Differential Games}\label{sec:numerics-lq}

We consider an $N$-agent linear-quadratic game on graphs~\eqref{eqn:LQ_dynamics}--\eqref{eqn:LQ_terminal_cost}, where direct agent interactions are restricted to \(1\)-neighborhoods, i.e., \(\ell = 1\). Agents are heterogeneous and exhibit mean-reversion behavior, seeking to align their private states with the averages of their direct neighbors. 
See \cite[Section~2.2]{hu2024finite} for the derivations of the benchmark solution. 

\smallskip

\noindent\textbf{Numerical results.} We adopt the following model parameters:
\begin{equation}
    N=10,\quad T = 1.0,\quad a=0.1,\quad \sigma=0.5,\quad q = 0,\quad \EPS = 1.0,\quad c = 1.0,\quad \delta_0 = 0.5,
\end{equation}
with initial state \(X_0^i \overset{\mathrm{i.i.d.}}{\sim} U(-\delta_0,\delta_0),\ \forall i\in[N]\) independent of the Brownian motions. Empirically, random initialization of $X_0$ encourages exploration and aids learning near time 0.

\begin{table}[ht!]
\centering
\begin{threeparttable}
\caption{\(\mathrm{RMSE}_X\) and \(\mathrm{RMSE}_\alpha\) for the LQ game on graphs in Section~\ref{sec:numerics-lq}}
\label{tab:LQ_RMSE}
\centering
  \begin{tabular}{c|cccccc}
    \toprule
    \multirow{2}{*}{Problem} &
      \multicolumn{2}{c}{\(G = C_{10}\)} &
      \multicolumn{2}{c}{\(G = S_{10}\)} &
      \multicolumn{2}{c}{\(G = \mathrm{RST}_{10}\)}\\
    & \(\mathrm{RMSE}_X\) & \(\mathrm{RMSE}_\alpha\) & \(\mathrm{RMSE}_X\) & \(\mathrm{RMSE}_\alpha\) & \(\mathrm{RMSE}_X\) & \(\mathrm{RMSE}_\alpha\)\\
      \midrule
   DP 
& \makecell{\(1.89e{-2}\)\\ \((1.2e{-4})\)} 
& \makecell{\(2.69e{-2}\)\\ \((2.0e{-4})\)} 
& \makecell{\(2.10e{-2}\)\\ \((1.4e{-4})\)} 
& \makecell{\(2.47e{-2}\)\\ \((1.5e{-4})\)} 
& \makecell{\(1.89e{-2}\)\\ \((1.3e{-4})\)} 
& \makecell{\(2.68e{-2}\)\\ \((1.7e{-4})\)} \\

NTM-DP 
& \makecell{\(1.88e{-2}\)\\ \((1.2e{-4})\)} 
& \makecell{\(2.57e{-2}\)\\ \((1.9e{-4})\)} 
& \makecell{\(2.09e{-2}\)\\ \((1.4e{-4})\)} 
& \makecell{\(2.36e{-2}\)\\ \((1.5e{-4})\)} 
& \makecell{\(1.88e{-2}\)\\ \((1.3e{-4})\)} 
& \makecell{\(2.56e{-2}\)\\ \((1.7e{-4})\)} \\

DBSDE 
& \makecell{\(1.88e{-2}\)\\ \((1.2e{-4})\)} 
& \makecell{\(3.44e{-2}\)\\ \((2.0e{-4})\)} 
& \makecell{\(2.08e{-2}\)\\ \((1.4e{-4})\)} 
& \makecell{\(3.12e{-2}\)\\ \((1.5e{-4})\)} 
& \makecell{\(1.88e{-2}\)\\ \((1.3e{-4})\)} 
& \makecell{\(3.46e{-2}\)\\ \((1.7e{-4})\)} \\

NTM-DBSDE 
& \makecell{\(1.88e{-2}\)\\ \((1.2e{-4})\)} 
& \makecell{\(2.91e{-2}\)\\ \((2.0e{-4})\)} 
& \makecell{\(2.08e{-2}\)\\ \((1.4e{-4})\)} 
& \makecell{\(2.57e{-2}\)\\ \((1.7e{-4})\)} 
& \makecell{\(1.88e{-2}\)\\ \((1.3e{-4})\)} 
& \makecell{\(2.83e{-2}\)\\ \((1.8e{-4})\)} \\
    \bottomrule
  \end{tabular}
\begin{tablenotes}
\footnotesize
\item Values outside (resp. inside) parentheses denote the sample mean (resp. sample standard deviation) computed over \(20\) independent evaluation runs under a fixed trained model.
\end{tablenotes}
\end{threeparttable}
\end{table}

Table~\ref{tab:LQ_RMSE} compares the RMSEs of equilibrium state and strategy trajectories across algorithms and graph structures. 
Consistently small standard deviations indicate low variability in the Monte Carlo estimates, thereby supporting the reliability of the RMSE evaluation.
The results show that, regardless of the graph, DP (DBSDE) achieves essentially the same RMSE as NTM-DP (NTM-DBSDE), indicating that the NTM architecture performs on par with the standard FNN.

\begin{figure}[htpb]
    \centering
    \includegraphics[width=\linewidth]{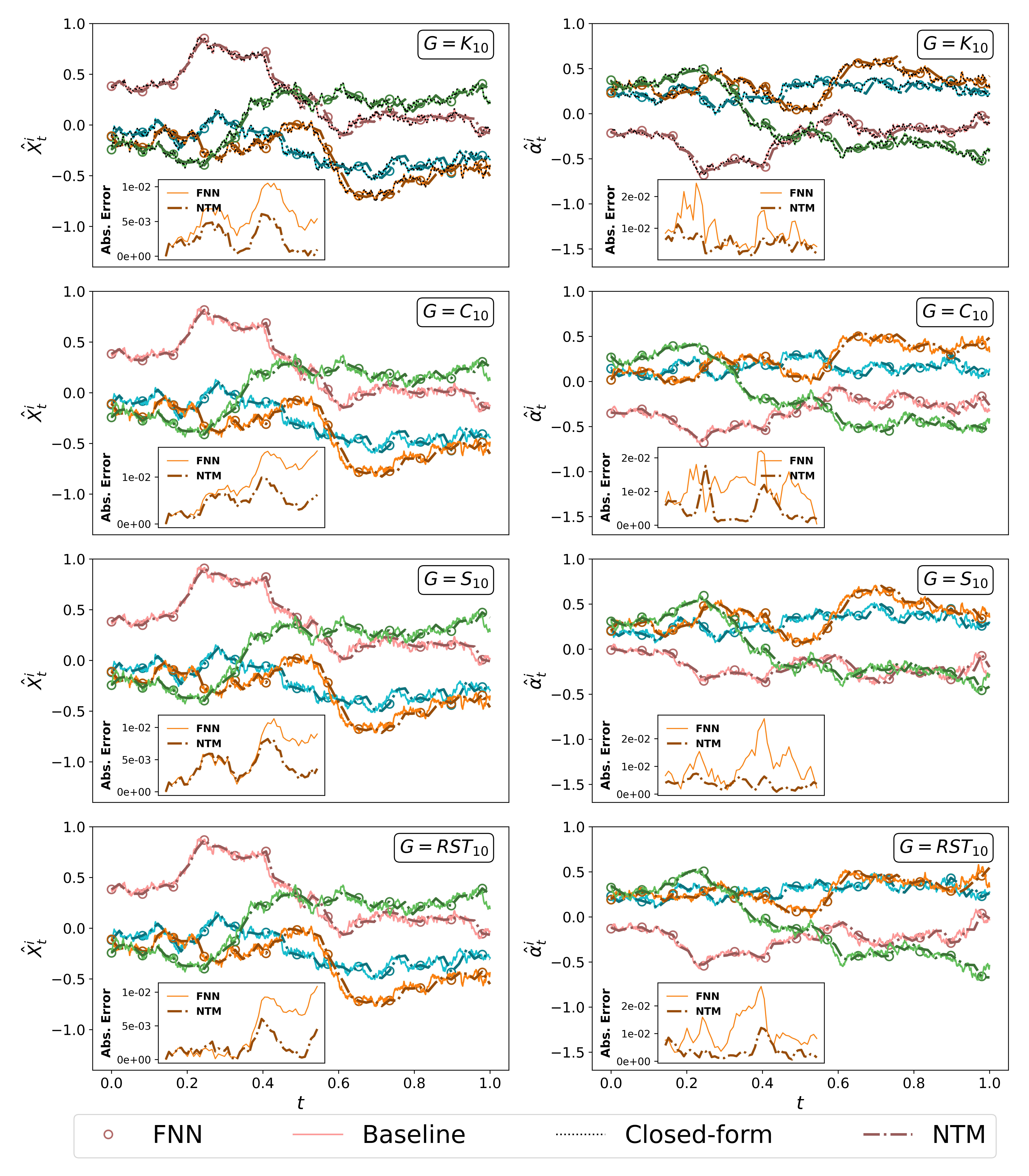}
    \caption{Comparisons of equilibrium state (left panels) and strategy (right panels) trajectories for \(N=10\) players in the LQ game on various graphs (cf. Section~\ref{sec:numerics-lq}).
    In the main panels, the colored solid lines represent the baseline solution, the dark-colored circles are obtained by DP, and the dark-colored dotted lines are obtained by NTM-DP. 
    For complete graph, the black dots represent the closed-form solution given in \cite{carmona2013mean}.
    Different players are distinguished by different colors.
    For clarity, only the trajectories of four players (indexed by 1, 2, 3, 4) are shown.
    In the inset panels, the solid (resp. dotted) lines show the absolute error of the trajectories generated by DP (resp. NTM-DP).
    For clarity, only the error curves for one representative player are displayed, and the strategy error is smoothed using a moving average with a window size of \(3\).
    }
    \label{fig:LQ_DP}
\end{figure}

Figure~\ref{fig:LQ_DP} further illustrates this comparison by plotting equilibrium state and strategy trajectories for four representative players.
Subplots in the first row serve as sanity checks, confirming the alignment among the FNN, NTM, and baseline trajectories, as well as the closed-form solution \cite{carmona2013mean}, on the complete graph.
Notably, in the remaining subplots of Figure~\ref{fig:LQ_DP}, both FNN and NTM trajectories lie close to their baseline counterparts, indicating similar approximation accuracy.
The absolute errors are further illustrated by the inset panels, which plot the error curves for one representative player; for visualization purposes, the strategy error curves are smoothed using a moving average with a window size of \(3\).
These plots show that FNN and NTM achieve comparable error levels in solving the game, while NTM uses substantially fewer parameters on sparse graphs.
Specifically, for $K = 3$ and $M = 3$, FNN requires 385 parameters per time step $t \in \Delta$, whereas NTM uses fewer than 190, amounting to only \textbf{49.35\%} of the parameters while retaining interpretability.

Additional results, including additional evaluation metrics and (NTM-)DBSDE trajectory plots, are provided in Appendix~\ref{app:M3_LQ}. These results support the same qualitative conclusions. Interestingly, NTM also performs well numerically with channel width $M=1$, even though Theorem~\ref{thm:express} requires $M\geq 2$. Appendix~\ref{app:M1_LQ} presents these results, where NTM uses fewer than 70 parameters per time step, only \textbf{18.18\%} of the parameters of an FNN.

\subsection{Multi-Agent Portfolio Games}\label{sec:numerics-portfolio}

We consider an $N$-agent game in which agents trade in a Black–Scholes market subject to common noise and share a common investment horizon. Agents are heterogeneous, endowed with constant absolute risk aversion (CARA) preferences with different parameters, and aim to maximize their expected utility based on the performance relative to that of their \(1\)-neighbors. 
In what follows, we first present the model setup and derive benchmark solutions, followed by numerical results.

\smallskip

\noindent\textbf{Model Setup.} Each agent $i\in[N]$ trades between a riskless bond earning interest rate $r$ and a private risky asset  $\{S_t^i\}$ evolving as:
\begin{equation}
    \ud S^i_t= S_t^i(\mu_i\ud t + \nu_i\ud W^i_t + \sigma_i\ud W_t^0),
\end{equation}
where $\{W^i_t\}$ denotes an idiosyncratic noise specific to asset $\{S_t^i\}$, and $\{W_t^0\}$ represents the common noise shared across all assets.
Here, \(\mu_i\in\R\) is the mean return of asset \(S^i\), \(\nu_i>0\) the idiosyncratic volatility, and \(\sigma_i>0\) the volatility from the common noise. Without loss of generality, we set $r = 0$.

Let the Markovian strategy \(\alpha^i_t\) be the dollar amount agent \(i\) invests in \(S^i\) at time \(t\).
Under the self-financing condition, the wealth process $\{X^i_t\}$ satisfies
\begin{equation}
    \label{eqn:port_state}
    \ud X^i_t = \mu_i\alpha^i_t\ud t + \nu_i\alpha^i_t\ud W^i_t + \sigma_i\alpha^i_t\ud W_t^0.
\end{equation}
Each agent maximizes the expected utility of its terminal wealth relative to direct neighbors:
\begin{equation}  
    \sup_\alpha \EE[U_i(X_T^i, \BAR{X}_T^{(-i)})]
    ,\quad \text{where}  \quad   \BAR{X}_T^{(-i)} := \frac{1}{\sqrt{d_{v_i}}}\sum_{j:v_j\sim v_i}\frac{1}{\sqrt{d_{v_j}}}X^j_T = X_T^i - e_i\transpose L X_T,
\end{equation}
under a CARA utility function of the form \(U_i(x,y) := - e^{-\frac{1}{\delta_i}(x-\theta_i y)}\), where \(\delta_i>0\) is the reciprocal of the risk aversion level of agent \(i\), and \(\theta_i>0\) characterizes the individual level of competition or complementary effects. This setup aligns with the general formulation (cf. \eqref{eqn:state_dynamics} and \eqref{eqn:J}):
\begin{equation}
    b^i = \mu_i \alpha, \quad \sigma^i = \nu_i\alpha , \quad \sigma^i_0 = \sigma_i\alpha, \quad f^i \equiv 0, \quad  g^i = \exp\left\{-\frac{1}{\delta_i}[(1-\theta_i)x^i + \theta_i e_i\transpose Lx]\right\}.
\end{equation}

\noindent\textbf{HJB system and benchmark solution.} 
Let \(v^i:[0,T]\times\R^N\to\R\) be the value function of agent \(i\).
By the dynamic programming principle (DPP), it satisfies the HJB system
\begin{equation}
    \label{eqn:port_HJB}
    \partial_t v^i + \inf_{\alpha^i}\SET{\sum_{k=1}^N\mu_k\alpha^k\partial_{x^k} v^i + \frac{1}{2}\sum_{k=1}^N(\nu_k\alpha^k)^2\partial_{x^kx^k} v^i + \frac{1}{2}\sum_{k=1}^N\sum_{j=1}^N\sigma_j\sigma_k\alpha^j\alpha^k\partial_{x^jx^k}v^i} = 0,
\end{equation}
with terminal condition $v^i(T,x) = g^i(x)$. Solving the infimum yields equilibrium conditions for \(\hat{\alpha}^i\):
\begin{equation}
    \label{eqn:port_equilibrium}
    \hat{\alpha}^i(\nu_i^2 + \sigma_i^2)\partial_{x^ix^i}v^i = -\mu_i\partial_{x^i} v^i - \sum_{j\neq i}\hat{\alpha}^j\sigma_j\sigma_i\partial_{x^jx^i}v^i,\ \forall i\in[N].
\end{equation}
Using an exponential ansatz \(v^i(t,x) = \rho^i_te^{-\frac{1}{\delta_i}(x^i - \theta_i\BAR{x}^{(-i)})}\), where \(\BAR{x}^{(-i)}:=x^i - e_i\transpose Lx\) and \(\rho^i:[0,T]\to\R\) is deterministic and measurable in $t$, we obtain a system for $\rho_t^i$:
\begin{equation}
    \dot{\rho}^i_t + \rho^i_t\left[-\frac{1}{\delta_i}(\mu_i\hat{\alpha}^i - \theta_i\widehat{\mu\alpha}) + \frac{1}{2\delta_i^2}((\nu_i\hat{\alpha}^i)^2 - \theta_i^2\widehat{(\nu\alpha)^2}) + \frac{1}{2\delta_i^2}(\sigma_i\hat{\alpha}^i - \theta_i\widehat{\sigma\alpha})^2\right] = 0,
\end{equation}
with terminal condition \(\rho^i_T = 1\), where \(\widehat{\mu\alpha}:= -\sum_{j\neq i}\mu_j (e_j\transpose Le_i)\hat{\alpha}^j(t,x)\), \( \widehat{\sigma\alpha}:= -\sum_{j\neq i}\sigma_j (e_j\transpose Le_i)\hat{\alpha}^j(t,x)\)
, and \(\widehat{(\nu\alpha)^2}:= -\sum_{j\neq i}(\nu_j)^2 (e_j\transpose Le_i)[\hat{\alpha}^j(t,x)]^2\).
The solution to the HJB system~\eqref{eqn:port_HJB} and the NE exist provided the linear system is solvable:
\begin{equation}
    (\nu_i^2 + \sigma_i^2)\hat{\alpha}^i(t,x) = \delta_i\mu_i + \sigma_i\theta_i\widehat{\sigma\alpha},\ \forall i\in[N].
\end{equation}
The resulting constant (in $(t,x)$) NE is referred to as the baseline equilibrium strategy.

\medskip
\noindent\textbf{Numerical results.}
Numerical results are demonstrated based on the following model parameters:
\begin{align}
    &N=10,\quad T = 1.0,\quad \mu_i\overset{\mathrm{i.i.d.}}{\sim} U(0.05,0.1),\quad \nu_i\overset{\mathrm{i.i.d.}}{\sim} U(0.2,0.25),\\
    &\delta_0 = 0.3,\quad  \sigma_i\overset{\mathrm{i.i.d.}}{\sim} U(0.15,0.2), \quad \delta_i\overset{\mathrm{i.i.d.}}{\sim} U(0.8,1.2),\quad \theta_i\overset{\mathrm{i.i.d.}}{\sim} U(0.4,0.6),
\end{align}
with i.i.d initial state \(X_0^i \overset{\mathrm{i.i.d.}}{\sim} U(-\delta_0,\delta_0),\ \forall i\in[N]\) independent of the Brownian motions.
The parameters are sampled at initialization and fixed throughout training and evaluation. Since the game features controlled diffusion coefficients, the Deep BSDE algorithm is not directly applicable; we present results only for (NTM-)DP.

\begin{table}[ht!]
\centering
\begin{threeparttable}
\caption{\(\mathrm{RMSE}_X\) and \(\mathrm{RMSE}_\alpha\) for portfolio games on graphs in Section~\ref{sec:numerics-portfolio}}
\label{tab:port_RMSE}
\centering
  \begin{tabular}{c|cccccc}
    \toprule
    \multirow{2}{*}{Problem} &
      \multicolumn{2}{c}{\(G = C_{10}\)} &
      \multicolumn{2}{c}{\(G = S_{10}\)} &
      \multicolumn{2}{c}{\(G = \mathrm{RST}_{10}\)}\\
    & \(\mathrm{RMSE}_X\) & \(\mathrm{RMSE}_\alpha\) & \(\mathrm{RMSE}_X\) & \(\mathrm{RMSE}_\alpha\) & \(\mathrm{RMSE}_X\) & \(\mathrm{RMSE}_\alpha\)\\
      \midrule
    DP 
& \makecell{\(1.79e{-2}\)\\ \((1.2e{-4})\)} 
& \makecell{\(9.63e{-3}\)\\ \((4.5e{-5})\)} 
& \makecell{\(1.75e{-2}\)\\ \((1.1e{-4})\)} 
& \makecell{\(9.57e{-3}\)\\ \((4.2e{-5})\)} 
& \makecell{\(1.77e{-2}\)\\ \((1.2e{-4})\)} 
& \makecell{\(9.47e{-3}\)\\ \((4.3e{-5})\)} \\

NTM-DP 
& \makecell{\(1.67e{-2}\)\\ \((1.2e{-4})\)} 
& \makecell{\(7.57e{-3}\)\\ \((4.2e{-5})\)} 
& \makecell{\(1.70e{-2}\)\\ \((1.2e{-4})\)} 
& \makecell{\(8.56e{-3}\)\\ \((4.5e{-5})\)} 
& \makecell{\(1.67e{-2}\)\\ \((1.2e{-4})\)} 
& \makecell{\(7.68e{-3}\)\\ \((4.1e{-5})\)} \\
    \bottomrule
  \end{tabular}
\begin{tablenotes}
\footnotesize
\item Values outside (resp. inside) parentheses denote the sample mean (resp. sample standard deviation) computed over \(20\) independent evaluation runs under a fixed trained model.
\end{tablenotes}
\end{threeparttable}
\end{table}

Table~\ref{tab:port_RMSE} compares RMSEs of equilibrium state and strategy trajectories across algorithms and graph structures. Figure~\ref{fig:Port_DP} illustrates the equilibrium trajectories for four representative players. 
Based on the \(\mathrm{RMSE}_\alpha\) values and the right panels of Figure~\ref{fig:Port_DP}, it is clear that NTM-parameterized strategies exhibit less fluctuation compared to FNN parameterized ones.
This provides empirical evidence of a statistical benefit of NTM's lower model complexity, complementary to its potential computational advantages, namely, the reduced risk of overfitting, which is particularly obvious when NE strategies admit a simple (e.g., constant) structure.

Additional evaluation metrics are reported in Appendix~\ref{app:M3_Port}, leading to the same qualitative conclusions. Results for the NTM architecture with $M=1$ are provided in Appendix~\ref{app:M1_Port}.
Notably, the NTM with \(M=1\) achieves a much more significant reduction in \(\mathrm{RMSE}_\alpha\), demonstrating a larger benefit by having an even lower model complexity.

\begin{figure}[htbp]
    \centering
    \includegraphics[width=\linewidth]{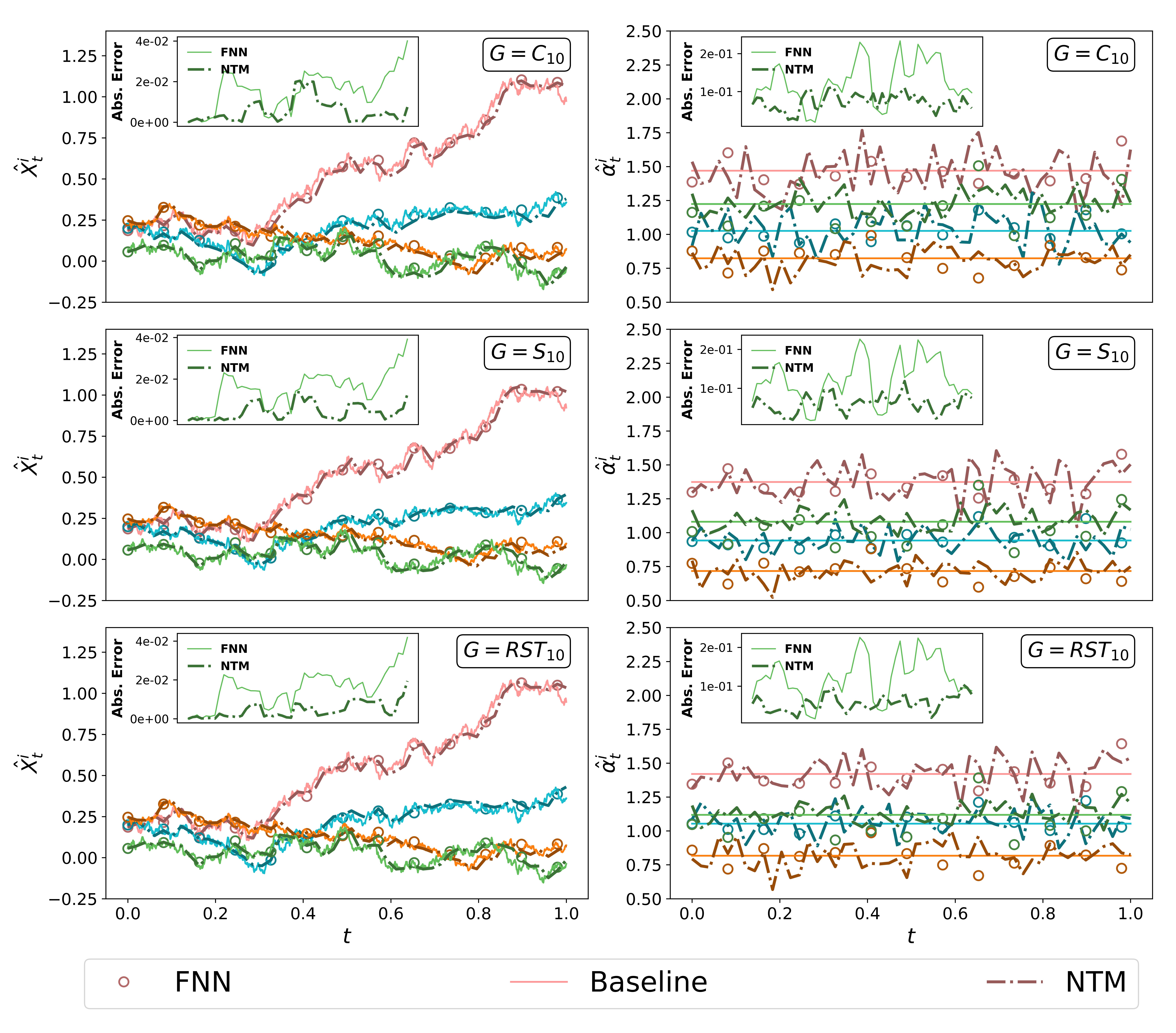}
    \caption{Comparisons of equilibrium state (left panels) and strategy (right panels) trajectories for \(N=10\) players in the Portfolio game on various graphs (cf. Section~\ref{sec:numerics-portfolio}).
    In the main panels, the colored solid lines represent the baseline solution, the dark-colored circles are obtained by DP, and the dark-colored dotted lines are obtained by NTM-DP. 
    Different players are distinguished by different colors.
    For clarity, only the trajectories of four players (indexed by 1, 4, 6, 9) are shown.
    In the inset panels, the solid (resp. dotted) lines show the absolute error of the trajectories generated by DP (resp. NTM-DP).
    For clarity, only the error curves for one representative player are displayed, and the strategy error is smoothed using a moving average with a window size of \(3\).
    }
    \label{fig:Port_DP}
\end{figure}

\subsection{A Variant of the LQ Game in Section~\ref{sec:numerics-lq}}\label{sec:numerics_nonlq}

We consider a non-LQ variant of the $N$-agent game on graphs presented in Section~\ref{sec:numerics-lq}. As before, direct interactions are restricted to 1-neighborhoods. The agents are heterogeneous and exhibit mean reversion behavior, aiming to keep their private states close to the average of their neighbors.

The cost functionals of this non-LQ variant follow \cite[equations~(2.4)--(2.5)]{hu2024finite}, while the state dynamics of each player \(i\in[N]\) take the form \eqref{eqn:state_dynamics} with drift and diffusion coefficients:
\begin{equation}
    b^i =  a\Big[\frac{1}{\sqrt{d_{v_i}}}\sum_{j:v_j\sim v_i}\frac{1}{\sqrt{d_{v_j}}}x^j - x^i\Big]^3 + \alpha, \quad \sigma^i \equiv \sigma, \quad \sigma_0^i \equiv 0,
\end{equation}
where \(a\geq 0\) and \(\sigma>0\).
The key difference from the LQ model is that, the mean reversion term is cubic rather than linear.
Thus, when \(X^i_t\) is far from the neighborhood average, the mean reversion effect is stronger; when close, it is weaker.

Following the standard dynamic programming approach, one can still derive the HJB system satisfied by the value function. Unlike the LQ case, this system cannot be reduced to coupled ODEs. Nevertheless, this variant has been investigated in \cite{han2020deep}, where a numerical baseline was constructed using the FNN architecture. We therefore take the baseline equilibrium strategy to be the neural network–parameterized strategy obtained by DP/DBSDE under FNN.

\medskip
\noindent\textbf{Numerical results.} We adopt the following parameters:
\begin{equation}
    N=10,\quad T = 1.0,\quad a=0.1,\quad \sigma=0.5,\quad q = 0,\quad \EPS = 1.0,\quad c = 1.0,\quad \delta_0 = 0.5,
\end{equation}
with initial states \(X_0^i\overset{\mathrm{i.i.d.}}{\sim} U(-\delta_0,\delta_0),\ \forall i\in[N]\) independent of the Brownian motions.

\begin{table}[ht!]
\centering
\begin{threeparttable}
\caption{\(\mathrm{RMSE}_X\) and \(\mathrm{RMSE}_\alpha\) for non-LQ games on graphs in Section~\ref{sec:numerics_nonlq}}
\label{tab:nonLQ_RMSE}
\centering
  \begin{tabular}{c|cccccc}
    \toprule
    \multirow{2}{*}{Problem} &
      \multicolumn{2}{c}{\(G = C_{10}\)} &
      \multicolumn{2}{c}{\(G = S_{10}\)} &
      \multicolumn{2}{c}{\(G = \mathrm{RST}_{10}\)}\\
    & \(\mathrm{RMSE}_X\) & \(\mathrm{RMSE}_\alpha\) & \(\mathrm{RMSE}_X\) & \(\mathrm{RMSE}_\alpha\) & \(\mathrm{RMSE}_X\) & \(\mathrm{RMSE}_\alpha\)\\
      \midrule
    NTM-DP 
& \makecell{\(2.00e{-4}\)\\ \((9.6e{-7})\)} 
& \makecell{\(2.64e{-3}\)\\ \((9.0e{-6})\)} 
& \makecell{\(1.58e{-4}\)\\ \((8.6e{-7})\)} 
& \makecell{\(2.15e{-3}\)\\ \((1.3e{-5})\)} 
& \makecell{\(1.72e{-4}\)\\ \((1.1e{-6})\)} 
& \makecell{\(2.49e{-3}\)\\ \((1.2e{-5})\)} \\

NTM-DBSDE 
& \makecell{\(1.40e{-4}\)\\ \((7.9e{-7})\)} 
& \makecell{\(1.23e{-2}\)\\ \((3.4e{-5})\)} 
& \makecell{\(8.50e{-5}\)\\ \((1.0e{-6})\)} 
& \makecell{\(1.00e{-2}\)\\ \((3.0e{-5})\)} 
& \makecell{\(1.43e{-4}\)\\ \((1.1e{-6})\)} 
& \makecell{\(1.18e{-2}\)\\ \((3.9e{-5})\)} \\
    \bottomrule
  \end{tabular}
  \begin{tablenotes}
\footnotesize
\item Values outside (resp. inside) parentheses denote the sample mean (resp. sample standard deviation) computed over \(20\) independent evaluation runs under a fixed trained model.
\end{tablenotes}
\end{threeparttable}
\end{table}

Table~\ref{tab:nonLQ_RMSE} compares RMSE values for equilibrium state and strategy trajectories across algorithms and graphs. The results show that, regardless of the graph structure, equilibrium states and strategies computed via NTM-DP (NTM-DBSDE) are very close to those computed by DP (DBSDE), confirming that NTM matches the performance of FNN.

\begin{figure}[!htbp]
    \centering
    \includegraphics[width=\linewidth]{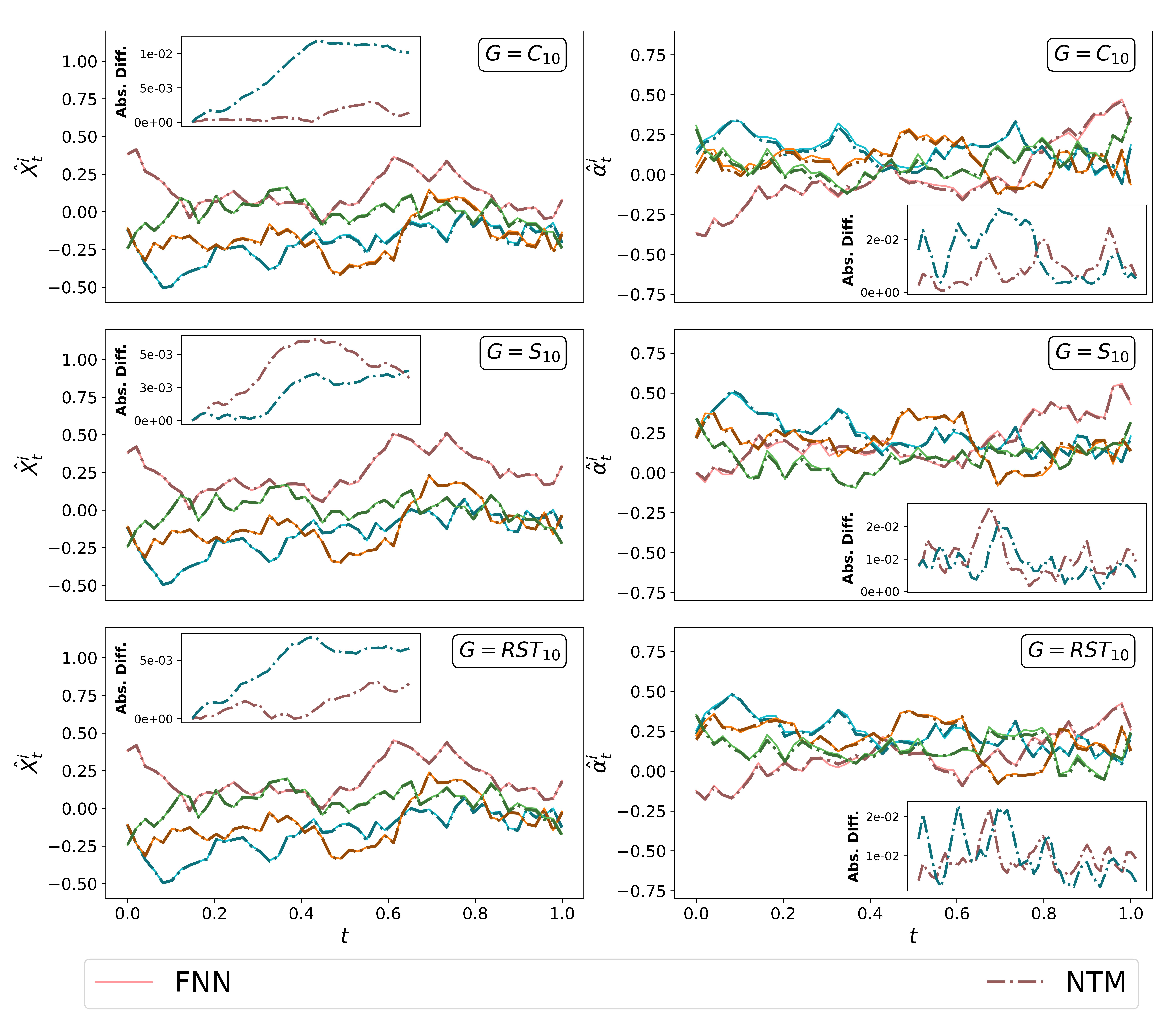}
    \caption{Comparisons of equilibrium state (left panels) and strategy (right panels) trajectories for \(N=10\) players in the variant of the LQ game (cf. Section~\ref{sec:numerics_nonlq}) on various graphs.
    In the main panels, the colored solid lines represent the baseline solution obtained by DP, and the dark-colored dotted lines are obtained by NTM-DP. 
    Different players are distinguished by different colors.
    For clarity, only the trajectories of four players (indexed by 1, 2, 3, 4) are shown.
    In the inset panels, the dotted lines show the absolute difference between trajectories generated by DP and NTM-DP.
    For clarity, only the error curves for two representative players are displayed, and the strategy error is smoothed using a moving average with a window size of \(3\).
    }
    \label{fig:nonLQ_DP}
\end{figure}

Figure~\ref{fig:nonLQ_DP} further illustrates this comparison for four representative players. Additional metrics and (NTM-)DBSDE trajectory plots are provided in Appendix~\ref{app:M3_non_LQ}, showing qualitatively similar conclusions. Results for the case $M=1$ are given in Appendix~\ref{app:M1_non_LQ}.

\section{Conclusion and Future Studies}\label{sec:conclusion_future}

In this paper, we proposed a novel graph-based deep learning architecture, NTM, constructed by fixing selected entries of feedforward neural network parameters to be non-trainable.
The design is naturally motivated by information propagation on graphs, and can be interpreted in terms of decision-making in games on graphs.
Our theoretical analysis establishes a universal approximation result for equilibrium feedback functions. Numerical experiments further show that, for strategy parameterization, NTM provides improved training stability over state-of-the-art GCNs. As an interpretable and parameter-efficient architecture, NTM can be readily combined with modern deep learning–based game-solving algorithms, and is broadly applicable to general models of games on graphs. Compared with a standard FNN, NTM requires fewer parameters on large sparse graphs, thereby reducing model complexity of the architecture and mitigating overfitting.

Several directions remain open for future study. A natural extension is to further explore NTM numerically, particularly in multi-hop ($\ell \geq 2$) and multi-dimensional test cases, and to investigate non-trainable variants of other game-solving algorithms beyond DP and Deep BSDE.
We emphasize that the numerical results in this paper are primarily validated on low-dimensional test cases (e.g., \(N=10\)). Extending the proposed approach to larger-scale settings and systematically investigating its scalability remain important directions for future work.
On the theoretical side, our explanation for the observed instability of GCNs is hypothetical, and fundamental questions remain: despite recent progress in understanding GCNs' transductive learning ability \cite{detering2025learning}, their function approximation capabilities and training instabilities have yet to be systematically analyzed. For NTM, the precise characterization of its expressivity class is still an open problem. Addressing these questions would yield valuable insights into the design of robust and efficient graph-based neural architectures.

Beyond the scope of this work, there is also room for broader exploration. 
To fully realize the practical advantages of the reduced number of parameters in NTM, such as improved computational efficiency and lower memory usage, it is essential to develop efficient training algorithms tailored to sparse neural network architectures (see, e.g., \cite{nikdan2023sparseprop}).
While NTM is motivated by deep fictitious play (DFP) for multi-agent interactions, other techniques, such as policy update \cite{han2020convergence}, may call for alternative graph-based architectures. Developing such models and comparing them with existing approaches offers a promising avenue for future research.

\bibliographystyle{plain}
%\bibliography{Reference.bib}

\section*{Acknowledgement}
R.H. was partially supported by the Simons Foundation (MP-TSM-00002783), the ONR grant under \#N00014-24-1-2432, and the NSF grant DMS-2420988.

\appendix
\mathtoolsset{showonlyrefs}

\renewcommand{\thesection}{\Alph{section}}

\section{Proofs of Theorems and Propositions}\label{app:proof}

\subsection{Proof of Theorem~\ref{thm:express}}\label{app:proof_exp}

\begin{proof}[Proof of Theorem~\ref{thm:express}]
We construct the approximation in two steps: first by analyzing the convergence of the fixed-point iteration under Assumption~\ref{assu:express}(i), and then by showing that the NTM architecture can approximate this iteration up to a desired accuracy under Assumption~\ref{assu:express}(ii).

\smallskip
\noindent
\textbf{Step 1: Fixed-point iteration and contraction.}
Define the iteration:
\[
    \phi^{(k+1)}(x) := A(x, \phi^{(k)}(x)), \quad \text{with } \phi^{(0)} \equiv 0.
\]
By Assumption~\ref{assu:express}(i), we have $
    \|\phi^{(k)}(x) - \widehat{\phi}(x)\|_\infty 
    \leq \rho \|\phi^{(k-1)}(x) - \widehat{\phi}(x)\|_\infty.
$
Therefore, by induction,
\begin{equation}
    \label{eqn:step1}
    \|\phi^{(K-1)}(x) - \widehat{\phi}(x)\|_\infty \leq \rho^{K-1} \|\widehat{\phi}(x)\|_\infty.
\end{equation}

\smallskip
\noindent\textbf{Step 2: NTM Approximation Error.}
Let \(\phi^{\mathrm{NTM}}_{K,M,G} := (\phi^{\mathrm{NTM}}_{1,K,M,G}, \ldots, \phi^{\mathrm{NTM}}_{N,K,M,G})\) denote the NTM architecture mapping \(\R^{Nd_\mathrm{in}} \to \R^{Nd_{\mathrm{out}}}\).  
We define the projections \(\pi^{(x)}\) and \(\pi^{(a)}\) to extract the state and action components of \(y \in \R^{d}\):
\[
\pi^{(x)}(y) := (y_1,\dots,y_{d_{\mathrm{in}}}) \in \R^{d_{\mathrm{in}}}, \quad \pi^{(a)}(y) := (y_{d_{\mathrm{in}}+1},\dots,y_d) \in \R^{d_{\mathrm{out}}}.
\]

We construct weights for the NTM (cf. \eqref{eqn:NTM_general_1}--\eqref{eqn:NTM_general_2}) layers \(k \in[K-1]\) and  \(r\in[M]\) such that for each agent \(p\in[N]\) and its neighbor \(v_q \in \mathcal{N}_G(v_p)\), the message weights are:
\begin{align}
W_{\mathrm{in},p} &= [I_{d_{\mathrm{in}}}\;  0_{d_\mathrm{in}\times d_\mathrm{out}}]\transpose, \quad b_{\mathrm{in},p} = 0_{d}, \quad
W_{pr,p}^{(k)} = \begin{bmatrix} c_r I_{d_{\mathrm{in}}} & 0 \\\gamma_{pr,p}^{(x)} & 0 \end{bmatrix},\\
W_{pr,q}^{(k)} &= \begin{bmatrix} 0 & 0 \\ \gamma_{pr,q}^{(x)} & \gamma_{pr,q}^{(a)} \end{bmatrix},\quad W_{\mathrm{out},i} = \big[0_{d_{\mathrm{out}}\times d_{\mathrm{in}}}\; I_{d_{\mathrm{out}}}\big],\quad W_{\mathrm{out},q} = 0_{d_{\mathrm{out}}\times d},\\
h_{pr}^{(k)} &= \mathrm{Concat}(0_{d_{\mathrm{in}}}, \eta_{pr}), \quad
g_{pr}^{(k)} = \mathrm{Concat}(c_r \mathbf{1}_{d_{\mathrm{in}}}, \beta_{pr}),\quad b^{(k)}_p = 0_{d},\quad b_{\mathrm{out}} = 0_{d_{\mathrm{out}}},
\end{align}
where \(p\neq q\) and \(i\neq q\). The constant \(c_r = 1\) for \(r=1\), \(c_r = -1\) for \(r=2\), and is zero otherwise.

Under ReLU activation \(\sigma\), and with input \(z^{(0)} = x\), we have:
\begin{align}\label{eqn:exp_state}
  \pi^{(x)}(z_p^{(k+1)}) &= \sum_{r=1}^M c_r\,\sigma\Big(c_r\,\pi^{(x)}(z_p^{(k)})\Big) = \pi^{(x)}(z_p^{(k)}) = x \quad \text{(constant across layers)},\\
\pi^{(a)}(z_p^{(k+1)}) &= \sum_{r=1}^M \beta_{pr}\odot \sigma\Big(\gamma^{(x)}_{pr,p}\;\pi^{(x)}(z_p^{(k)}) + \sum_{q:v_q\in\mathcal{N}_G(v_p)} \big[\gamma^{(x)}_{pr,q}\;\pi^{(x)}(z_q^{(k)}) + \gamma^{(a)}_{pr,q}\;\pi^{(a)}(z_q^{(k)})\big] + \eta_{pr}\Big). 
\end{align}
By Assumption~\ref{assu:express}(ii), this exactly matches the approximate best response:
    \begin{equation}
        \label{eqn:exp_action}
        \pi^{(a)}(z_p^{(k+1)}) = \widetilde A_p\Big(\pi^{(x)}(z_p^{(k)}),\pi^{(x)}(z_{\mathcal{N}(p)}^{(k)}),\pi^{(a)}(z_{\mathcal{N}(p)}^{(k)})\Big),\ \forall k\in[K-1],\ p\in[N].
    \end{equation}
Combining it with equation~\eqref{eqn:exp_state} yields
    \begin{equation}
        \label{eqn:exp_action_1}
        \pi^{(a)}(z_p^{(k+1)}) = \widetilde A_p\Big(x,x^{\mathcal{N}(p)},\pi^{(a)}(z_{\mathcal{N}(p)}^{(k)})\Big), \ \forall k\in[K-1],\ p\in[N].
    \end{equation}

Let \(\widetilde{\phi}^{(k)}\) denote the approximation iteration:
\[
\widetilde{\phi}^{(k+1)}(x) := \widetilde{A}(x, \widetilde{\phi}^{(k)}(x)), \quad \widetilde{\phi}^{(0)} \equiv 0.
\]
Then, by construction,
\[
\pi^{(a)}(z_p^{(k+1)})= \widetilde\phi_p^{(k)}(x),\ \forall k\in[K-1],\ p\in[N], \quad \text{ and }\quad \phi^{\mathrm{NTM}}_{K,M,G}(x) = \widetilde{\phi}^{(K-1)}(x),
\]
 where \(\widetilde\phi^{(k)} = \mathrm{Concat}(\widetilde\phi_1^{(k)},\ldots,\widetilde\phi_N^{(k)})\). 
    Using this relationship, we estimate the error:
  \begin{align}
        &\|\phi^{\mathrm{NTM}}_{K,M,G}(x) - {\phi}^{(K-1)}(x)\|_\infty
        = \|\widetilde\phi^{(K-1)}(x) - {\phi}^{(K-1)}(x)\|_\infty\\
        &= \|\widetilde A(x,\widetilde\phi^{(K-2)}(x)) -  A(x, \phi^{(K-2)}(x))\|_\infty\\
        &\leq \|\widetilde A(x,\widetilde\phi^{(K-2)}(x)) -  A(x, \widetilde\phi^{(K-2)}(x))\|_\infty + \| A(x,\widetilde\phi^{(K-2)}(x)) -  A(x, \phi^{(K-2)}(x))\|_\infty\\
        &\leq \delta + \rho\, \| \widetilde\phi^{(K-2)}(x)- \phi^{(K-2)}(x)\|_\infty.
    \end{align}
Iterating this inequality gives:
\begin{equation}
\label{eqn:step2}
\|\phi^{\mathrm{NTM}}_{K,M,G}(x) - \phi^{(K-1)}(x)\|_\infty 
= \|\widetilde{\phi}^{(K-1)}(x) - \phi^{(K-1)}(x)\|_\infty 
\leq \frac{\delta}{1 - \rho}.
\end{equation}
\end{proof}

\subsection{Proof of Proposition~\ref{prop:NE_approx}}\label{app:proof_approx}

\begin{proof}[Proof of Proposition~\ref{prop:NE_approx}]
In this proof, all matrix inequalities are understood in the positive semi-definite sense. Without loss of generality, assume \(\mathcal{K}\) is the closed unit ball in \(\R^N\).

    \smallskip

\noindent\textbf{Step 1: Algebraic structure of \(\mathscr{P}_L\).} Let $h(L)$ denote an analytic function of $L$, where \( h(x):= \sum_{n=0}^{\infty}a_n^hx^n\). We show that \(h(L)\in \mathscr{P}_L\).

    Since \(L\in\mathbb{S}^{N\times N}\), it has a characteristic polynomial \(\Lambda\) of degree \(N\). By the Cayley-Hamilton theorem, \(\Lambda(L) = 0\). Then, for any analytic function $h$, there exist polynomials \(q\) and \(r\) with \(\deg r\leq N-1\), such that \(h(x) = q(x)\Lambda(x) + r(x)\), implying \(h(L) = r(L)\in\mathscr{P}_L\).

    Using the same technique, one can show that \(\mathscr{P}_L\) is closed under matrix multiplication; that is, \(X_1 X_2\in\mathscr{P}_L,\ \forall X_1,X_2\in \mathscr{P}_L\). Thus, \(\mathscr{P}_L\) forms an algebra.
    In particular, \(R'(t), R(t) \in \mathscr{P}_L\)  for all \( t\in[0,T]\), by equation~\eqref{eqn:second_charac}.

    \smallskip
    \noindent\textbf{Step 2: Polynomial approximation of \([I+R(T-t)cL]^{-1}\).} By \cite[Theorem~4.13]{hu2024finite}, we have \(0\leq R'(t)\leq (1+cT)I\) and \( 0\leq R(t)\leq (T + \frac{c}{2}T^2)I,\ \forall t\in[0,T]\).
    Since \(0\leq L\leq 2I\) \cite{chung1997spectral}, it follows that \(0\leq R(T-t)L\leq (2T + cT^2)I\) has a compact spectrum.
    By the Stone-Weierstrass theorem, there exists  a polynomial $P_k(x)$ of degree \(k \in \mathbb{N}\) such that
    \begin{equation}
        \|P_k(R(T-t)L) - [I+R(T-t)cL]^{-1}\|< \frac{\delta}{2c(1+cT)}.
    \end{equation}

    \smallskip

    \noindent\textbf{Step 3: Approximation of the NE strategy.} By Theorem~\ref{thm:semi-explicit}, the Markovian NE for player \(i\) exists and is given by
    \begin{equation}
        \hat{\alpha}^i(t,x) = -qe_i\transpose L x - e_i\transpose F^i_t x,
    \end{equation}
    where, by \cite[Lemma~4.7(ii)]{hu2024finite} and equation~\eqref{eqn:closed_form_F},
    \begin{equation}
        e_i\transpose F^i_t x = e_i\transpose R'(T-t)cL[I+R(T-t)cL]^{-1}x.
    \end{equation}    
    With the construction of \(A_t(L)\) as follows:
    \begin{equation}
        A_t(L) :=  -qL-R'(T-t)cL\cdot P_k(R(T-t)L),
    \end{equation}
    it is clear that \(A_t(L)\in\mathscr{P}_L\).
    This is because \(R(T-t)L\in \mathscr{P}_L\), so that \(P_k(R(T-t)L)\in\mathscr{P}_L\) by Step~1.
    The approximation error is then bounded by
    \begin{equation}
        \sup_{\|x\|\leq 1}|\hat{\alpha}^i(t,x) - e_i\transpose A_t(L) x|\leq \|R'(T-t)cL\|\cdot \|P_k(R(T-t)L) - [I+R(T-t)cL]^{-1}\|< \delta,
    \end{equation}
    which concludes the proof.
\end{proof}

\section{Additional Results for Supervised Learning}\label{app:SVL}

This appendix provides additional numerical results for the supervised learning task in Section~\ref{sec:SVL}, reinforcing the conclusions drawn there.
Appendix~\ref{app:SVL_loss} presents training loss curves that demonstrate the instability of GCNs.
Appendix~\ref{app:SVL_approximation} shows function approximations produced by trained neural networks, offering a visual comparison of architectural performance.

\subsection{Training Loss}\label{app:SVL_loss}

We report training loss trajectories for various supervised learning test cases under multiple independent trials.
Since results are qualitatively similar across different graphs, we present the case of a cycle graph $G= C_N$, using the same hyperparameters as in Section~\ref{sec:SVL}.

\begin{figure}[htpb]
    \centering
    \includegraphics[width=\linewidth]{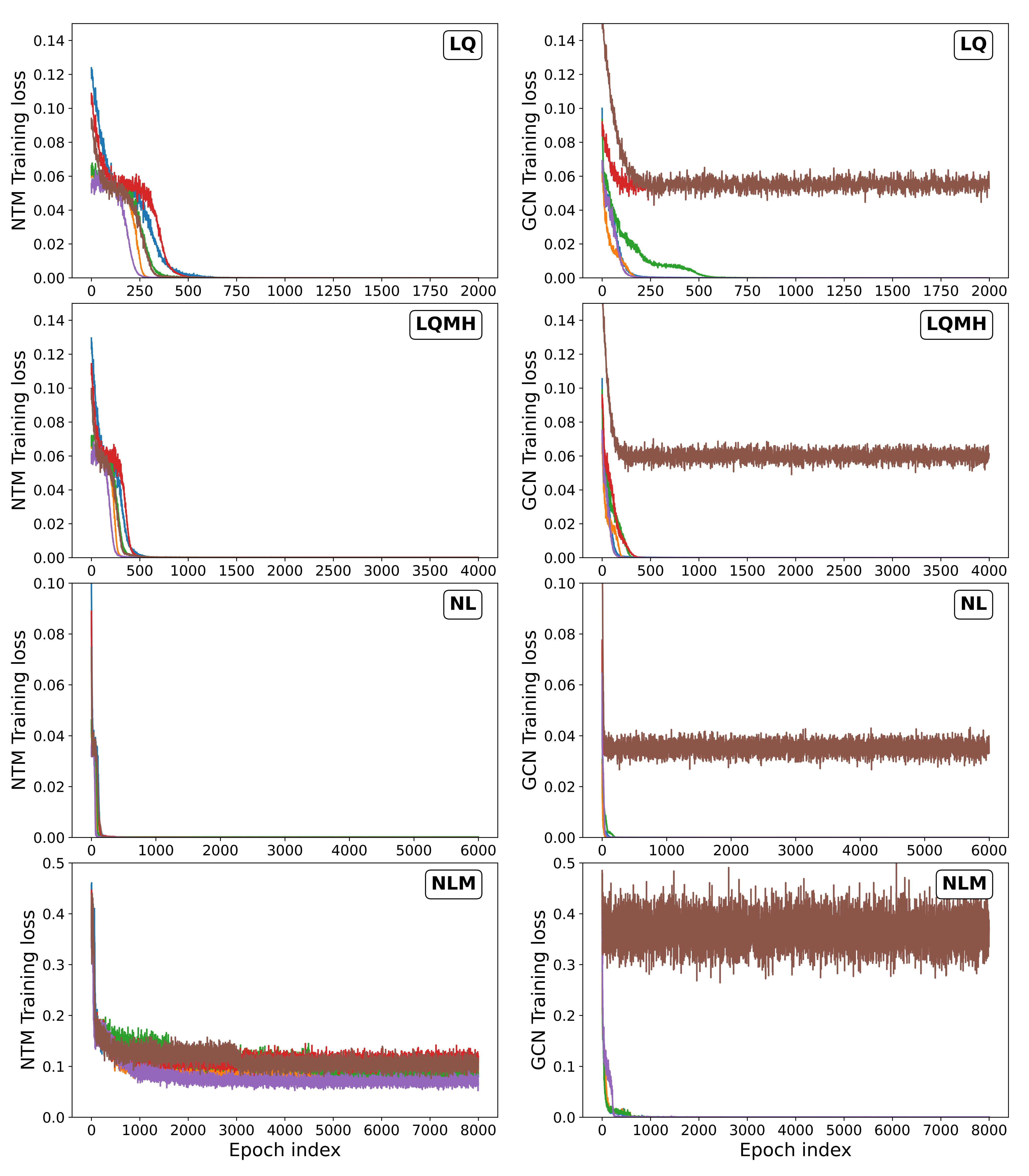}
    \caption{
 Training loss trajectories for NTM (left) and Chebyshev GCN (right) over 6 independent runs on test cases from Section~\ref{sec:SVL}, using the cycle graph $G = C_N$. Each color represents a different run. }      
    \label{fig:SVL_loss}
\end{figure}

Figure~\ref{fig:SVL_loss} compares the training loss curves of NTM \textit{vs.} Chebyshev GCN over \(6\) independent runs on \(G = C_N\).
NTM consistently converges (except in test case \textbf{NLM}, where limited expressivity is expected), while Chebyshev GCN often exhibits at least one divergent or stagnant trajectory per test case, indicating training instability.
These observations support our earlier claims regarding the expressivity of NTM and the robustness issues with GCNs.

\subsection{Function Approximation}\label{app:SVL_approximation}
We visualize the approximated function \(\tilde{f}(x)\) produced by trained neural networks in the supervised learning task with target function $f(x)$. The comparison includes four architectures: standard FNN, NTM with $M=3$, NTM with $M=1$, and a misspecified NTM:
\begin{itemize}
    \item The NTM architecture with \(M=3\) matches the architecture used in Section~\ref{sec:SVL}.
    \item The NTM with \(M=1\) uses Tanh activation function, potentially violating the condition of Theorem~\ref{thm:express}, but retains interpretability and uses fewer parameters.
    \item The misspecified NTM refers to \(\phi^{\mathrm{NTM}}_{i,K,M,G}\) with an incorrectly specified player index \(i=3\), while the learning targets in all test cases are analogous to player \(1\)'s feedback function in games on graphs (cf. Section~\ref{sec:SVL}). This setting evaluates the performance degradation caused by misaligned sparsification.
\end{itemize}

Since results are qualitatively consistent across different graphs, we use the cycle graph $G = C_N$, and apply the same hyperparameters as in Section~\ref{sec:SVL}.
Because these plots are based on a single training run, architectures with instability (e.g., GCNs) are excluded from this comparison.

\begin{figure}[htbp]
    \centering
    \includegraphics[width=\linewidth, height = 0.7\textheight]{FigSM2.jpg}
    \caption{Function approximation results for test cases \textbf{LQ, LQMH, NL, NLM} (from top to bottom; cf. Section~\ref{sec:SVL}) on the cycle graph \(G = C_{10}\).
    Each panel shows the mapping \(x_j\mapsto \tilde{f}(x)\) for \(j\in[N]\), with all other components \(x_1,\ldots,x_{j-1},x_{j+1},\ldots x_N\) sampled from \(U(0,1)\).}
    \label{fig:SVL_NLM_approx}
\end{figure}

Figure~\ref{fig:SVL_NLM_approx} displays the learned approximation \(\tilde{f}(x)\) as a function of each coordinate $x_j$, where \(x\in\R^N\).
Because the NTM used here has $K = 3$ and is constructed for player \(1\), it cannot capture the dependence on \(x_j\) for \(j\in\{5,6,7\}\), consistent with our discussion in Section~\ref{sec:NTM}.
For other values of \(j\), both NTM architectures perform well, except in test case \textbf{NLM}, where limited expressivity is expected (cf. Section~\ref{sec:SVL}). 
The misspecified NTM performs significantly worse across all test cases, illustrating the critical role of aligning sparsification with the graph structure.
This further supports the claim that NTM's effectiveness stems from its interpretable, graph-aware design.

\subsection{Summary Statistics associated with Figure~\ref{fig:SVL}}\label{app:summary_stat}

To facilitate a quantitative comparison of errors in the supervised learning experiments (cf. Section~\ref{sec:SVL}), we provide numerical summary statistics associated with the box plots in Figure~\ref{fig:SVL}.

Each of Tables~\ref{tab:SVL_cycle}--\ref{tab:SVL_RST} reports summary statistics of RMSE values over 1000 independent runs of the supervised learning experiments (Section~\ref{sec:SVL}) for all test cases on a given graph structure.
Within the tables, we report the bootstrap-based 95\% confidence interval (CI) for the median RMSE, obtained via $10000$ bootstrap samples, to account for the non-Gaussianity of the RMSE distribution.
The interquartile range (IQR) and the variance are reported to measure the dispersion of the RMSE distribution.
The \(2.5\%\) and \(97.5\%\) RMSE quantiles, denoted by \(Q_{2.5\%}\) and \(Q_{97.5\%}\), characterize the tail behavior and provide additional insight into the expressivity and training stability of the architectures.
While Figure~\ref{fig:SVL} displays box plots of log-RMSE values for visual clarity, the tables report summary statistics on the original RMSE scale to facilitate direct quantitative interpretation.
The conclusions drawn from these numerical results are consistent with those presented in Section~\ref{sec:SVL} and are therefore omitted for brevity.

\begin{table}[htbp]
\centering
\begin{threeparttable}
\caption{Summary statistics of RMSE values over 1000 runs for all test cases on \(G = C_{10}\).}
\label{tab:SVL_cycle}
\begin{tabular}{ll|llll}
\toprule
\makecell{\textbf{Test}\\ \textbf{case}} & \makecell{\textbf{Architec-}\\ \textbf{ture}} & \textbf{Median (95\% CI)} & \textbf{IQR} & \textbf{Variance} & $[Q_{2.5\%}, Q_{97.5\%}]$ \\
\midrule
\multirow{3}{*}{\makecell{\textbf{LQ}}}
& \makecell{FNN} &  $2.31e{-4}\ (2.27e{-4}, 2.36e{-4})$ & $6.45e{-5}$ & $2.98e{-9}$ & $[1.65e{-4},3.83e{-4}]$ \\
& \makecell{NTM} & $8.68e{-5}\ (8.59e{-5}, 8.86e{-5})$ & $1.32e{-4}$ & $6.97e{-3}$ & $[7.38e{-5},7.41e{-4}]$ \\
& \makecell{GCN} & $1.00e{-0}\ (1.00e{-0},1.00e{-0})$ & $1.00e{-0}$ & $2.49e{-1}$ & $[3.23e{-6},1.00e{-0}]$ \\
\midrule
\multirow{3}{*}{\makecell{\textbf{LQMH}}}
& \makecell{FNN} &  $1.37e{-4}\ (1.35e{-4}, 1.40e{-4})$ & $3.75e{-5}$ & $1.07e{-9}$ & $[1.03e{-4},2.33e{-4}]$ \\
& \makecell{NTM} & $3.05e{-3}\ (3.05e{-3}, 3.05e{-3})$ & $7.94e{-6}$ & $5.93e{-3}$ & $[3.04e{-3},1.29e{-2}]$ \\
& \makecell{GCN} & $1.00e{-0}\ (1.00e{-0},1.00e{-0})$ & $1.00e{-0}$ & $2.49e{-1}$ & $[7.07e{-7},1.00e{-0}]$ \\
\midrule
\multirow{3}{*}{\makecell{\textbf{NL}}}
& \makecell{FNN} &  $1.38e{-4}\ (1.29e{-4}, 1.47e{-4})$ & $2.23e{-4}$ & $1.27e{-7}$ & $[5.28e{-5},1.14e{-3}]$ \\
& \makecell{NTM} & $5.26e{-4}\ (5.06e{-4},5.38 e{-4})$ & $1.35e{-3}$ & $1.29e{-2}$ & $[1.69e{-4},7.95e{-3}]$ \\
& \makecell{GCN} & $1.00e{-0}\ (1.00e{-0},1.00e{-0})$ & $1.00e{-0}$ & $2.38e{-1}$ & $[8.15e{-6},1.00e{-0}]$ \\
\midrule
\multirow{3}{*}{\makecell{\textbf{NLM}}}
& \makecell{FNN} &  $1.04e{-2}\ (1.02e{-2},1.05e{-2})$ & $3.27e{-3}$ & $3.04e{-5}$ & $[8.38e{-3},2.53e{-2}]$ \\
& \makecell{NTM} & $2.29e{-1}\ (2.23e{-1}, 2.33e{-1})$ & $9.62e{-2}$ & $1.50e{-2}$ & $[1.26e{-1},4.71e{-1}]$ \\
& \makecell{GCN} & $1.00e{-0}\ (1.00e{-0},1.00e{-0})$ & $9.99e{-1}$ & $2.42e{-1}$ & $[3.13e{-4},1.00e{-0}]$ \\
\bottomrule
\end{tabular}
\end{threeparttable}
\end{table}

\begin{table}[htbp]
\centering
\begin{threeparttable}
\caption{Summary statistics of RMSE values over 1000 runs for all test cases on \(G = S_{10}\).}
\label{tab:SVL_star}
\begin{tabular}{ll|llll}
\toprule
\makecell{\textbf{Test}\\ \textbf{case}} & \makecell{\textbf{Architec-}\\ \textbf{ture}} & \textbf{Median (95\% CI)} & \textbf{IQR} & \textbf{Variance} & $[Q_{2.5\%}, Q_{97.5\%}]$ \\
\midrule
\multirow{3}{*}{\makecell{\textbf{LQ}}} & \makecell{FNN} & $1.29e{-4}\ (1.25e{-4}, 1.33e{-4})$ & $6.55e{-5}$ & $2.23e{-9}$ & $[6.14e{-5},2.35e{-4}]$ \\ & \makecell{NTM} & $7.12e{-6}\ (6.76e{-6}, 7.48e{-6})$ & $9.00e{-6}$ & $1.14e{-9}$ & $[1.31e{-6},1.46e{-4}]$ \\ & \makecell{GCN} & $1.05e{-4}\ (7.86e{-5},7.14e{-2})$ & $1.43e{-1}$ & $5.07e{-3}$ & $[2.95e{-6},1.43e{-1}]$\\
\midrule
\multirow{3}{*}{\makecell{\textbf{LQMH}}} & \makecell{FNN} & $5.72e{-5}\ (5.56e{-5}, 5.86e{-5})$ & $2.75e{-5}$ & $4.17e{-10}$ & $[2.95e{-5},1.05e{-4}]$ \\ & \makecell{NTM} & $1.31e{-6}\ (1.23e{-6}, 1.38e{-6})$ & $1.44e{-6}$ & $2.52e{-11}$ & $[2.42e{-7},2.19e{-5}]$ \\ & \makecell{GCN} & $4.89e{-5}\ (3.70e{-5},7.14e{-2})$ & $1.43e{-1}$ & $5.07e{-3}$ & $[1.33e{-6},1.43e{-1}]$ \\
\midrule
\multirow{3}{*}{\makecell{\textbf{NL}}} & \makecell{FNN} & $5.37e{-5}\ (4.89e{-5}, 5.92e{-5})$ & $8.29e{-5}$ & $8.07e{-9}$ & $[2.04e{-5},2.97e{-4}]$ \\ & \makecell{NTM} & $2.83e{-4}\ (2.76e{-4}, 2.91e{-4})$ & $1.09e{-4}$ & $1.71e{-7}$ & $[1.35e{-4},4.25e{-4}]$ \\ & \makecell{GCN} & $9.99e{-5}\ (4.78e{-5},1.15e{-1})$ & $1.15e{-1}$ & $3.33e{-3}$ & $[4.23e{-6},1.15e{-1}]$ \\
\midrule
\multirow{3}{*}{\makecell{\textbf{NLM}}} & \makecell{FNN} & $9.07e{-3}\ (8.86e{-3}, 9.29e{-3})$ & $4.16e{-3}$ & $7.28e{-5}$ & $[5.00e{-3},3.48e{-2}]$ \\ & \makecell{NTM} & $1.35e{-1}\ (1.33e{-1}, 1.38e{-1})$ & $3.80e{-2}$ & $8.38e{-4}$ & $[7.05e{-2},1.86e{-1}]$ \\ & \makecell{GCN} & $3.40e{-2}\ (1.68e{-2},9.70e{-1})$ & $9.69e{-1}$ & $2.32e{-1}$ & $[2.75e{-4},9.70e{-1}]$ \\
\bottomrule
\end{tabular}
\end{threeparttable}
\end{table}

\begin{table}[htbp]
\centering
\begin{threeparttable}
\caption{Summary statistics of RMSE values over 1000 runs for all test cases on \(G = K_{5,5}\).}
\label{tab:SVL_CB}
\begin{tabular}{ll|llll}
\toprule
\makecell{\textbf{Test}\\ \textbf{case}} & \makecell{\textbf{Architec-}\\ \textbf{ture}} & \textbf{Median (95\% CI)} & \textbf{IQR} & \textbf{Variance} & $[Q_{2.5\%}, Q_{97.5\%}]$ \\
\midrule
\multirow{3}{*}{\makecell{\textbf{LQ}}} & \makecell{FNN} & $2.19e{-4}\ (2.16e{-4}, 2.24e{-4})$ & $6.52e{-5}$ & $2.68e{-9}$ & $[1.53e{-4},3.51e{-4}]$ \\ & \makecell{NTM} & $1.28e{-5}\ (1.22e{-5}, 1.35e{-5})$ & $9.09e{-6}$ & $1.15e{-10}$ & $[3.74e{-6},3.36e{-5}]$ \\ & \makecell{GCN} & $1.00e{-0}\ (1.00e{-0},1.00e{-0})$ & $1.00e{-0}$ & $2.48e{-1}$ & $[1.85e{-6},1.00e{-0}]$\\
\midrule
\multirow{3}{*}{\makecell{\textbf{LQMH}}} & \makecell{FNN} & $1.32e{-4}\ (1.30e{-4}, 1.33e{-4})$ & $3.28e{-5}$ & $9.06e{-10}$ & $[9.89e{-5},2.17e{-4}]$ \\ & \makecell{NTM} & $3.78e{-6}\ (3.60e{-6}, 3.98e{-6})$ & $3.57e{-6}$ & $2.11e{-11}$ & $[1.21e{-6},1.89e{-5}]$ \\ & \makecell{GCN} & $1.00e{-0}\ (1.00e{-0},1.00e{-0})$ & $1.00e{-0}$ & $2.48e{-1}$ & $[6.25e{-7},1.00e{-0}]$ \\
\midrule
\multirow{3}{*}{\makecell{\textbf{NL}}} & \makecell{FNN} & $1.68e{-4}\ (1.57e{-4}, 1.77e{-4})$ & $2.43e{-4}$ & $1.43e{-7}$ & $[7.62e{-5},1.40e{-3}]$ \\ & \makecell{NTM} & $2.12e{-4}\ (2.01e{-4}, 2.23e{-4})$ & $2.12e{-4}$ & $5.79e{-8}$ & $[1.01e{-4},9.16e{-4}]$ \\ & \makecell{GCN} & $1.00e{-0}\ (1.00e{-0},1.00e{-0})$ & $1.00e{-0}$ & $2.40e{-1}$ & $[9.12e{-6},1.00e{-0}]$ \\
\midrule
\multirow{3}{*}{\makecell{\textbf{NLM}}} & \makecell{FNN} & $4.68e{-3}\ (4.61e{-3}, 4.76e{-3})$ & $1.22e{-3}$ & $1.04e{-6}$ & $[3.34e{-3},7.51e{-3}]$ \\ & \makecell{NTM} & $1.05e{-1}\ (1.02e{-1}, 1.08e{-1})$ & $5.14e{-2}$ & $1.51e{-3}$ & $[3.74e{-2},1.89e{-1}]$ \\ & \makecell{GCN} & $8.11e{-3}\ (5.31e{-3},1.00e{-0})$ & $9.97e{-1}$ & $2.49e{-1}$ & $[7.72e{-4},1.00e{-0}]$ \\
\bottomrule
\end{tabular}
\end{threeparttable}
\end{table}

\begin{table}[htbp]
\centering
\begin{threeparttable}
\caption{Summary statistics of RMSE values over 1000 runs for all test cases on \(G = \mathrm{RST}_{10}\).}
\label{tab:SVL_RST}
\begin{tabular}{ll|llll}
\toprule
\makecell{\textbf{Test}\\ \textbf{case}} & \makecell{\textbf{Architec-}\\ \textbf{ture}} & \textbf{Median (95\% CI)} & \textbf{IQR} & \textbf{Variance} & $[Q_{2.5\%}, Q_{97.5\%}]$ \\
\midrule
\multirow{3}{*}{\makecell{\textbf{LQ}}} & \makecell{FNN} & $1.81e{-4}\ (1.75e{-4}, 1.86e{-4})$ & $8.65e{-5}$ & $4.16e{-9}$ & $[1.06e{-4},3.39e{-4}]$ \\ & \makecell{NTM} & $1.76e{-5}\ (1.57e{-5}, 2.06e{-5})$ & $8.89e{-5}$ & $4.04e{-6}$ & $[2.62e{-6},6.51e{-4}]$ \\ & \makecell{GCN} & $2.84e{-3}\ (1.29e{-3},1.40e{-2})$ & $5.63e{-1}$ & $7.82e{-2}$ & $[1.06e{-5},5.63e{-1}]$\\
\midrule
\multirow{3}{*}{\makecell{\textbf{LQMH}}} & \makecell{FNN} & $9.88e{-5}\ (9.70e{-5}, 1.01e{-4})$ & $3.85e{-5}$ & $9.43e{-10}$ & $[6.43e{-5},1.78e{-4}]$ \\ & \makecell{NTM} & $2.08e{-6}\ (1.95e{-6}, 2.23e{-6})$ & $2.89e{-6}$ & $3.06e{-6}$ & $[4.66e{-7},6.13e{-3}]$ \\ & \makecell{GCN} & $7.54e{-4}\ (3.45e{-4},5.07e{-3})$ & $5.67e{-1}$ & $7.93e{-2}$ & $[3.73e{-6},5.67e{-1}]$ \\
\midrule
\multirow{3}{*}{\makecell{\textbf{NL}}} & \makecell{FNN} & $1.26e{-4}\ (1.18e{-4}, 1.38e{-4})$ & $1.54e{-4}$ & $3.66e{-8}$ & $[5.63e{-5},6.47e{-4}]$ \\ & \makecell{NTM} & $5.56e{-4}\ (5.23e{-4}, 6.01e{-4})$ & $6.76e{-4}$ & $2.84e{-4}$ & $[1.27e{-4},1.64e{-3}]$ \\ & \makecell{GCN} & $5.29e{-1}\ (5.29e{-1},5.29e{-1})$ & $5.29e{-1}$ & $6.98e{-2}$ & $[9.01e{-6},5.29e{-1}]$ \\
\midrule
\multirow{3}{*}{\makecell{\textbf{NLM}}} & \makecell{FNN} & $1.72e{-2}\ (1.70e{-2}, 1.73e{-2})$ & $3.60e{-3}$ & $4.35e{-5}$ & $[6.04e{-3},3.10e{-2}]$ \\ & \makecell{NTM} & $1.05e{-1}\ (1.00e{-1}, 1.09e{-1})$ & $6.13e{-2}$ & $3.14e{-3}$ & $[4.15e{-2},2.20e{-1}]$ \\ & \makecell{GCN} & $1.83e{-2}\ (1.77e{-2},1.93e{-2})$ & $9.95e{-1}$ & $2.43e{-1}$ & $[3.69e{-4},9.96e{-1}]$ \\
\bottomrule
\end{tabular}
\end{threeparttable}
\end{table}

\subsection{Additional Test Cases with Nonlinear State Interactions}

To complement the discussion in Section~\ref{sec:SVL}, we present additional numerical results aimed at probing the representational behavior of NTM under different classes of nonlinear state interactions that may arise in a single player's equilibrium strategies for general games on graphs.

The following test cases all resemble the \textbf{NL} test case in Section~\ref{sec:SVL}, differing only in the choice of the nonlinear state interaction function applied component-wise.
Denote by \(f_{\mathcal{C}}:\R^N\to\R\) the supervised learning target in the test case \(\mathcal{C}\), as specified below:

\smallskip
\textbf{QUAD} (Quadratic Interaction): \(f_{\mathrm{QUAD}}(x) := 3 e_1\transpose \frac{\det(0.5I + L)\, L + \|(I +L)^{-1}\|_F\, L^2}{\|\det(0.5I + L)\, L + \|(I +L)^{-1}\|_F\, L^2\|_F} (x\odot x)\).
This test case evaluates the ability to approximate functions exhibiting superlinear growth and smooth global structure.

\smallskip
\textbf{RQ} (Rational Quadratic Interaction): \(f_{\mathrm{RQ}}(x) := 3 e_1\transpose \frac{\det(0.5I + L)\, L + \|(I +L)^{-1}\|_F\, L^2}{\|\det(0.5I + L)\, L + \|(I +L)^{-1}\|_F\, L^2\|_F} \frac{x\odot x}{\mathbf{1}_N+x}\), where the division is performed component-wise.
This test case evaluates the ability to approximate nonlinear interactions with saturating (subquadratic) behavior.

\smallskip
\textbf{KINK} (Kinked Interaction): \(f_{\mathrm{KINK}}(x) := 3 e_1\transpose \frac{\det(0.5I + L)\, L + \|(I +L)^{-1}\|_F\, L^2}{\|\det(0.5I + L)\, L + \|(I +L)^{-1}\|_F\, L^2\|_F} \max\{x-\frac{1}{2}\mathbf{1}_N,\mathbf{0}_N\}\), where the maximum applies component-wise.
This test case evaluates the ability to approximate functions with localized non-smooth behavior.

We compare the RMSEs produced by the following architectures: FNN with \(\tanh\) activation, FNN with \(\mathrm{ReLU}\) activation, NTM with depth \(K=3\), NTM with depth \(K = 6\), and GCN.
Notably, the inclusion of both \(\tanh\) and \(\mathrm{ReLU}\) activations for the FNN baseline allows us to disentangle the effect of activation functions from architectural differences.
In terms of numerical implementation, the training procedures and hyperparameters are identical to those used in the test case \textbf{NL} reported in Appendix~\ref{app:hyper_SVL}.

Together, these test cases provide a systematic study of how different architectures respond to distinct nonlinearities in the target function.

\begin{figure}[ht]
    \centering
    \includegraphics[width=1.0\linewidth]{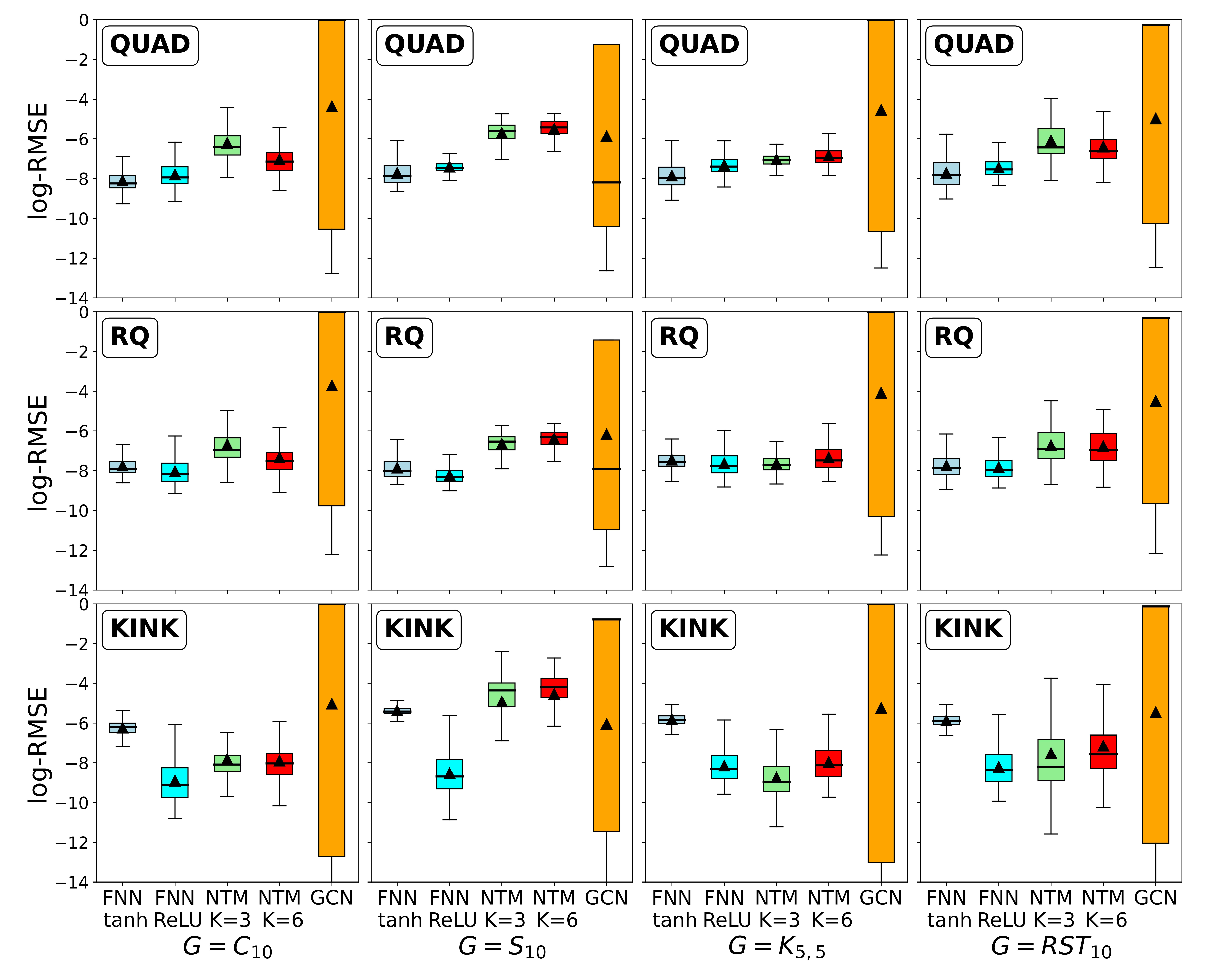}
    \caption{Box plots of log-RMSE over \(1000\) runs for additional test cases with nonlinear state interactions. Different NN architectures are compared across various graphs. Bold black lines indicate the medians; black triangles indicate the means.}
    \label{fig:extra_nonlinear}
\end{figure}

Figure~\ref{fig:extra_nonlinear} shows the box plots of log-RMSE values across different architectures and graph topologies, computed from \(1000\) independent runs.
The qualitative conclusions in Section~\ref{sec:SVL} remain valid, in particular demonstrating the large variability and instability of GCNs.
Numerical results show that NTMs with depth \(K=3\) yield performance similar to NTMs with depth \(K=6\), supporting our previous claim that a small value of \(K\) usually suffices for numerical applications.

For the \textbf{QUAD} and \textbf{RQ} test cases, NTMs achieve reasonable and stable performance across different graph topologies, although FNNs often remain more accurate.
This demonstrates that NTM can approximate quadratic and saturating nonlinear interactions on general graphs while maintaining an interpretable and parameter-efficient architecture. 
For the \textbf{KINK} test case, NTM outperforms the FNN with tanh activation and achieves a level of accuracy comparable to the FNN with ReLU activation on most graph topologies, except on the star graph, where the vertices are highly heterogeneous and the performance of NTM deteriorates.
It is worth emphasizing that NTM with \(K=6\) uses \(804\) parameters, whereas the FNN used in our numerical experiments has \(1441\) parameters. 
As a result, although NTM does not uniformly outperform FNNs in terms of raw approximation accuracy, it often achieves comparable performance using less than 55.8\% of the parameters.

These numerical results suggest that the strength of NTM lies not in uniformly dominating FNNs in terms of raw RMSE, but rather in providing a favorable balance among approximation accuracy, parameter efficiency, and training stability.

\section{Technical Details of NTM-DBSDE}\label{app:DBSDE}

This appendix complements our discussion of (NTM-)DBSDE in Section~\ref{sec:alg}, focusing on algorithmic design (Appendix~\ref{app:design}) and the numerical implementation (Appendix~\ref{app:implement}). For a more in-depth discussion, we refer readers to \cite{han2017deep,han2020deep,han2025brief}.

Let $X_t = [X_t^1, \ldots, X_t^N]$ denote the state dynamics of all players,
\begin{equation}
    \label{eqn:gen_state_dynamics}
    \ud X_t = B(t,X_t,\alpha_t)\ud t + \Sigma(t,X_t) \ud W_t,
\end{equation}
where \(B:[0,T]\times\R^N\times\R^N\to\R^N\), \(\Sigma:[0,T]\times\R^N\to\R^{N\times N}\), $\alpha_t$ is the collection of all controls, and \(\{W_t\}\) is an \(\R^N\)-valued Brownian motion.
Under Markovian strategies \(\alpha^i_t = \phi_i(t,X_t)\), player \(i\) aims to  minimize the following objective:
\begin{equation}
    \label{eqn:gen_J}
    J^i(\alpha) := \E \left[\int_0^T f^i(t,X_t,\alpha_t)\ud t + g^i(X_T)\right],
\end{equation}
over admissible feedback functions \(\phi_i\), where  \(f^i\) and \(g^i\) are running and terminal costs.

\subsection{Deep Fictitious Play based on Deep BSDE}\label{app:design}
The algorithm presented below was introduced in \cite{han2020deep} and extends the Deep BSDE method \cite{han2017deep}, originally developed for solving semi-explicit PDEs, to solving general finite-player stochastic differential games.

Let \(v^i:[0,T]\times\R^N\to\R\) denote the value function of player $i$, defined by
\begin{equation}
    v^i(t,x) := \E \left[\int_t^T f^i(s,X_s,\alpha_s)\ud s + g^i(X_T)\Big|X_t = x\right].
\end{equation}
By the dynamic programming principle (DPP), \(v^1,\ldots,v^N\) satisfy the HJB system:
\begin{equation}
    \label{eqn:app_HJB}
    \partial_t v^i + \inf_{\alpha^i}\{[B(t,x,\alpha)]\transpose \partial_x v^i + f^i(t,x,\alpha)\} + \frac{1}{2}\Tr(\Sigma\Sigma\transpose (t,x)\;\partial_{xx} v^i)=0,
\end{equation}
with terminal condition \(v^i(T,x) = g^i(x)\).
Assuming the infimum admits a unique explicit minimizer, we define:
\begin{equation}
    \label{eqn:DBSDE_opt_strategy}
    \hat{\alpha}^i(t,x,\partial_x v^i;\alpha^{-i}) := {\arg\inf}_{\alpha^i}\{[B(t,x,\alpha)]\transpose \partial_x v^i + f^i(t,x,\alpha)\},
\end{equation}
where \(\alpha^{-i}\) is the strategy profile of the other players except player \(i\), and is associated with feedback functions \(\phi_{-i}\).
Substituting this into the HJB system~\eqref{eqn:app_HJB} yields semilinear parabolic PDEs:
\begin{equation}
    \label{eqn:semilinear_PDE}
    \partial_t v^i + [\mu^i(t,x;\alpha^{-i})]\transpose \partial_x v^i + \frac{1}{2}\Tr(\Sigma\Sigma\transpose(t,x)\; \partial_{xx} v^i) + h^i(t,x,\Sigma\transpose \partial_x v^i;\alpha^{-i}) = 0,
\end{equation}
where \(\mu^i\) and \(h^i\) are functions that can be further specified to match~\eqref{eqn:app_HJB}\footnote{For consistent notations in subsequent context, we describe the functional dependence of \(h^i\) on \(\partial_x v^i\) through \(\Sigma\transpose \partial_x v^i\), which is equivalent due to the non-degeneracy of \(\Sigma\).}.

The nonlinear Feynman-Kac formula \cite{pardoux2005backward,pardoux1999forward} yields the correspondence:
\begin{equation}
    \label{eqn:FK_correspondence}
    Y^i_t := v^i(t,\chi^i_t),\quad Z^i_t := \Sigma\transpose(t,\chi^i_t)\;\partial_x v^i(t,\chi^i_t).
\end{equation}
Here \(\{\chi^i_t\}\) is an auxiliary process satisfying the forward SDE and the value process \(\{Y^i_t\}\) satisfies the backward SDE:
\begin{align}
    \label{eqn:DBSDE_FSDE}
    \ud \chi^i_t &= \mu^i(t,\chi^i_t;\phi_{-i}(t,\chi^i_t))\ud t + \Sigma(t,\chi^i_t)\ud W_t,\quad \chi^i_0 = \chi_0,\\
    \label{eqn:DBSDE_BSDE}
    \ud Y^i_t &= -h^i(t,\chi^i_t,Z^i_t;\phi_{-i}(t,\chi^i_t))\ud t + Z^i_t\ud W_t,\quad Y^i_T = g^i(\chi^i_T),
\end{align}
where \(\chi_0\) is a square-integrable random variable independent of the Brownian motion \(\{W_t\}\).

The Deep BSDE method is then applied to solve this system \eqref{eqn:DBSDE_FSDE}-\eqref{eqn:DBSDE_BSDE} iteratively, approximating solutions to the HJB system~\eqref{eqn:app_HJB} and the associated equilibrium strategy of the game.

\subsection{Numerical implementation}\label{app:implement}

As mentioned in Section~\ref{sec:alg}, we parameterize $Z_t^i$, i.e. \(x\mapsto\Sigma\transpose(t,x)\; \partial_x v^i(t,x)\) using a neural network \(\phi^{\mathrm{NN}}(t,\cdot):\R^N\to\R^N\) for discrete times \(t\in\Delta\) in a partition of $[0,T]$. Thus \((\Sigma\transpose)^{-1}\;\phi^{\mathrm{NN}}\) approximates \(\partial_x v^i\), which determines the optimal strategy $\hat \alpha^i$ through~\eqref{eqn:DBSDE_opt_strategy}. This motivates the use of NTM for parameterizing the adjoint process $Z_t$.

Sample paths of \(\{\chi^i_t\}\) and \(\{Y^i_t\}\) are generated via the Euler-Maruyama scheme. The loss function for training $\phi^{\mathrm{NN}}$ is the expected mismatch of the BSDE terminal condition:
\begin{equation}
    \label{eqn:DBSDE_loss}
    \E [Y^i_T - g^i(\chi^i_T)]^2,
\end{equation}
approximated by Monte Carlo with \(N_{\mathrm{batch}}\) sample paths. Performing this procedure above once for each player \(i\in[N]\) concludes a single round of DFP.

For the two game models studied in this paper, the specifications of \(\mu^i\) and \(h^i\)  align with those reported in \cite{han2020deep}, and the selection may not be unique; in practice, including more  terms in \(\mu^i\) is often preferred.
Implied by the correspondence~\eqref{eqn:FK_correspondence}, when the current state is \(X_t\) at time \(t\), the neural network output \(\phi^{\mathrm{NN}}(t,X_t)\) provides an approximation for the equilibrium strategy:
\begin{equation}
    \label{eqn:alpha_DBSDE}
    \hat{\alpha}^i(t,X_t) \approx -qe_i\transpose LX_t - \frac{1}{\sigma}[\phi^{\mathrm{NN}}(t,X_t)]_i.
\end{equation}

\subsection{The Non-trainable Variant of Deep BSDE}\label{app:NTM-DBSDE}

With the derivations and implementations of Deep BSDE introduced above, we state the NTM-DBSDE as Algorithm~\ref{alg:NTM-DBSDE} to complement discussions in Section~\ref{sec:alg}.

\begin{algorithm}
\renewcommand{\algorithmicrequire}{\textbf{Input:}}
\renewcommand{\algorithmicensure}{\textbf{Output:}}
\caption{NTM-DBSDE: Non-Trainable Deep BSDE for Games on Graphs}
\label{alg:NTM-DBSDE}
\begin{algorithmic}[1]
    \REQUIRE A family of NTM architectures \(\{\phi^{\mathrm{NTM},j}_{K,M,G}(t,\cdot)\}_{j\in[N],t\in[0,T]}\) 
    \STATE Initialize trainable parameters of \(\phi^{\mathrm{NTM},j}_{K,M,G}(t,\cdot),\ \forall j \in [N],\, t \in [0,T]\)
    \REPEAT 
        \FOR{each player \(i\in[N]\)}
            \STATE Simulate sample paths of the FBSDE~\eqref{eqn:DBSDE_FSDE}--\eqref{eqn:DBSDE_BSDE} using network outputs \(\phi^{\mathrm{NTM},j}_{K,M,G}(t,x),\ \forall j\in[N],\ t\in[0,T]\)
            \STATE Compute the loss~\eqref{eqn:DBSDE_loss} that quantifies the mismatch at the terminal condition
            \STATE Update trainable parameters of \(\phi^{\mathrm{NTM}, i}_{K,M,G}(t,\cdot)\), for all \(t \in [0,T]\), using the loss
        \ENDFOR
    \UNTIL{convergence or maximum number of DFP rounds reached}
    \ENSURE A family of trained NTMs approximating the Nash equilibrium strategies
\end{algorithmic}  
\end{algorithm}

\section{Additional Numerical Results for Games in Section~\ref{sec:numerics}}\label{app:M3}

This section presents additional numerical results that further demonstrate the advantage of the NTM architecture. For each of the three game models, we provide: (i) training loss curves of NTM-DBSDE and Deep BSDE, (ii) a table reporting the maximum relative error (MRE) of the value functions across all players, and (iii) equilibrium state and strategy trajectories produced by (NTM-)DBSDE. 
Unless otherwise specified, all experiments use the same architectures and hyperparameters as in Section~\ref{sec:numerics}.
In this section, the reported RMSE values correspond to single evaluation runs (rather than averages over multiple runs), each computed using \(N_{\mathrm{test}}=5000\) paths.
Nevertheless, as demonstrated in Section~\ref{sec:numerics}, the variability across evaluation runs is small.

\smallskip

\noindent\textbf{Training loss.}
The training loss of the neural network associated with player \(i\) is defined in equation~\eqref{eqn:DBSDE_loss}. This loss represents the mismatch of the terminal condition in the backward SDE~\eqref{eqn:DBSDE_BSDE}. As training progresses, the loss should approach zero if the neural network converges.

The training loss curve reflects the learning dynamics of the network and provides a sanity check for our numerical implementation of (NTM-)DBSDE. Since the loss curves exhibit qualitatively similar behavior across different graph structures, we present results only for the star graph $G= S_N$.

\smallskip
\noindent\textbf{Maximum relative error (MRE).}
Following the notations of Section~\ref{sec:numerics}, for any player \(i\in[N]\), let \(\{\hat{X}^{i,m}_t\}\) and \(\{\hat{\alpha}^{i,m}_t\}\) be the \(m\)-th sample path of the equilibrium state and strategy processes of player \(i\) under the baseline equilibrium strategy.
Similarly, let \(\{\tilde{X}^{i,m}_t\}\) and \(\{\tilde{\alpha}^{i,m}_t\}\) be those generated under the neural network parameterized equilibrium strategy.
With \(N_{\mathrm{test}} = 25000\) sample paths simulated, the expected cost of player \(i\) (cf. equation~\eqref{eqn:J}) under the baseline equilibrium strategy can be calculated and is denoted by \(\hat{V}^{i}\), while \(\tilde{V}^{i}\) denotes the expected cost under the neural network parameterized strategy.

The MRE is then defined as:
\begin{equation}
    \mathrm{MRE} := \max_{i\in[N]}\left|\frac{\tilde{V}^i - \hat{V}^i}{\hat{V}^i}\right|,
\end{equation}
which measures the maximum relative error in the expected cost across all players. A smaller MRE indicates better performance of the neural network in approximating the equilibrium.

\subsection{Linear-Quadratic Stochastic Differential Games (Section~\ref{sec:numerics-lq})}\label{app:M3_LQ}

\begin{figure}[htbp]
    \centering
    \includegraphics[width=\linewidth]{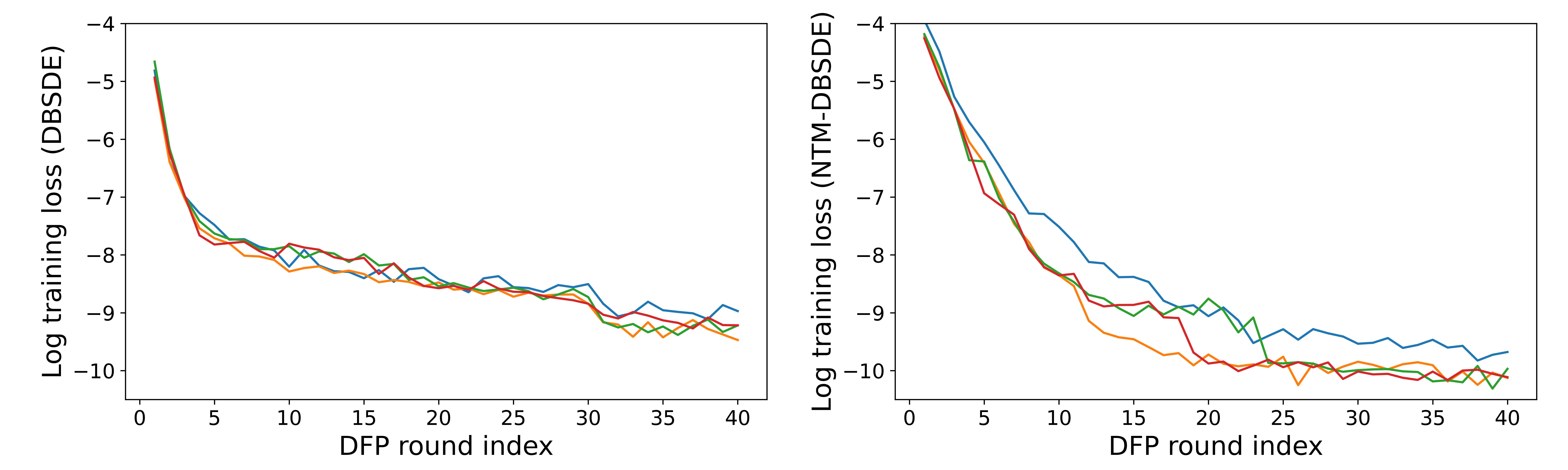}
    \caption{Training loss of Deep BSDE (left) and NTM-DBSDE (right) for \(N=10\) players in the LQ game on the star graph \(G = S_{10}\) (cf. Section~\ref{sec:numerics-lq}).
    Colored solid lines represent the logarithmic training loss of individual players (only four shown for clarity). }
    \label{fig:LQ_DBSDE_loss}
\end{figure}

\begin{table}[htbp]
\caption{MRE for the LQ game on graphs in Section~\ref{sec:numerics-lq}}
\label{tab:LQ_value_M3}
\centering
  \begin{tabular}{c|ccc}
    \toprule
   {Problem} & {\(G = C_{10}\)} & {\(G = S_{10}\)} & {\(G = \mathrm{RST}_{10}\)}\\
      \midrule
    DP & \(2.36e{-2}\) & \(2.09e{-2}\) & \(2.36e{-2}\) \\
    NTM-DP & \(2.14e{-2}\) & \(1.89e{-2}\) & \(2.30e{-2}\) \\
    DBSDE & \(2.74e{-2}\) & \(2.35e{-2}\) & \(2.45e{-2}\) \\
    NTM-DBSDE & \(1.66e{-2}\) & \(3.58e{-2}\) & \(2.90e{-2}\) \\
    \bottomrule
  \end{tabular}
\end{table}

\begin{figure}[htbp]
    \centering
    \includegraphics[width=\linewidth]{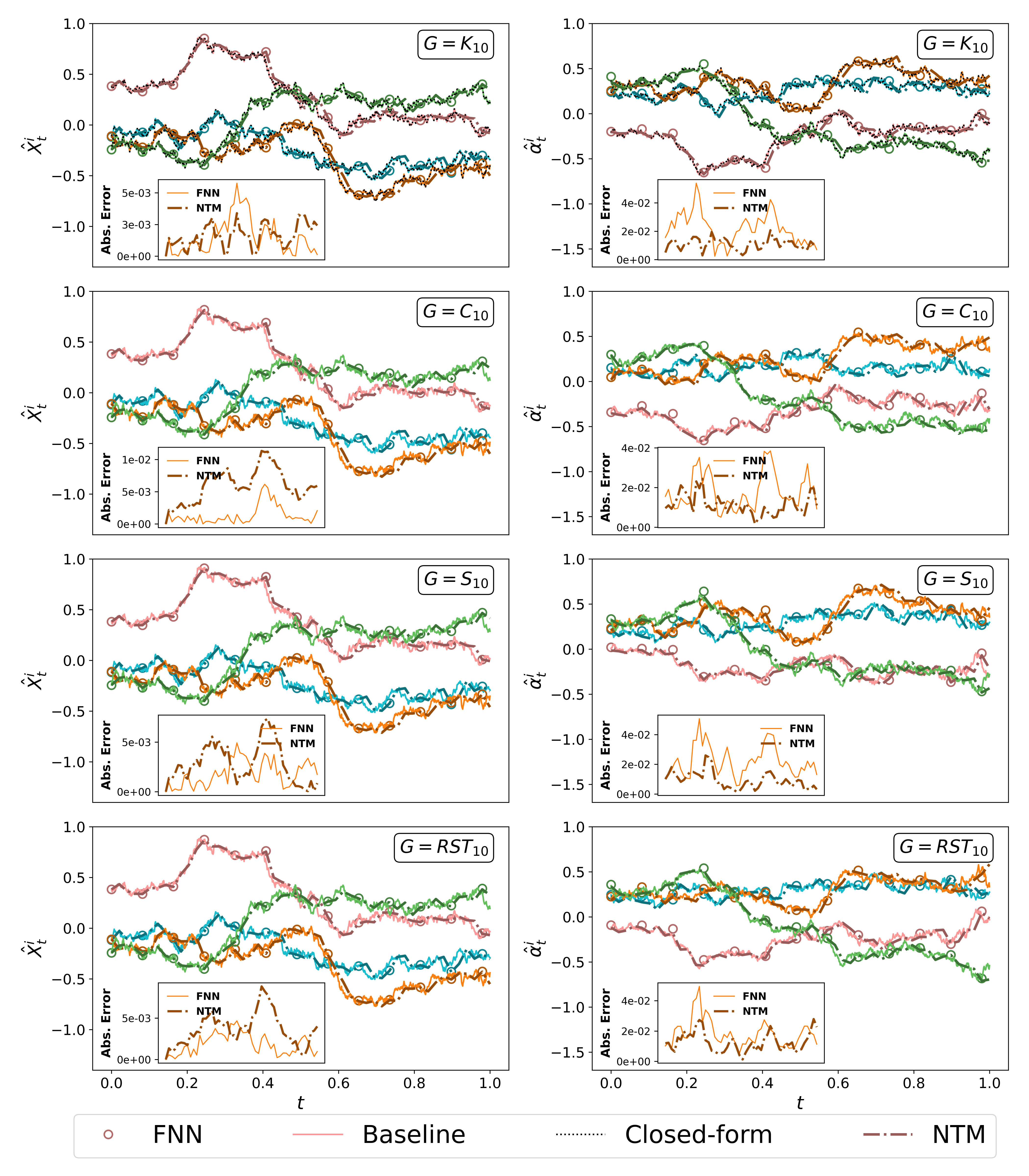}
    \caption{
    Equilibrium state (left) and strategy (right) trajectories for \(N=10\) players in the LQ game on various graphs (cf. Section~\ref{sec:numerics-lq}).
    Solid lines: baseline solution; dark-colored circles: Deep BSDE; dark-colored dotted lines: NTM-DBSDE; black dotted lines: closed-form solution. Only four players (indexed by 1, 2, 3, 4) are shown for clarity.
    In the inset panels, the solid (resp. dotted) lines show the absolute error of the trajectories generated by Deep BSDE (resp. NTM-DBSDE).
    Only one representative player is displayed for clarity, and the strategy error is smoothed using a moving average with a window size of \(3\).}
    \label{fig:LQ_DBSDE}
\end{figure}

Figure~\ref{fig:LQ_DBSDE_loss} shows the training loss curves of different players in NTM-DBSDE and Deep BSDE. For both methods,
the loss decreases toward zero as training progresses, confirming numerical convergence. 
Table~\ref{tab:LQ_value_M3} reports the MRE across different algorithms for the LQ game on graphs. Figure~\ref{fig:LQ_DBSDE} compares equilibrium state and strategy trajectories for the LQ game. 
The consistently small MREs, in terms of value functions, and the close alignment of trajectories produced by Deep BSDE and NTM-DBSDE demonstrate that NTM achieves performance comparable to standard FNN architectures.

\subsection{Multi-Agent Portfolio Games (Section~\ref{sec:numerics-portfolio})}\label{app:M3_Port}

\begin{table}[ht!]
\caption{MRE for the multi-agent portfolio game in Section~\ref{sec:numerics-portfolio}}
\label{tab:port_value_M3}
\centering
  \begin{tabular}{c|ccc}
    \toprule
   {Problem} & {\(G = C_{10}\)} & {\(G = S_{10}\)} & {\(G = \mathrm{RST}_{10}\)}\\
      \midrule
    DP & \(5.41e{-3}\) & \(7.01e{-3}\) & \(5.73e{-3}\) \\
    NTM-DP & \(5.60e{-3}\) & \(6.25e{-3}\) & \(5.17e{-3}\) \\
    \bottomrule
  \end{tabular}
\end{table}

Since the multi-agent portfolio game involves a controlled diffusion coefficient, the Deep BSDE method cannot be directly applied.
Accordingly,  we report only the MRE values in Table~\ref{tab:port_value_M3}. The numbers are comparable across algorithms, leading to conclusions consistent with those in Appendix~\ref{app:M3_LQ}.

\subsection{A Variant of the LQ Game (Section~\ref{sec:numerics_nonlq})}\label{app:M3_non_LQ}

\begin{table}[htbp]
\caption{MRE for the variant of the LQ game in Section~\ref{sec:numerics_nonlq}}
\label{tab:nonLQ_value_M3}
\centering
  \begin{tabular}{c|ccc}
    \toprule
   {Problem} & {\(G = C_{10}\)} & {\(G = S_{10}\)} & {\(G = \mathrm{RST}_{10}\)}\\
      \midrule
    NTM-DP & \(1.90e{-2}\) & \(1.01e{-2}\) & \(1.30e{-2}\) \\
    NTM-DBSDE & \(2.10e{-2}\) & \(2.20e{-2}\) & \(2.06e{-2}\) \\
    \bottomrule
  \end{tabular}
\end{table}

\begin{figure}[htbp]
    \centering
    \includegraphics[width=\linewidth]{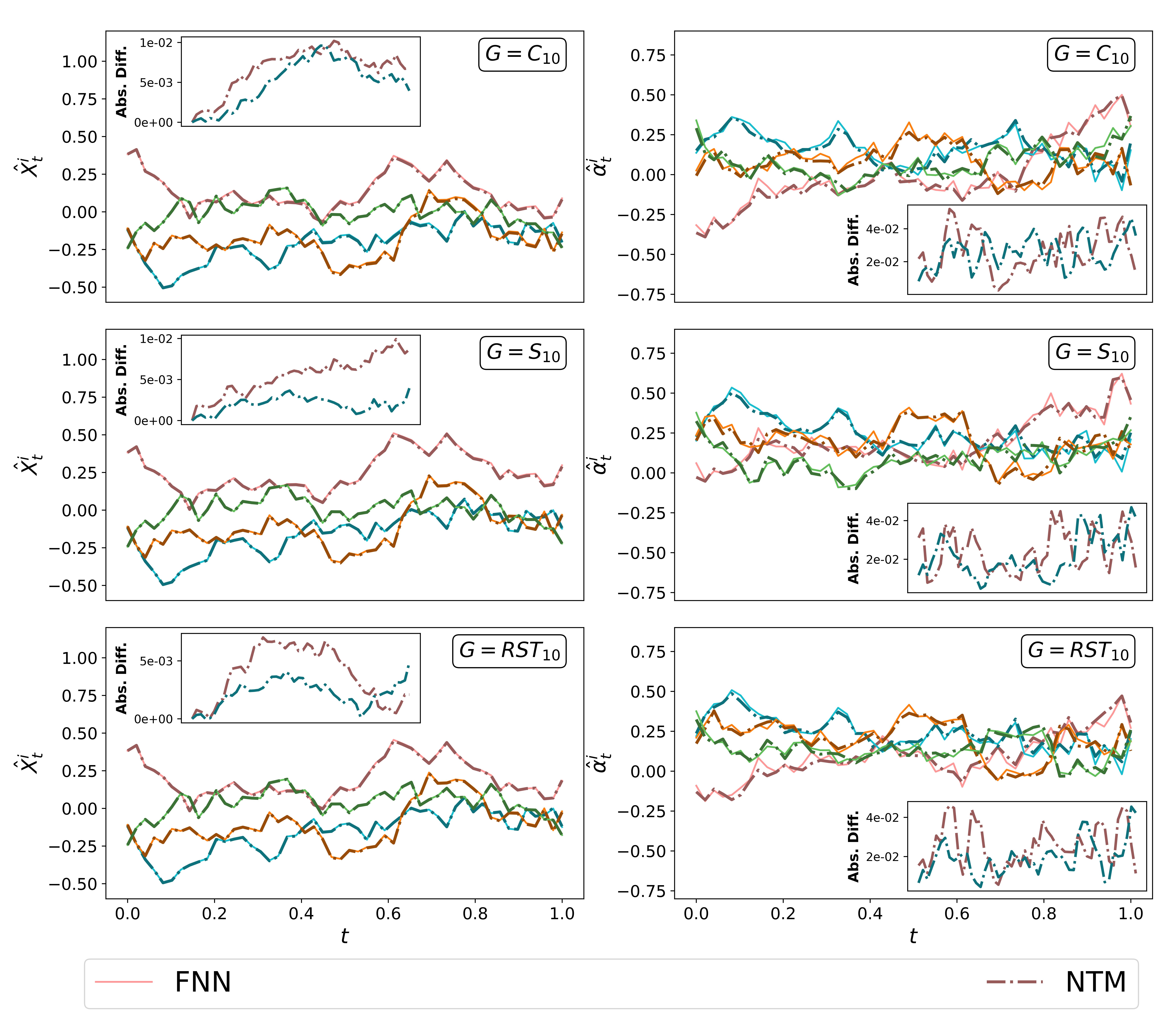}
    \caption{Equilibrium state (left) and strategy (right) trajectories for \(N=10\) players in the variant of the LQ game on various graphs (cf. Section~\ref{sec:numerics_nonlq}).
    Solid lines: baseline solution from Deep BSDE; dark-colored dotted lines: NTM-DBSDE. Only four players (indexed by 1, 2, 3, 4) are shown for clarity.
    In the inset panels, dotted lines show the absolute difference between Deep BSDE and NTM-DBSDE trajectories.
    Only one representative player is displayed for clarity, and the strategy error is smoothed using a moving average with a window size of \(3\).}
    \label{fig:nonLQ_DBSDE}
\end{figure}

Table~\ref{tab:nonLQ_value_M3} reports the MRE across different algorithms for the variant of the LQ game on graphs.
Figure~\ref{fig:nonLQ_DBSDE} compares equilibrium state and strategy trajectories for the variant of the LQ game on various graphs.

\medskip

In brief, we conclude that the non-trainable version of DBSDE works as well as that of DP, which shows the general applicability of NTM for game models and game-solving algorithms.

\subsection{Performance of GCN-Based Architectures}\label{app:GCN_fail}

Complementing the performance comparison between FNN and NTM architectures, we briefly examine the performance of GCN-based architectures for solving games on graphs.
Throughout the remainder of this section, the GCN architecture of our consideration is identical to that used in supervised learning experiments (cf. Section~\ref{sec:SVL}).
The games are solved using a GCN-based variant of the DP algorithm (denoted GCN-DP), in which each player maintains an independent GCN at each time step.

Our first observation is that, despite potentially having fewer parameters in large-scale settings, GCN requires substantially more memory than FNN and NTM when solving games on graphs.
For instance, when solving the same game with \(N =10\) players on an Nvidia GeForce RTX 2080 Ti GPU with about 10.5 GB of available memory, GCN encounters memory overflow, whereas FNN and NTM both operate normally.
We attribute this observation to the storage of computation graphs in \texttt{PyTorch}, which is necessary for gradient evaluation during backpropagation.
For FNN, NTM and GCN with the same width/hidden feature dimension \(F\), the hidden layer outputs of FNN and NTM are tensors with \(N_{\mathrm{batch}}F\) entries, while those of GCN have \(NN_{\mathrm{batch}}F\) entries.
The additional factor of \(N\) leads to a significantly larger memory footprint for storing the computation graph, thereby explaining the higher memory demand of GCN.

Consequently, we evaluate the game-solving performance of GCNs on a smaller-scale setting with \(N=5\) players, which is computationally much easier than the \(10\)-player case.
We test GCN across all three models in Section~\ref{sec:numerics} under the same model parameters.
We observe that the numerical results are largely insensitive to the choice of training hyperparameters and random seeds.
Therefore, we adopt the same set of training hyperparameters as those reported in Appendix~\ref{app:hyper_game}.

\begin{figure}
    \centering
    \includegraphics[width=1.0\linewidth]{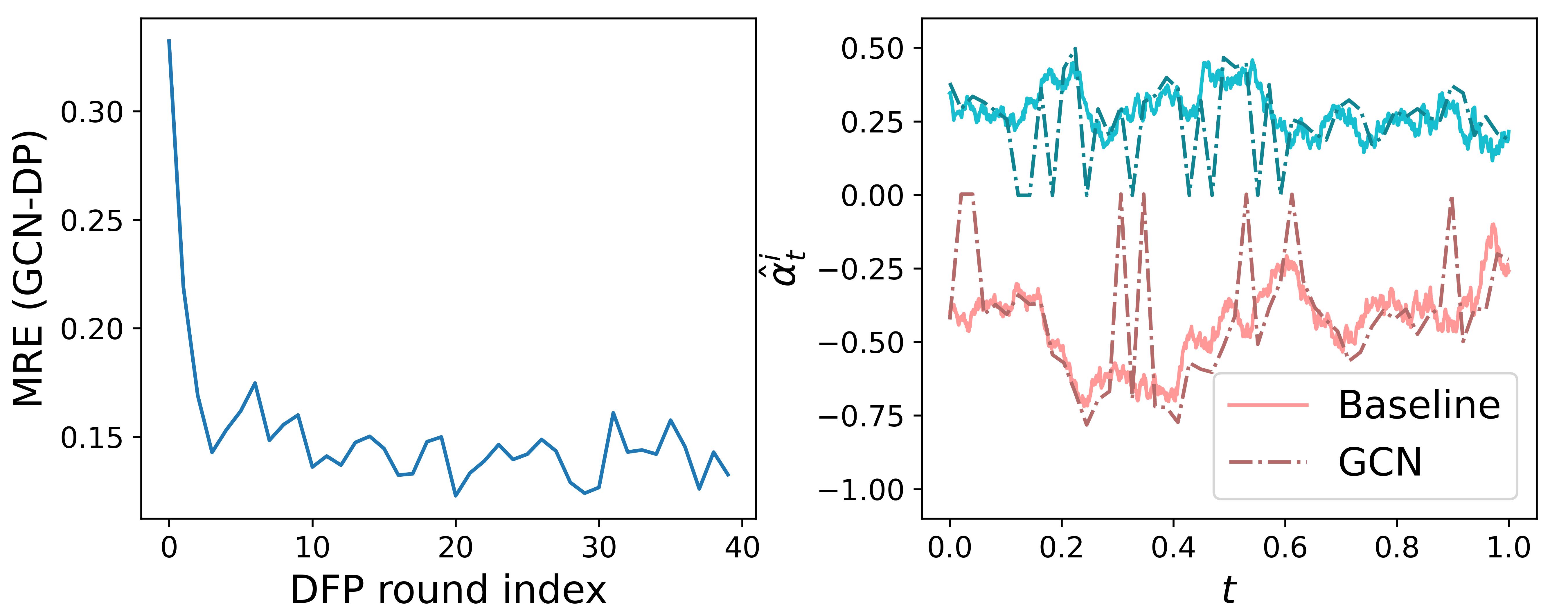}
    \caption{MRE curves under GCN-DP (left) and equilibrium strategy trajectories (right) for \(N=5\) players in the LQ game on the cycle graph \(G = C_5\) (cf. Section~\ref{sec:numerics-lq}).
    In the right panel, solid lines denote the baseline solution and dark-colored dotted lines are generated by GCN-DP. Different players are distinguished by different colors. Only two players (indexed by 1, 2) are shown for clarity.}
    \label{fig:GCN_failure}
\end{figure}

Figure~\ref{fig:GCN_failure} illustrates the MRE evolution during the training of GCN-DP and compares the resulting equilibrium strategy trajectories in the LQ game on the cycle graph (cf. Section~\ref{sec:numerics-lq}).
The left panel shows that the MRE initially decreases but then plateaus at around \(15\%\), failing to admit further improvements.
This indicates that GCN-DP does not fully capture the correct equilibrium value function.
The right panel reveals a characteristic failure pattern of GCN in solving games on graphs: while it generally captures the correct magnitude of the equilibrium strategy, it exhibits sudden spikes at certain time steps.
We note that this observation is consistent with the instability of GCN reported in Section~\ref{sec:SVL}.
At these time steps, the approximation fails, which in turn propagates and degrades the learning performance at subsequent time steps.
The corresponding quantitative evaluation metrics for GCN-DP are \(\mathrm{RMSE}_X = 1.26\%\), \(\mathrm{RMSE}_\alpha = 23.2\%\), and \(\mathrm{MRE} = 13.26\%\).

Interestingly, for the multi-agent portfolio game on the star graph (cf. Section~\ref{sec:numerics-portfolio}), we observe markedly different behavior. 
In this setting, GCN-DP achieves performance comparable to FNN and NTM in learning the equilibrium, with quantitative evaluation metrics given by \(\mathrm{RMSE}_X = 1.54\%\), \(\mathrm{RMSE}_\alpha = 0.67\%\), and \(\mathrm{MRE} = 1.18\%\).
We hypothesize that the improved performance of GCN-DP may be attributed to the simplicity of the equilibrium strategy (which is constant-valued). 
However, this explanation remains empirical and preliminary, and a rigorous understanding is left for future work.

Regarding the non-LQ game on graphs (cf. Section~\ref{sec:numerics_nonlq}), we observe behavior similar to that in the LQ setting and therefore omit the corresponding results for brevity. 
We have verified that the same qualitative behavior persists across multiple graph structures, and thus report results only on a representative graph.

In summary, while GCN-DP performs well in the portfolio game setting, its performance is not consistent across models. 
In contrast, the instability observed in the LQ setting suggests that GCN-based approaches may lack robustness in capturing more complex equilibrium structures.

\section{Game-Solving Results using NTM Architecture with \(M=1\)}\label{app:M1}

In this section, we check the game-solving performance of NTM with width \(M=1\).
Based on numerical results reported in Appendix~\ref{app:SVL_approximation}, such an architecture still admits adequate expressivity.
We aim to show that, reducing the width in the NTM architecture does not have a significant impact on the empirical performance, but enlarges the gain in sparsity by further lowering the parameter count.

Numerical experiments are conducted in the same environment using the same hyperparameters as those reported in Section~\ref{sec:numerics}, except that the NTM architecture mentioned in this section has width \(M=1\).
We report RMSE and MRE as evaluation metrics (see Section~\ref{sec:numerics} and Appendix~\ref{app:M3}). 
The reported RMSE values correspond to single evaluation runs (rather than averages over multiple runs), each computed using \(N_{\mathrm{test}}=5000\) paths.
Nevertheless, as demonstrated in Section~\ref{sec:numerics}, the variability across evaluation runs is small.
Unless otherwise noted, the equilibrium state and strategy trajectories exhibit behavior similar to those presented in Section~\ref{sec:numerics}, and are therefore omitted for brevity.

\subsection{Linear-Quadratic Stochastic Differential Games (Section~\ref{sec:numerics-lq})}\label{app:M1_LQ}

\begin{table}[ht!]
\caption{\(\mathrm{RMSE}_X\) and \(\mathrm{RMSE}_\alpha\) for the LQ game on graphs in Section~\ref{sec:numerics-lq} when \(M=1\)}
\label{tab:LQ_RMSE_M1}
\centering
  \begin{tabular}{c|cccccc}
    \toprule
    \multirow{2}{*}{Problem} &
      \multicolumn{2}{c}{\(G = C_{10}\)} &
      \multicolumn{2}{c}{\(G = S_{10}\)} &
      \multicolumn{2}{c}{\(G = \mathrm{RST}_{10}\)}\\
    & \(\mathrm{RMSE}_X\) & \(\mathrm{RMSE}_\alpha\) & \(\mathrm{RMSE}_X\) & \(\mathrm{RMSE}_\alpha\) & \(\mathrm{RMSE}_X\) & \(\mathrm{RMSE}_\alpha\)\\
      \midrule
    DP & \(1.91e{-2}\) & \(2.70e{-2}\) & \(2.11e{-2}\) & \(2.48e{-2}\) & \(1.90e{-2}\) & \(2.69e{-2}\)\\
    NTM-DP & \(1.90e{-2}\) & \(2.65e{-2}\) & \(2.11e{-2}\) & \(2.44e{-2}\) & \(1.90e{-2}\) & \(2.64e{-2}\)\\
    DBSDE & \(1.89e{-2}\) & \(3.45e{-2}\) & \(2.10e{-2}\) & \(3.14e{-2}\) & \(1.89e{-2}\) & \(3.47e{-2}\)\\
    NTM-DBSDE & \(1.90e{-2}\) & \(3.21e{-2}\) & \(2.11e{-2}\) & \(3.60e{-2}\) & \(1.90e{-2}\) & \(3.22e{-2}\)\\
    \bottomrule
  \end{tabular}
\end{table}

\begin{table}[ht!]
\caption{MRE for the LQ game on graphs in Section~\ref{sec:numerics-lq} when \(M=1\)}
\label{tab:LQ_value_M1}
\centering
  \begin{tabular}{c|ccc}
    \toprule
   {Problem} & {\(G = C_{10}\)} & {\(G = S_{10}\)} & {\(G = \mathrm{RST}_{10}\)}\\
      \midrule
    DP & \(2.36e{-2}\) & \(2.09e{-2}\) & \(2.36e{-2}\) \\
    NTM-DP & \(2.46e{-2}\) & \(2.04e{-2}\) & \(2.33e{-2}\) \\
    DBSDE & \(2.74e{-2}\) & \(2.35e{-2}\) & \(2.45e{-2}\) \\
    NTM-DBSDE & \(1.96e{-2}\) & \(1.85e{-2}\) & \(1.66e{-2}\) \\
    \bottomrule
  \end{tabular}
\end{table}

Tables~\ref{tab:LQ_RMSE_M1}--\ref{tab:LQ_value_M1} present the RMSE and MRE across different algorithms for the LQ game on graphs.
Consistently small RMSEs and MREs show that, measured in terms of paths and values, NTM with \(M=1\) achieves a comparable performance to FNN, regardless of the graph and the algorithm.

\subsection{Multi-agent portfolio games (Section~\ref{sec:numerics-portfolio})}\label{app:M1_Port}

\begin{table}[ht!]
\caption{\(\mathrm{RMSE}_X\) and \(\mathrm{RMSE}_\alpha\) for Portfolio games on graphs in Section~\ref{sec:numerics-portfolio} when \(M=1\)}
\label{tab:port_RMSE_M1}
\centering
  \begin{tabular}{c|cccccc}
    \toprule
    \multirow{2}{*}{Problem} &
      \multicolumn{2}{c}{\(G = C_{10}\)} &
      \multicolumn{2}{c}{\(G = S_{10}\)} &
      \multicolumn{2}{c}{\(G = \mathrm{RST}_{10}\)}\\
    & \(\mathrm{RMSE}_X\) & \(\mathrm{RMSE}_\alpha\) & \(\mathrm{RMSE}_X\) & \(\mathrm{RMSE}_\alpha\) & \(\mathrm{RMSE}_X\) & \(\mathrm{RMSE}_\alpha\)\\
      \midrule
    DP & \(1.82e{-2}\) & \(9.53e{-3}\) & \(1.78e{-2}\) & \(9.48e{-3}\) & \(1.80e{-2}\) & \(9.38e{-3}\)\\
    NTM-DP & \(1.41e{-2}\) & \(3.00e{-3}\) & \(1.45e{-2}\) & \(4.19e{-3}\) & \(1.41e{-2}\) & \(2.98e{-3}\)\\
    \bottomrule
  \end{tabular}
\end{table}

\begin{table}[ht!]
\caption{MRE for multi-agent portfolio games in Section~\ref{sec:numerics-portfolio} when \(M=1\)}
\label{tab:port_value_M1}
\centering
  \begin{tabular}{c|ccc}
    \toprule
   {Problem} & {\(G = C_{10}\)} & {\(G = S_{10}\)} & {\(G = \mathrm{RST}_{10}\)}\\
      \midrule
    DP & \(5.41e{-3}\) & \(7.01e{-3}\) & \(5.73e{-3}\) \\
    NTM-DP & \(5.26e{-3}\) & \(6.31e{-3}\) & \(5.30e{-3}\) \\
    \bottomrule
  \end{tabular}
\end{table}

\begin{figure}
    \centering
    \includegraphics[width=\linewidth]{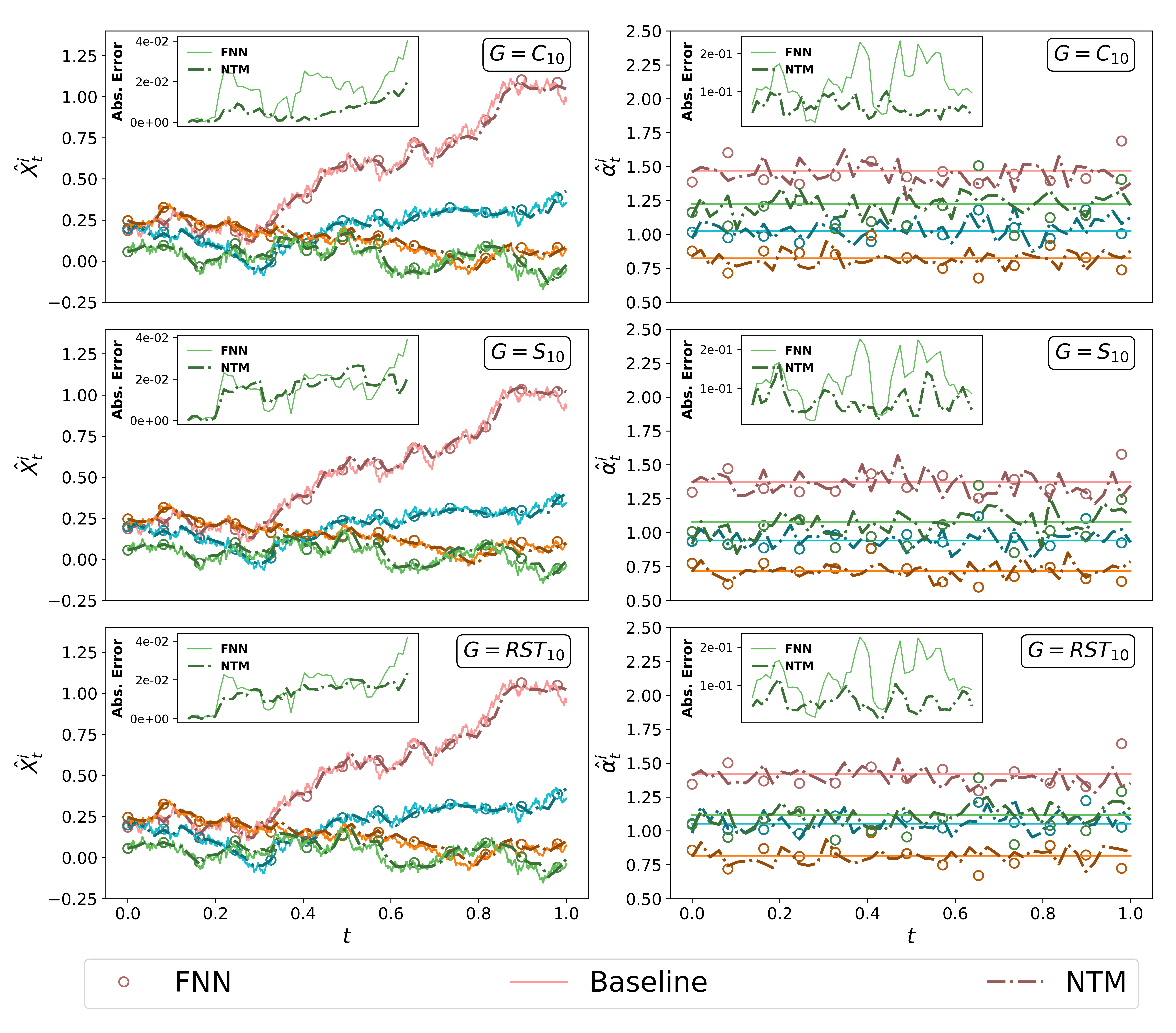}
    \caption{Equilibrium state (left) and strategy (right) trajectories for \(N=10\) players in the Portfolio game on various graphs (cf. Section~\ref{sec:numerics-portfolio}).
    Solid lines: baseline solution; dark-colored circles: DP; dark-colored dotted lines: NTM-DP with \(M=1\). 
    Only four players (indexed by 1, 4, 6, 9) are shown for clarity. 
    In the inset panels, the solid (resp. dotted) lines show the absolute error of the trajectories generated by DP (resp. NTM-DP with \(M=1\)).
    Only one representative player is displayed for clarity, and the strategy error is smoothed using a moving average with a window size of \(3\).}
    \label{fig:Port_DP_M1}
\end{figure}

Since the multi-agent portfolio game has a controlled diffusion coefficient, DBSDE does not apply.
Therefore, we only present tables of RMSE and MRE values, together with equilibrium trajectories solved under (NTM-)DP, as complementary materials.
Tables~\ref{tab:port_RMSE_M1}--\ref{tab:port_value_M1} present the RMSE and MRE across different algorithms for the multi-agent portfolio game.
Figure~\ref{fig:Port_DP_M1} demonstrates the comparison of equilibrium state and strategy trajectories for the multi-agent portfolio game on various graphs for the algorithm DP.
The observations and conclusions align with those mentioned in Appendix~\ref{app:M1_LQ}.
Importantly, we emphasize that the equilibrium strategy trajectories produced by NTM with \(M=1\) (cf. Figure~\ref{fig:Port_DP_M1}) exhibit significantly less fluctuation compared to those generated by NTM with \(M=3\) (cf. Figure~\ref{fig:Port_DP}).
This suggests that, while the NTM architecture already mitigates overfitting due to its lower model complexity, using \(M=1\) further enhances this advantage.

\subsection{A Variant of the LQ Game on Graphs (Section~\ref{sec:numerics_nonlq})}\label{app:M1_non_LQ}

\begin{table}[ht!]
\caption{\(\mathrm{RMSE}_X\) and \(\mathrm{RMSE}_\alpha\) for the variant of the LQ game in Section~\ref{sec:numerics_nonlq} when \(M=1\)}
\label{tab:nonLQ_RMSE_M1}
\centering
  \begin{tabular}{c|cccccc}
    \toprule
    \multirow{2}{*}{Problem} &
      \multicolumn{2}{c}{\(G = C_{10}\)} &
      \multicolumn{2}{c}{\(G = S_{10}\)} &
      \multicolumn{2}{c}{\(G = \mathrm{RST}_{10}\)}\\
    & \(\mathrm{RMSE}_X\) & \(\mathrm{RMSE}_\alpha\) & \(\mathrm{RMSE}_X\) & \(\mathrm{RMSE}_\alpha\) & \(\mathrm{RMSE}_X\) & \(\mathrm{RMSE}_\alpha\)\\
      \midrule
    NTM-DP & \(2.09e{-4}\) & \(3.18e{-3}\) & \(1.74e{-4}\) & \(2.71e{-3}\) & \(1.94e{-4}\) & \(3.13e{-3}\)\\
    NTM-DBSDE & \(1.16e{-4}\) & \(1.47e{-2}\) & \(1.55e{-4}\) & \(1.84e{-2}\) & \(1.15e{-4}\) & \(1.46e{-2}\)\\
    \bottomrule
  \end{tabular}
\end{table}

\begin{table}[ht!]
\caption{MRE for the variant of the LQ game in Section~\ref{sec:numerics_nonlq} when \(M=1\)}
\label{tab:nonLQ_value_M1}
\centering
  \begin{tabular}{c|ccc}
    \toprule
   {Problem} & {\(G = C_{10}\)} & {\(G = S_{10}\)} & {\(G = \mathrm{RST}_{10}\)}\\
      \midrule
    NTM-DP & \(2.15e{-2}\) & \(1.41e{-2}\) & \(1.52e{-2}\) \\
    NTM-DBSDE & \(1.42e{-2}\) & \(1.61e{-2}\) & \(1.35e{-2}\) \\
    \bottomrule
  \end{tabular}
\end{table}

Tables~\ref{tab:nonLQ_RMSE_M1}--\ref{tab:nonLQ_value_M1} present the RMSE and MRE across different algorithms for the variant of the LQ game on graphs.
The observations and conclusions align with those mentioned in Appendix~\ref{app:M1_LQ}.

\medskip

In brief, we conclude that NTM with \(M=1\) is sufficient for numerically solving games on graphs.

\section{Implementation Details and Hyperparameters for Numerical Experiments}\label{app:hyperparam}

\subsection{The Supervised Learning Experiment in Section~\ref{sec:SVL}}\label{app:hyper_SVL}

\textbf{Implementation details.}
For numerical implementation, \(f_3\) and \(f_4\) are directly constructed, while \(f_1\) and \(f_2\) require solving the Riccati system numerically. We adopt an explicit Runge-Kutta method of order 8, discretizing the time horizon \([0,T]\) into \(N_T = 1000\) sub-intervals.

For each target function \(f:\R^N\to\R\) and neural network architecture \(\phi^{\mathrm{NN}}:\R^N\to\R\), training is carried out over \(N_{\mathrm{epoch}}\) epochs.
In each epoch, \(N_{\mathrm{batch}}\) state samples \(x^1,\ldots,x^{N_{\mathrm{batch}}}\) are drawn from the unit hypercube \([0,1]^N\) using Latin hypercube sampling (LHS) for better space-filling. The learning target is evaluated on these samples to form labeled pairs \((x^j,f(x^j))\), and the training loss is computed as the mean squared error:
\begin{equation}
    \mathrm{loss} := \frac{1}{N_{\mathrm{batch}}}\sum_{m=1}^{N_{\mathrm{batch}}} [\phi^{\mathrm{NN}}(x^m) - f(x^m)]^2.
\end{equation}
Gradients are then backpropagated, and network parameters are updated via the optimizer. For the Chebyshev GCN, we additionally apply a learning rate scheduler that reduces the learning rate by a factor \(\gamma\in(0,1)\) per \(\tau\) epochs to improve convergence.

\medskip
\noindent \textbf{Random seed policy.}
We fix the random seeds for all libraries (e.g., \texttt{NumPy} and \texttt{PyTorch}) at the beginning of the experiment to ensure reproducibility. Multiple runs are then executed sequentially, each using the evolving state of the random number generator. This produces distinct (effectively independent) randomness across runs while ensuring that the entire experiment is fully reproducible.

\medskip
\noindent \textbf{Hyperparameter choices.} All architectures use the same depth and width across test cases. Specifically, FNN consists of 4 layers, with 32 neurons in each of the two hidden layers. 
NTM \(\phi^{\mathrm{NTM}}_{i,K,M,G}\) also has 4 layers, with the channel width \(M=3\) and \(i=1\).
The Chebyshev GCN has 4 layers, with feature dimensions \(F_1 = 1,\ F_2 = 64,\ F_3 = 1\) in its graph convolution layers. Following standard practice, FNN uses the hyperbolic tangent activation \(\sigma = \tanh\), while NTM and Chebyshev GCN use the rectified linear unit activation \(\sigma = \mathrm{ReLU}\).

For training, we use the Adam optimizer with batch size \(N_{\mathrm{batch}} = 256\). Chebyshev GCN includes a learning rate scheduler with hyperparameters \(\tau = 500\) and \(\gamma = 0.5\).
Each run of the supervised learning experiments is trained for a fixed number of epochs, after which training is terminated.
The number of training epochs scales with task complexity: \(N_{\mathrm{epoch}} = 2000\times j\) for the \(j\)-th test case \((j\in\{1,2,3,4\})\).

The learning rate \(\eta\) depends on both the task and the architecture.
Let \(\eta^{\mathrm{ARC}}_j\) denote the learning rate for architecture \(\mathrm{ARC}\) in the \(j\)-th test case. The values are:
\begin{align}
    &\eta^{\mathrm{FNN}}_1 =\eta^{\mathrm{NTM}}_1= \eta^{\mathrm{FNN}}_2 = \eta^{\mathrm{NTM}}_2 = 0.001,\quad \eta^{\mathrm{FNN}}_3 =\eta^{\mathrm{NTM}}_3= 0.005,
    \quad \eta^{\mathrm{FNN}}_4 = \eta^{\mathrm{NTM}}_4 = 0.01,\\
    &\eta^{\mathrm{CHEB}}_1 = \eta^{\mathrm{CHEB}}_2 = 0.002,\qquad \eta^{\mathrm{CHEB}}_3 = 0.01, \qquad \eta^{\mathrm{CHEB}}_4 = 0.02.
\end{align}

\subsection{Game-Solving Experiments in Section~\ref{sec:numerics}}\label{app:hyper_game}

\noindent \textbf{Implementation details.}
For non-trainable entries, we modify the default initialization of NN parameters in \texttt{PyTorch}. Trainable entries retain Xavier initialization, while non-trainable entries are initialized as zeros. Gradient hooks are then registered to block gradient updates for these entries, keeping them fixed at zero. Since \texttt{requires\_grad} in \texttt{PyTorch} is defined at the tensor rather than entry level, gradient hooks are necessary. For the same reason, training time is not significantly reduced despite the smaller number of parameters. Achieving numerical acceleration would require further development of efficient algorithms tailored for sparse neural network training (e.g., \cite{nikdan2023sparseprop}), which lies beyond the scope of this paper.

For simulating processes, the time horizon \([0,T]\) is divided into \(N_T\) subintervals of equal lengths \(h:=T/N_T\). 
Let \(\Delta := \{kh:k\in\{0,1,\ldots,N_T-1\}\}\) denote the time discretization grid. Forward and backward SDEs are forwardly simulated using the Euler scheme under this discretization \(\Delta\).
For instance, in DP and NTM-DP, the sample path \(\{\tilde{X}_t\}\) of the state process~\eqref{eqn:state_dynamics} is simulated as
\begin{equation}
    \tilde{X}^i_{t+h} =\tilde{X}^i_t + b^i(t,\tilde{X}^i_t,\tilde{X}^{\mathcal{N}^{\ell(i)}}_t,\tilde{\alpha}^i_t)\,h + \sigma^i(t,\tilde{X}^i_t,\tilde{X}^{\mathcal{N}^{\ell(i)}}_t,\tilde{\alpha}^i_t) \sqrt{h}\xi^i_t + \sigma_0^i(t,\tilde{X}^i_t,\tilde{X}^{\mathcal{N}^{\ell(i)}}_t,\tilde{\alpha}^i_t) \sqrt{h}\xi^0_t,
\end{equation}
where \(\ \xi^0_t,\ldots,\xi^N_t\overset{\mathrm{i.i.d.}}{\sim} \mathcal{N}(0,1),\ \forall i\in[N],\ t\in\Delta\) and the approximated feedback strategy \(\tilde{\alpha}\) results from the forward propagation of NNs
\begin{equation}
    \tilde{\alpha}_t := \begin{bmatrix}\phi^{\mathrm{NTM}}_{1,K,M,G}(t,\tilde{X}_t),\phi^{\mathrm{NTM}}_{2,K,M,G}(t,\tilde{X}_t),\ldots,\phi^{\mathrm{NTM}}_{N,K,M,G}(t,\tilde{X}_t)\end{bmatrix}\transpose.
\end{equation}

For strategy parameterization, as outlined in Algorithms~\ref{alg:NTM-DP}--\ref{alg:NTM-DBSDE}, a family of independent NN architectures is maintained at each time \(t\in\Delta\).
As empirically observed, this helps capture the time inhomogeneity of the equilibrium strategies, particularly over long time horizons.
As a result, in both (NTM-)DP and (NTM-)DBSDE, each player \(i\in[N]\) is associated with \(N_T\) independent neural networks, resulting in a total of \(N_TN\) independent NNs to parameterize the NE.

Training consists of \(N_{\mathrm{round}}\) rounds of DFP, each containing \(N_{\mathrm{epoch}}\) epochs. In each epoch, \(N_{\mathrm{batch}}\) sample paths are simulated to compute the loss function. Parameters are updated using Adam with initial learning rate \(\eta\), and with a scheduler that reduces the rate by a factor \(\gamma\in(0,1)\) per \(\tau\) DFP rounds.

Different from prior work \cite{hu2019deep,han2020deep}, we adopt alternating DFP instead of simultaneous DFP, which is known to exhibit a lower space complexity. In alternating DFP, all players update their strategies using the most recent information rather than those from the previous round. This requires specifying the order of updates within each round; in the absence of prior knowledge, we sample a random permutation of players at the beginning of each round. 
Numerical experiments show the empirical validity of alternating DFP, which has received limited theoretical or numerical attention in the existing literature.
For further discussion of simultaneous versus alternating DFP, see \cite[Section~5.3]{hu2019deep}.

\medskip
\noindent \textbf{Random seed policy.}
We fix the random seeds for all libraries (e.g., \texttt{NumPy} and \texttt{PyTorch}) at the beginning of the experiment to ensure reproducibility. Multiple runs are then executed sequentially, each using the evolving state of the random number generator. This produces distinct (effectively independent) randomness across runs while ensuring that the entire experiment is fully reproducible.

\medskip
\noindent \textbf{Hyperparameter choices.}
Across models, we use the same network depth and width: each FNN has 3 layers with 32 hidden neurons, while each NTM has 3 layers with channel width $M = 3$. Both architectures employ \texttt{ReLU} activations with ResNet-type skip connections \cite{he2016deep}, which empirically improve training stability over long horizons. For (NTM-)DBSDE, each player requires an additional network to approximate the mapping \(\R^N\ni x\mapsto v^i(0,x)\in\R\). This is implemented as a 4-layer FNN of width 32 with \texttt{Tanh} activation.

The following hyperparameters work robustly across models and algorithms: small perturbations of these values do not qualitatively alter the results.
\begin{equation}
    N_T = 50,\quad N_{\mathrm{round}} = 40,\quad N_{\mathrm{epoch}} = 150,\quad N_{\mathrm{batch}} = 256,\quad \eta = 0.001,\quad \gamma = 0.5,\quad \tau = 30.
\end{equation}

\end{document}